\documentclass[preprint,12pt]{elsarticle}

\usepackage{graphicx}
\usepackage{amssymb}

\usepackage{lineno}

\usepackage{latexsym}
\usepackage{hyperref}
\usepackage{url}
\usepackage{multirow}
\usepackage[utf8]{inputenc}
\usepackage{adjustbox}
\usepackage{csquotes}
\renewcommand{\mkbegdispquote}[2]{\vspace{5pt}\itshape}
\renewcommand{\mkenddispquote}[2]{\vspace{5pt}}
\usepackage[usenames,dvipsnames]{color}
\usepackage{algorithm}
\usepackage[noend]{algpseudocode}

\usepackage[english]{babel}
\usepackage{amsfonts,amsmath,mathrsfs}
\usepackage{amsthm}
\usepackage{dsfont}
\usepackage{enumerate}
\usepackage{float}
\usepackage{soul}

\usepackage{subfigure}
\usepackage[table]{xcolor}
\usepackage{lscape}
\usepackage{caption}

\usepackage{tikz}
\usepackage{framed}
\usepackage{pifont}
\usepackage{soul}

\newcommand{\xmark}{\ding{55}}
\usetikzlibrary{arrows,calc}
\usetikzlibrary{decorations.markings}

\usetikzlibrary{shapes,automata,petri,positioning}

\newcommand\score[2]{
\pgfmathsetmacro\pgfxa{#1+1}
\tikzstyle{scorestars}=[star, star points=5, star point ratio=2.25, draw,inner sep=1.3pt,anchor=outer point 3]
  \begin{tikzpicture}[baseline]
    \foreach \i in {1,...,#2} {
    \pgfmathparse{(\i<=#1?"yellow":"gray")}
    \edef\starcolor{\pgfmathresult}
    \draw (\i*1.75ex,0) node[name=star\i,scorestars,fill=\starcolor]  {};
   }
  \end{tikzpicture}
}

\newcolumntype{s}{>{\columncolor[HTML]{AAACED}} p{3.5cm}}
\newcolumntype{l}{>{\columncolor[HTML]{AAACED}} p{4.5cm}}
\newcommand{\comment}[1]{}

\theoremstyle{definition}

\newtheorem{defi}{Definition}[section]
\newtheorem{prop}{Proposition}[section]

\newtheorem{lemma}{Lemma}[section]

\newtheorem{ex}{Example}[section]
\newtheorem{ex*}{Example*}

\newcommand{\lin}{\ensuremath{\mathtt{in}}}
\newcommand{\lout}{\ensuremath{\mathtt{out}}}
\newcommand{\lundec}{\ensuremath{\mathtt{undec}}}
\newcommand{\rel}{\mathcal{R}}
\newcommand{\sen}{\mathcal{S}}
\renewcommand{\P}{\mathcal{P}}

\newcommand{\N}{\mathbb{N}}

\newcommand{\R}{\mathbb{R}}

\definecolor{ForestGreen}{RGB}{33, 166, 55}

\newcommand{\natalia}[1]{\textcolor{red}{#1}}
\newcommand{\jordi}[1]{\textcolor{ForestGreen}{#1}}
\newcommand{\jar}[1]{\textcolor{olive}{#1}}
\newcommand{\maite}[1]{\textcolor{orange}{#1}}

\newcommand{\todo}[1]{\textcolor{magenta}{\underline{{\bf TODO:}} #1}}

\journal{}

\usepackage{natbib}
\setcitestyle{square,numbers}
\DeclareGraphicsExtensions{.pdf,.png,.jpg}

\biboptions{square,comma,sectionbib}

\begin{document}

\begin{frontmatter}


\title{A model to support collective reasoning: Formalization, analysis and computational assessment}
%
%



\author[add1]{Jordi Ganzer}
\ead{ jordi.ganzer\_ ripoll@kcl.ac.uk }

\author[add1]{Natalia Criado}
\ead{natalia.criado\_pacheco@kcl.ac.uk}

\author[add2]{Maite Lopez-Sanchez}
\ead{maite\_lopez@ub.edu}

\author[add1,add4]{Simon Parsons}
\ead{sparsons@lincoln.ac.uk}

\author[add3]{Juan A. Rodriguez-Aguilar}
\ead{jar@iiia.csic.es}

\address[add1]{Department of Informatics, King's College London}
\address[add2]{Faculty of Mathematics and Computer Science, University of Barcelona}
\address[add3]{Artificial Intelligence Research Institute (IIIA-CSIC)}
\address[add4]{School of Computer Science, University of Lincoln}

\begin{abstract}

Inspired by e-participation systems, in this paper we propose a new model to represent human debates and methods to obtain collective conclusions from them. This model overcomes drawbacks of existing approaches by allowing users to introduce new pieces of information into the discussion, to relate them to existing pieces, and also to express their opinion on the pieces proposed by other users. 
In addition, our model does not assume that users' opinions are rational in order to extract information from it, an assumption that significantly limits current approaches. 
Instead, we define a weaker notion of rationality that characterises coherent opinions, and we consider different scenarios based on the coherence of individual opinions and the level of consensus that users have on the debate structure. 
Considering these two factors, we analyse the outcomes of different opinion aggregation functions that compute a collective decision based on the individual opinions and the debate structure. 
In particular, we demonstrate that aggregated opinions can be coherent even if there is a lack of consensus and individual opinions are not coherent. 
We conclude our analysis with a computational evaluation demonstrating that collective opinions can be computed efficiently for real-sized debates.

\end{abstract}


\end{frontmatter}


\section{Introduction}

We live in challenging times. With the rise of artificial intelligence and the looming climate emergency, it is no exaggeration to say that humanity is facing several simultaneous existential crises. At the same time, traditional democracy is being undermined by falling participation \cite{hooghe2017tipping,warren2002can}. In response there have been calls \cite{allen2015process,font2015participation,smith2009democratic} for new ways to tighten the relationship between political institutions and the citizens they represent. Social networks and the palpable impact they had in society have inspired new technologies aiming to improve this relationship --- e-governance, or more specifically e-participation systems \cite{fung2001deepening}. These are systems that are designed to allow citizens to propose, discuss, and even decide, government policy through online platforms.


Significant examples of deployed e-participation systems are those used by the city governments of Barcelona \cite{decidim}, Reykjavik \cite{better-reykjavik}, Madrid \cite{decideMAD}, Helsinki \cite{Helsinki19}, and the one used by the French government \cite{parlement-et-citoyens}. In these systems, participants can carry out structured discussions around some issue.
For example the systems deployed by Barcelona, Reykjavik, Madrid, and Helsinki are for proposals about local issues, while the French government system, Parlement et Citoyens, is for the discussion of potential national legislation. Turning to systems that are not tied to a specific institution, there are tools such as consider.it \cite{consider-it}, Appgree \cite{appgree} and Baoqu \cite{baoqu} whose main focus is scalability --- making the systems fit for use by large numbers of participants.

%

Several methodologies have been already explored as a means to tackle this problem. \emph{Social choice theory} \cite{aziz2017computational,List2017} is an approach for establishing the collective opinion of a group facing a choice between many alternatives. Given a set of alternatives and a set of agents who possess preference relations over the alternatives, social choice theory focuses on how to yield a collective choice that appropriately reflects the agents’ individual preferences. A related approach, but one which focuses on the acceptability of a single issue, is \emph{judgement aggregation}. This tackles the problem of whether to collectively accept  a single issue once the participants have put forward their opinion on it \cite{endriss_moulin_2016,list:2002:aggregating}. 
A further approach, one that focuses more on resolving conflicts in opinions, is \emph{computational argumentation} \cite{Rahwan2009}. Given a set of arguments for particular options, and a set of attack relations (conflicts) between the arguments, argumentation is concerned with identifying those arguments that might be accepted by a rational agent (for different ideas of what makes an argument acceptable).
Mixing judgement aggregation or social choice theory together with argumentation we find several proposals that structure debates using arguments and attack relationships, such as \cite{Awad:2015:JAM,Leite2011}, or attack and defence (or support) relationships \cite{Ganzer2018}. These allow participants to put forward arguments, relations between arguments, and opinions about which of these arguments and relations hold. They then produce an output that is intended to reflect the collective opinion of the participants on the the debate.

This article proposes a new formal model, which we call the ``relational model'',  to structure the information in a debate while allowing more expressiveness than existing approaches permit, both in terms of the structure of the debate, and in terms of the opinions expressed by participants. We propose several methods for aggregating the information of the debate captured by the model to obtain an output reflecting the collective view of the participants involved. Furthermore, we analyse the performance of such methods with respect to several social choice properties adapted from those proposed in the social choice literature.
Finally, we study the computational properties of the aggregation methods to establish whether it is feasible to use them in practice.

In more detail the contributions of this work are:

\begin{itemize}[-]
    \item \emph{Increasing expressiveness.} Collective decision-making frameworks based on traditional argumentation models usually take as starting point a fixed argumentation structure that only model attack relationships between arguments or a combination of attack and support/defence relationships between arguments. These frameworks then allow participants to express opinions about the different arguments included in the debate \cite{Awad:2015:JAM,Ganzer2018,Leite2011}. The fixed nature of the argumentation structure, even if it is defined by the participants, represents a significant drawback for e-participation systems. Adopting a fixed argumentation structure using only attacks, or attacks plus supports/defences, reduces what participants are allowed to express. For instance, a fixed attack relationship between two arguments might be problematic for some participants who disagree on the classification of the relationship as an attack and would have defined the same relationship as a defence. 
    To solve this problem, our model uses relationships that are not subjectively classified as attacks or defences, but only represent the connections between elements of the debate. The subjective classification is applied individually by each participant not in terms of attack or defence, but in terms of acceptability of the connections. Thus, the structure of the debate is focused on organising relevant information, not on expressing the subjective opinion of the participants, which is added separately using other tools.
    
    \item \emph{Going beyond abstract argumentation.} Several approaches \cite{Awad:2015:JAM,Coste-Marquis2007,Leite2011} make use of abstract argumentation frameworks \cite{Dung1995}, or some variations of them, as in \cite{Ganzer2018}, to represent the elements of a debate. In such frameworks, whole arguments are the atomic elements. In work on argumentation, this limitation has led to work on ``structured'' \cite{modgil:prakken:ai} or ``rule-based'' \cite{garcia:simari:tplp}  argumentation which constructs arguments out of lower level components like facts and rules. 
    Since we believe that a debate can hinge on being able to address such lower level components, we take a similar, but more general, approach. 
    We construct debates from two types of abstract objects: statements, which represent sentences without reasoning, and the relationships between statements, which represent the existing reasoning connecting the statements\footnote{We could relate the statements with axioms, premises or conclusions, and the relationships by the rules or demonstration steps that lead from premises to conclusions, just as in structured argumentation.}. 
    We make a sharp distinction between these two types of information, statements and reasoning connecting statements, in order to allow them to be subsequently evaluated in different manners by the participants.
    
    \comment{Depicting the structure using a graphic representation, the statements could be viewed as nodes of a graph and the relationships or reasonings as its edges connecting the nodes.\natalia{Here we need to state why abstract argumentation is not good for e-participation systems}}
    
    \comment{\jordi{Maybe not here: Furthermore, since arguing from different perspectives can lead to connecting the same statements via different reasonings we allow the connection of the statements by several reasonings at the same time, each one represented by a different relationship. Having only one reasoning representing a relationship might lead to some participants to not feel sympathetic with the connection while allowing more relationships representing different reasonings might help the participants with different sensibilities to feel more sympathetic with some of the connections between statements and hence strengthen (via their opinion expressed afterwards) the dependence of the statements.}}
    
    \item \emph{Compound and real valued opinions.} Previous work on argumentation-based approaches has only allowed participants in a debate to express opinions about either the arguments \cite{Ganzer2018,Leite2011}, or about the relationships between arguments \cite{Dunne2011}\footnote{\cite{Dunne2011} is not about combining collective opinions on relationships between arguments, but it provides the groundwork for such a system by studying argumentation where the relationships between arguments have different weights.}. Here we allow participants to provide opinions on both statements and the relationships between them. Opinions about relationships capture participants' acceptance, or otherwise, of the reasoning that the relationship represents; and opinions about statements reflect participants' satisfaction with the statement itself.
    
    Furthermore, we follow \cite{Dunne2011,Leite2011} in allowing opinion about relationships and statements to be expressed using real values rather than the discrete values of \cite{Awad:2015:JAM,Ganzer2018}. This feature allows the participants to express their opinions in a wider range of values, making the approach more flexible. (It should be noted though that most existing e-participation systems just allow users to express agreement or disagreement.)

    \item \emph{A more flexible notion of coherence.} Previous work on determining collective opinions makes use of a notion of ``rationality'' in which an opinion is either determined to be acceptable or not acceptable (where ``acceptable'' has different interpretations but reflects the constraints on distributions of opinions across statements) \cite{Awad:2015:JAM,Ganzer2018,RagoToni2017}. We think that this is somewhat limiting. Since the opinions originate with human participants, and humans are not always consistent in their views, we feel that insisting on this rigid form can lead to lose valuable information.
    Hence, we propose a less restrictive notion of rationality, which we call ``coherence'', to assess the degree to which an opinion is coherent, be it from an agent or from the collective aggregation. To take a simple example, an agent valuing with opposite values two related statements, one supporting the other, would be understood as not coherent, while another agent valuing the same statements with similar values would be classified as coherent.
    Note that we do not compare different agents' opinions in establishing coherence, but look within the same opinion.
    
    \item \emph{Aggregation functions exploiting dependencies}. We propose several opinion aggregation functions that use the participant's opinion on a debate to compute a collective opinion. These proposed functions assume different forms of using the dependencies between opinions to combine them. We provide two families of function that include a function that does not use dependencies at all, and two other functions that use dependencies rather differently. These families of functions collective span the ways in which the dependencies can be taken account of, thus making it possible to choose a specific degree of use of the dependencies.
    
    \item \emph{Formal analysis}. We assess these families of functions against a wide-ranging set of properties designed to provide a detailed characterization of their behaviour. We use several properties adapted from the social choice literature \cite{list:2002:aggregating} to fit our model in order to assess the performance of the functions. We carry out the same  study in four scenarios that consider different assumptions for the participants' opinions.

    
    \item \emph{Computational analysis.} We follow the formal analysis with a computational analysis. This first computes the computational complexity of the aggregation functions, and then provides an empirical analysis of those functions when computing collective opinions for a range of scenarios that are larger (in terms of statements and number of opinions) than any online debates that we are aware of. This analysis shows that collective opinions can be computed in real time on quite modest hardware.
\end{itemize}
These contributions can be found in the following sections: section \ref{sec:ModelView} provides an introduction to the relational model; and sections \ref{sec:FormDef} and section \ref{sec:Expectation_Coherence} provide a formal definition of the model. Then, section \ref{sec:Aggregation} defines the problem of computing collective coherence and introduces the properties that will be used to assess aggregation functions, while section~\ref{sec:OpAggFunc} defines a family of aggregation functions and section~\ref{sec:Analysis} uses the properties to analyse the functions in different scenarios (the proofs can be found in \ref{sec:Proofs}). Section~\ref{sec:CompAnalysis} provides the computational assessment of the model; section \ref{sec:RelatedWork} relates the work presented in this paper to other relevant work in the literature; and section~\ref{sec:conclusions} summarises our conclusions.

\section{Foundations of the relational model}\label{sec:ModelView}

This section introduces the main features of the relational model while the following sections provide a formalisation of the model. We start this section by describing the elements of the model and the rationale behind it, as well as the aggregation of opinions using the information enclosed within the model. We also include a discussion of how the model could be used in a participatory system. 

\subsection{Relational Model}\label{sec:featuresRM}

The \emph{Relational Model}, RM for short, is a model designed to represent a collective debate where participants discuss a proposal by putting forward additional information relevant for the discussion and giving their opinions about it. 
The RM is composed of two main parts: the structural part, organising the information (or \emph{content}) of a debate; and the interpretative part, representing the participants' \emph{opinions} in a debate.

\paragraph{Content of a debate}

The  RM has two main abstract elements that capture the structure of a debate: \emph{statements} and the \emph{relationships} among them. \emph{Statements} represent plain sentences that describe facts
such as, for example, ``Increase of house prices in the neighbourhood''. Participants will afterwards be able to express their
opinion about the desirability or undesirability of the sentence. 
\emph{Relationships} represent the reasoning that connects statements. Each relationship connects a set of source statements to some destination statement.
We can think of relationships as logical inferences that may or may not be accepted by the participants of the debate. 
Such acceptance will also be included in participants' opinions.

Commonly, debates discuss a particular subject or proposal, and may even consider a set of proposals. 
In the RM we consider the \emph{target} (of a debate) as a set of statements that represents the proposal. 
The distinguishing feature of these statements, the target of the debate, is that none of them can be a destination statement in any relationship because they are the initiators of the debate. 
Thus, the target acts as root of the structure that captures the debate. 
This structure composed of the statements, relationships and  target is called a \emph{Directed Relational Framework}, or DRF for short. 

\paragraph{Participants' opinions} 
The opinions of the participants are encoded in the form of functions that assign values to the objects that make up the structure of the debate.
The \emph{valuation} function assigns values to statements and the \emph{acceptance} function assigns values to the relationships. 
The opinion of a participant is twofold in order to obtain the two types of subjectivity involved in the debate. 
The first, relating to the valuation function, provides a participant with a way to express their judgement about the facts or possibilities that arise in the debate 
(but does not express 
their view about the truth of the sentence).  
The second, relating to the acceptance function, allows a participant to express the truth they see in each relationship that is included in the debate.
Thus, the desirability or undesirability that each participant feels about each statement of the debate is represented by a positive or negative value assigned with the valuation function, and the conformity that each agent relates to the connections between the statements is represented by an acceptance value assigned by the acceptance function.

\comment{\begin{ex}
Let's follow the statement's example to illustrate these concepts:
The relationship $\r_1=$"If the people vote mostly Yes in the Brexit referendum then the UK will leave the EU" could be the reasoning explaining the relationship between the statements $s_1=$"Hold a Brexit referendum" and the statement $s_2=$"The UK leaves the EU". Participants in a debate that contains these elements can place values on the relationship, indicating whether they accept the reasoning it embodies. (Rational participants would presumably place a value indicating agreement, indicating that one statement follows from the other.). However, participants could place positive or negative opinions on the statements involved, and their choice would depend on whether each wants the UK to leave EU or not, and whether they want to risk this possibility by carrying out a referendum.
\end{ex}}

\begin{ex}
Consider a debate about building a sports centre.
Let's assume a scenario where a proposal $\tau$ ``Construction of a sports centre in particular location in the neighbour'' is discussed. 
In this setting, we may consider the reasoning that ``The construction of the sport centre will imply the demolition of  existing buildings that give historical relevance to the neighbourhood" is represented by the relationship $r_1$  that connects the target $\tau$ to a new statement $s_1$ ``Destruction of the character of the neighbourhood'', as depicted in Figure \ref{pic:ex1}.

Given these elements, participants in the debate can express their opinions by providing a value to both $\tau$ and $s_1$ and by assigning a degree of acceptance to the relationship $r_1$. The values assigned to $\tau$ or $s_1$ will represent how they feel about the respective statements and the degree given to $r_1$ will indicate how much the participant believes in the truth of the reasoning relating both statements. Thus, someone hoping for a new sports center in the neighbour --- hence valuing the target $\tau$ positively --- may not want the neighbourhood to be damaged --- so valuing $s_1$ negatively --- but still believes that the construction of the sport center indeed will imply the demolition of the existing buildings that now are historically relevant (and hence would create that damage) ---thus accepting the reasoning represented by $r_1$. 
\end{ex}

\comment{\begin{figure}[H]
    \centering
    \includegraphics[width=4.5cm]{pics/DRF1relationship.png}
    \caption{ Graphical representation of the relationship between  proposal $\tau$ (building a sports centre) and statement $s_1$ (destroying the neighbourhood's character).   \maite{\todo{Update symbol in the relationship} \maite{
    need also to change the rest (Fig 4, Table 2,...).} }}
    \label{pic:ex1}
\end{figure}}

\begin{figure}[H]
    \centering

   \begin{tikzpicture}[every node/.append style={inner sep=3mm},every place/.style={line width=0.3mm}]

\node[rectangle,draw,line width=0.7mm] (T) {\large{$\tau$}};
\node(s1) [circle,draw] at ($(T)+(-2,-2)$){$s_1$} ;

\node [] (r1) at ($(T)+(-1.1,-0.7)$){$r_1$};
\draw[->,thick,>=latex] (T) to[out=-135,in=45] ($(s1)+(0.5,0.5)$);


\end{tikzpicture}
    \caption{Graphical representation of the relationship between  proposal $\tau$ (building a sports centre) and statement $s_1$ (destroying the neighbourhood's character)}
    \label{pic:ex1}
\end{figure}
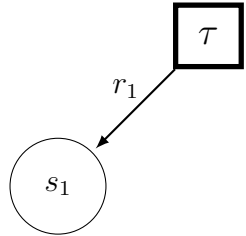

When considering this example it is worth elaborating on the direction of the relationship regarding content and opinion. 
Observe that relationship $r_1$ is directed from $\tau$ towards statement $s_1$ so that $s_1$ can be interpreted as its consequence.
But, conversely, the opinion about $s_1$ ---be it either positive or negative--- will also affect the opinion on $\tau$ through relationship $r_1$. 
Thus, the opinion that a participant will have on a statement, in this case the target, 
will be conditioned both by the consequences of 
this statement and by the things from which it itself follows. 
In general, a direct relationship between statements has to be established and we then decide which direction to take when evaluating the framework. We may define this evaluation so as to relate one statement $s$ (or a set of statements) to (i) its consequence statements, or, to (ii) the statement whose opinion is affected by the opinion about $s$.
We choose the first (i) option to reflect the direction of reasoning (from premises to conclusions) in the debate structure. 
However, our subsequent process aimed at opinion aggregation will traverse these relationships in the opposite (ii) direction, though we will use the same structure to represent the debate.

\comment{Next example illustrates how to structure a debate using the Directed Relational Framework. We will make use of the argumentative scheme we are used to see in discussions so we can translate it into our model and see how changes the structure of the debate. We will evolve this example throughout the paper to illustrate the different concepts we will be introducing.

\begin{ex}
Let be a debate about whether or not to build a Sports centre on a certain location on the surroundings of a neighbourhood.
Hence, the target of our debate is this proposal that can be divided using two statements:
\begin{itemize}
    \item [$\tau_1=$] An specific location on the surroundings of the neighbourhood for the building.
    \item [$\tau_2=$] Build a Sports centre in the neighbourhood.
\end{itemize}
Now, let's see what peoples arguments in a normal discussion can be to, afterwards, structure the information using the objects of directed relational framework.
\begin{itemize}
    \item The argument "The location expected for the sports centre implies the demolition of an antique building" it is clearly relating the "location for the building", so $\tau_1$, to the consequence that will "demolish an antique building", name it $s_1$. So we have a reasoning $r_1$ connecting $\tau_1$ towards $s_1$ representing the reasoning  "Building on that location implies the demolition of the antique building".
    \item Next consider the argument "A sport centre in that location is well placed and easy accessible for the neighbourhood". It decomposes as a initial statements of the relationship $\{\tau_1,\tau_2\}$ towards the statement $s_2=$"The sports centre will be accessible for the neighbourhood" via the reasoning $r_2=$"As the location is on the surroundings of the neighbourhood, it makes it an accessible location for a sports centre".

    \item  The argument "An sports centre will promote the sports practice in the neighbourhood" will divide in the following: the relationship $r_3$ representing the reasoning "An sports centre promotes the sports practice on the people around it." will relate the target $\tau_2$ towards the statement $s_3=$"promotion of sports practice  in the neighbourhood". 
    
    \item The argument "To have an sports centre and the fact that will be very accessible will increase the social activity in the surrounding" is converted into: the relationship $r_4$ representing the reasoning "An sports centre in a good location increases the social activity of the neighbourhood" connects the initial statements $\{s_2,s_3\}$ towards the new statement $s_4=$"Increase of social activity".
    \item And finally, the last argument "The increase of activity on this accessible place will also promote the commerce in the area" connects via the relationship $r_5$, reflecting the reasoning "Social activity in this accessible place will lead to more economic activity, thus the commerce will increase in the area", the initial statements $\{s_3,s_4\}$ to the final statement $s_5=$"Commerce increased in the neighbourhood".
\end{itemize}

Note that in this example we used a few arguments with simple reasonings and statements so that could serve us to illustrate the structure of our model. But, in reality this example would have been more complete and composed of more arguments about other issues relating the discussion, and hence more statements and relationships.

\end{ex}
}

\paragraph{Direct and Indirect Opinion}
\label{subsub:interpretation}

Each participant involved in a debate expresses its opinion in terms of the desirability or undesirability of the different statements in the model, and the acceptability of the relationships linking these statements. 
As mentioned above, relationships are a key element of the RM: they indicate that the opinion about one statement affects the opinion about another statement. 
In particular, we will use the terms: (i) \emph{direct opinion} to refer to the value directly given to a statement by a participant; and (ii) \emph{indirect opinion} to refer to the values given to the related objects whose opinions may condition the direct opinion.
Note that we are dealing with human opinions, and so we cannot expect that opinions are rational --- contradictions or inconsistencies between the direct and indirect opinions expressed by participants may arise, just as in the example shown above. 
By comparing the direct opinion with the indirect opinion about each statement we can categorise ``reasonable'' opinions and provide a notion of \emph{coherence} \cite{thagard2002coherence}.    

\begin{ex} \label{ex:opinion_initial}
Now consider three people taking part in the debate described above with the following opinions: 
\begin{itemize}[-]
    \item {\bf Participant 1} is a middle-aged woman with a family who lives in the neighbourhood where the sports centre is proposed to be built.  
    She values the proposal ($\tau$) very positively, because her family practices sports. 
    Although she has been living in the neighbourhood for a long time, she has no childhood memories of and doesn't care too much for the historical character of the neighbourhood, so she values statement $s_1$ as neutral. 
    Finally, she assigns a small acceptance value to relationship $r_1$ because she acknowledges that building the sports centre will imply the demolition of some buildings, but she doesn't believe that the character of the neighbour will be too affected by such a loss.
    
    \item {\bf Participant 2} is a retired elderly man. He has always lived in the neighbourhood and he would like to preserve its unique features. 
    Hence, he values statement $s_1$ negatively, and also target 
    $\tau$ because he is not interested in sports and would prefer another kind of public building instead. 
    He fully agrees with relationship $r_1$, since he considers that the buildings in the proposed location for the new sports centre are important to the neighbourhood and would be demolished if the centre is built.
    
    \item {\bf Participant 3} is a young postgraduate student who does not practice any regular sports that can be hold in the planed sports centre, so values $\tau$ quite negatively.  
    He agrees with relationship $r_1$ because he acknowledges that the existing buildings could be catalogued as of special architectural interest. 
    However, he is neutral with respect to $s_1$ because he does not care about preserving the character of the neighbourhood.
    
\end{itemize}
\end{ex}



\paragraph{Opinion Aggregation}
\label{sec:model_aggregation}

Once all users have expressed their opinions on statements and relationships, opinions must be aggregated to calculate a collective opinion. 
This aggregation can take into account direct opinions, indirect opinions, or a combination of both direct and indirect.
In establishing suitable aggregation functions, we have to take into account that individual opinions may be incoherent. 
Nonetheless, we aim to design aggregation functions that can combine these ``imperfect'' individual opinions into a ``reasonable'' collective opinion. 

The main contribution of the RM is to distinguish clearly between \emph{objective} and \emph{subjective} information in a simplified way, and to do this differently to many other argumentation models (as \cite{Awad:2015:JAM,Ganzer2017,Leite2011}), where subjective information is jointly represented with objective information in the debate structure leading to predefined relationships (such as attack or defence) that must be shared among all participants. 
Thus, the expressiveness of the RM is significantly higher being able to represent debates where participants do not need to agree on the relationships between different objective facts. 
In addition, these argumentation models usually presume
that users opinions are rational (i.e., fully consistent). However, human debates are  affected by information and/or cognitive limitations of users, which makes this rationality requirement too restrictive. 
We argue that our coherence notion is a weaker version of this rationality requirement that is better suited to model human debates. 


\subsection{The collective decision workflow}
\label{subsec:collectiveWorkflow}


A debate can be built in many ways following the parameters of a DRF. Next we propose a simple procedure that would produce 
a DRF and will serve as a support guide for the sections to come.
There are four main steps to applying our model:

\begin{figure}[p]
    \centering
    \includegraphics[trim=0 11cm 0 0,clip,width=8.5cm]{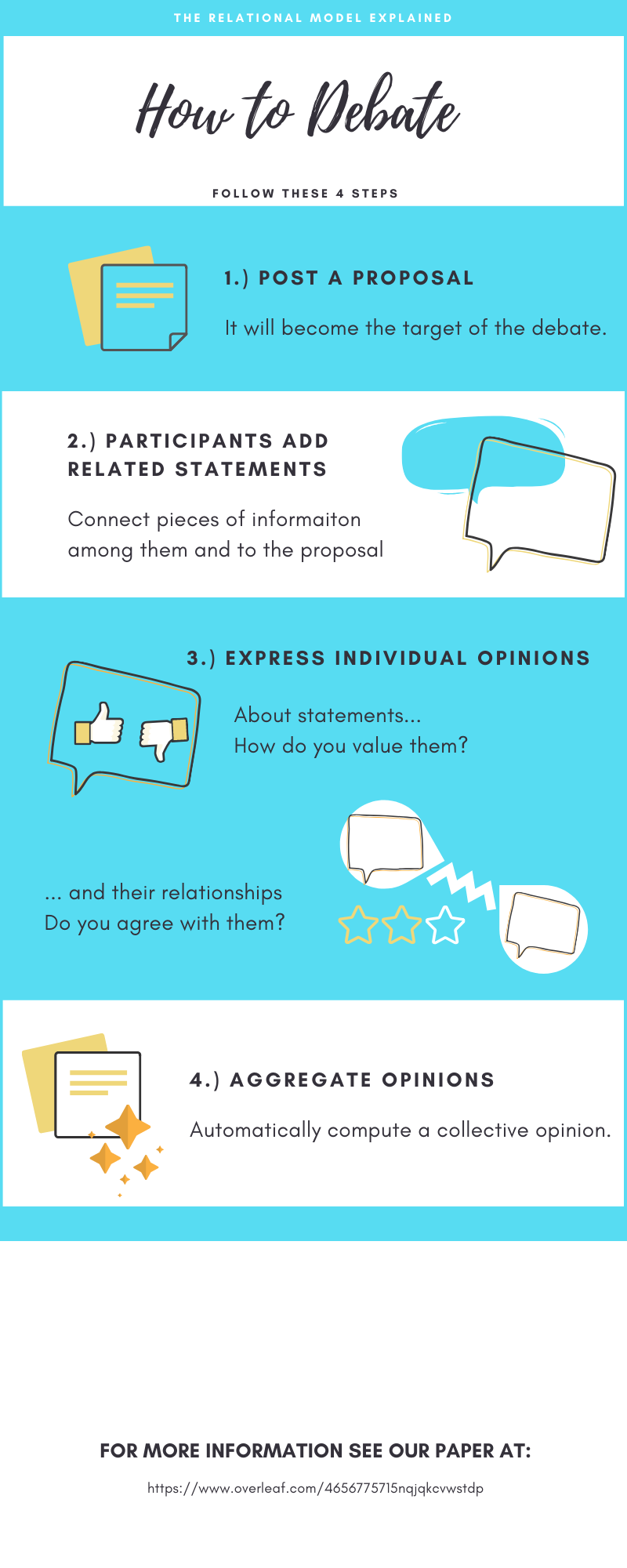}
    \caption{Participation using the relational model: debate steps graphically explained. 
    }
    \label{pic:DebateSteps}
\end{figure}


\begin{description}
\item {\bf Step 1 -} \emph{Start debate.} A set of statements for which we intend to obtain a collective opinion are chosen as targets of the RM.
\item {\bf Step 2 -} \emph{Extend debate.} Participants are then allowed to put forward relationships that will represent relevant reasoning.
A relationship may either 
be put forward in conjunction with new statements 
--so it connects from existing statements to a new statement-- or it may connect existing statements. 
This step continues until no participant wishes to add a further relationship.


\item {\bf Step 3 -} \emph{Input opinions.}
Participants express their opinions on all the relationships and  statements in the RM, by providing \emph{subjective evaluations} of them. The evaluations of statements expresses preferences over them, while the evaluation of relationships expresses agreement, or otherwise, with the reasoning represented by the relationship. 


\item {\bf Step 4 -} \emph{Obtain collective opinion.} The opinions of the participants are merged to establish a consensus view of each statement and relationship that are in the framework. Hence we also obtain the collective opinion of the target statements.
\end{description}
These steps are shown graphically in Figure~\ref{pic:DebateSteps}.

\section{Formalising the relational model}
\label{sec:FormDef}

Recall from section  \ref{sec:ModelView} that we conceptually distinguished between the structural part of the relational model
, namely the statements made during a debate together with their causal relationships, and the opinions put forward by participants in a debate. 
Figure \ref{pic:RelationaModelComp} depicts this distinction, together with all the basic elements presented in this section. 
Thus, in what follows, we formally introduce both concepts, first the structure and then the opinions, with the aid of 
an extension of the example shown in section \ref{sec:ModelView} that we employ throughout the paper.

\begin{figure}[p]
    \centering
     \includegraphics[trim=0 40cm 0 26cm,clip,width=8.9cm]{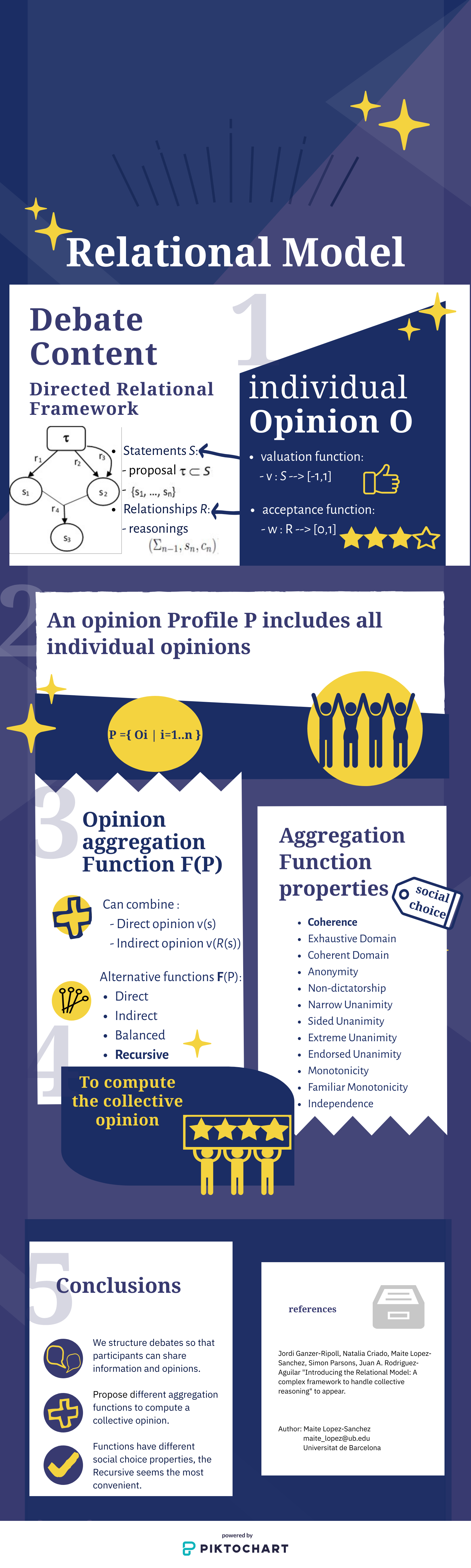}
    \caption{The basic elements of our relational model.}
    \label{pic:RelationaModelComp}
\end{figure}




\subsection{Formalising the structure}
\label{subsec:formalStructure}


First, we introduce the formal notion of  a \emph{relational framework} to capture the relationships between statements. Our notion of relationship will consider a non-empty set of source statements and a destination 
statement. In general, by relating a set of (source) statements to a (destination) statement, we indicate that the source statements support inferring the destination 
statement, though the framework is agnostic about the form that the support and the inference mechanism takes. For instance, in our example in Figure \ref{pic:ex_DRF}, 
statements $s_2$ and $s_3$ support inferring $s_4$. Formally:

\begin{defi}
\label{def:RF}
A \emph{relational framework} $RF$ is a pair $\langle\sen,\rel\rangle$, where $\sen$ is a set of statements and $\rel\subset\P(\sen)\times\sen \times \mathbb{N}$ is a relation fulfilling:

\begin{itemize}

\item Acyclicity. There are no cycles in $\rel$, namely there is no subset of relationships $\{(\Sigma_0,s_1, c_1),\ldots,(\Sigma_{n-1}, s_n, c_n
)\}\subset \rel$ such that $s_i\in \Sigma_i$, $i\in\{1,\dotsc,n-1\}$, and $s_n\in \Sigma_0$.
\end{itemize}
\end{defi}

Since the relation $\rel$ is acyclic, it follows that $\rel$ is neither \emph{reflexive} ( $\forall s\in\sen$, $(\Sigma\cup\{s\},s, c)\notin \rel$) nor \emph{symmetric} ($\forall s_1,s_2\in\sen$, if $(\Sigma_1\cup\{s_1\},s_2, c_2)\in \rel$ then $(\Sigma_2\cup\{s_2\},s_1, c_1)\notin\rel$). Note that we do not impose any restriction on the transitivity of the relation $\rel$.

Notice also that we include a natural number within the relation in order to differentiate relationships between the same set of statements $\Sigma$ and $s$. From a practical perspective, this allows to signal that alternative relationships can bear on the very same statements\footnote{The inclusion of a natural number into the specification of a relationship does not affect the formal contributions of the paper, since, as it will be shown later, relationships are grouped into (and subsequently traversed through) the sets $R^+$ defined in equation \ref{eq:descendants}, without considering any restrictions on the statements they relate.} (as shown
in
Figure 
\ref{pic:ex_DRF}, where target $\tau$ is related to statement 
$s_1$ through relationships 
$r_1$
and 
$r_6$).

Now, since debates are aimed at achieving a collective decision on target topics, we extend our definition above to incorporate the notion of \emph{target} statements as follows:

\begin{defi}
\label{def:DRF}
A \emph{directed relational framework} (DRF) is a tuple $\langle \sen,\rel, T \rangle$ such that:

\begin{enumerate}[(i)]
\item $\langle \sen,\rel\rangle$ is a relational framework;

\item $T \subset\sen$ is a set of target statements;

\item Target statements in $T$ can only be the source of relationships, namely for any relationship $(\Sigma, s, c)\in \rel$, $s\notin 
T
$; and

\item All non-target statements are connected to targets so that for any statement $s \in S$, $s \notin T$, there is a path $\{(\Sigma_0,s_1, c_1),\allowbreak\ldots,(\Sigma_{n-1}, s, c_n)\}\subset \rel$ such that $T \cap\Sigma_0\neq \emptyset$.

\item Every target statement shares some common descendant with some other target statement, namely for every $
\tau\in T$ there is another target $
\tau'\in T$ and some statement $s \in S$ such that there is a path from $
\tau$ to $s$ and another path from $
\tau'$ to $s$.

\end{enumerate}
\end{defi}

Note that
a DRF is constrained to be a \emph{connected} \emph{acyclic} graph, albeit one that can have several targets. This reflects the idea that, since a DRF represents a single debate, every statement in that debate should have some connection
 to the rest of the debate.

In what follows we extend our example in section \ref{sec:ModelView} to produce a graphical representation of a DRF that will help us visualise the information in a debate\footnote{For the sake of simplicity, we limit the example to three participants and a small number of statements and relationships.}. In this representation, statements are nodes and relationships are arcs. Recall that our example considered statements $\tau$ (Construct a new sports centre in a specific place in the neighbour) and $s_1$ (Destruction of the neighbour character), as well as relationship $r_1$ (
The construction of the sport centre will imply the demolition of existing buildings that give historical relevance to the
neighbourhood) connecting both. Besides that, next we consider further statements and relationships as listed in tables \ref{tab:statementsDRF} and \ref{tab:relationshipsDRF} respectively. Finally, figure \ref{pic:ex_DRF} depicts the connections between statements through relationships. Note that $r_4$ is a hyperedge, connecting three statements\footnote{Notice that one participant could introduce an extra relationships from $\tau$ to $s_5$, representing the reasoning ``A new community center will give more relevance to the neighbourhood, that will increase the house price'', which is not the sum of $r_2,r_3$ and $r_4$, but represents a whole new way to connect $\tau$ to $s_5$. This shows that the transitivity allowed in the model, connecting statements from not consecutive levels of the debate via a single relationship, is not intended to represent the combined reasoning formed by the reasoning steps in between.}.

\begin{table}[]
    \centering
            \footnotesize
    \begin{tabular}{|c|c|}
\hline
\textbf{Statement}&\textbf{Description}\\ \hline
$\tau$ & Construction of a sport centre in a particular location in the neighbour\\ \hline
$s_1$ & Destruction of the neighbour character\\ \hline
$s_2$ & Attraction of more affluent residents to the neighbour\\ \hline
$s_3$ & Attraction of new business to the neighbour\\ \hline
$s_4$ & Crime reduction in the neighbour\\ \hline
$s_5$ & Property values raise in the neighbour\\
\hline
\end{tabular}
    \caption{Statements for the sports centre example.}
    \label{tab:statementsDRF}
\end{table}
\normalsize

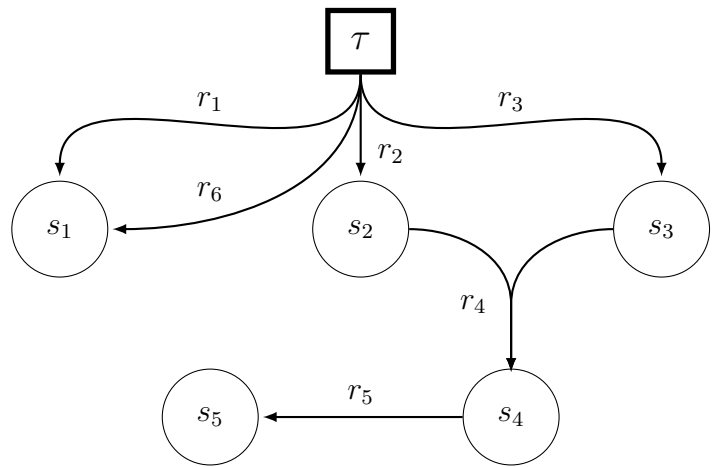
\begin{figure}[H]
    \centering

   \begin{tikzpicture}[every node/.append style={inner sep=3mm},every place/.style={line width=0.3mm}]

\node[rectangle,draw,line width=0.7mm] (T) {\large{$\tau$}};
\node(s1) [circle,draw] at ($(T)+(-4,-2.5)$){$s_1$} ; 
\node(s2) [circle,draw] at ($(T)+(0,-2.5)$){$s_2$} ; 
\node(s3) [circle,draw] at ($(T)+(4,-2.5)$){$s_3$} ; 
\node(s4) [circle,draw] at ($(T)+(2,-5)$){$s_4$} ; 
\node(s5) [circle,draw] at ($(T)+(-2,-5)$){$s_5$} ; 

\node [] (r1) at ($(T)+(-2,-0.8)$){$r_1$};
\draw[->,thick,>=latex] (T) to[out=-90,in=90]  ($(s1)+(0,0.7)$);

\node [] (r2) at ($(T)+(0.4,-1.5)$){$r_2$};
\draw[->,thick,>=latex] (T) to[out=-90,in=90]  ($(s2)+(0,0.7)$);

\node [] (r3) at ($(T)+(2,-.8)$){$r_3$};
\draw[->,thick,>=latex] (T) to[out=-90,in=90] ($(s3)+(0,0.7)$);

\node [] (r4) at ($(s2)+(1.5,-1)$){$r_4$};
\draw[->,thick,>=latex] (s2) to[out=0,in=90] (2,-3.5)  to[out=-90,in=90]   ($(s4)+(0,0.6)$); 
\draw[->,thick,>=latex] (s3) to[out=180,in=90] (2,-3.5) to[out=-90,in=90,line width=0.6mm] ($(s4)+(0,0.6)$);

\node [] (r5) at ($(s4)+(-2,0.3)$){$r_5$};
\draw[->,thick,>=latex] (s4) to[out=180,in=0]  ($(s5)+(0.7,0)$);

\node [] (r6) at ($(T)+(-2,-2)$){$r_6$};
\draw[->,thick,>=latex] (T) to[out=-90,in=0]  ($(s1)+(0.7,0)$);

\end{tikzpicture}
    \caption{The DRF for the sports centre example. }
    \label{pic:ex_DRF}
\end{figure}

\begin{table}[]
    \centering
        \footnotesize
    \begin{tabular}{|c|c|c|}
\hline
\textbf{Relationship}&\textbf{Reasoning}&\textbf{Connection}\\ \hline
$
r_1$ & The construction of the sport centre will imply  the & $\tau$ to $s_1$\\
 & demolition of existing buildings which now give  & \\   & historical relevance to the neighbour &\\
\hline
$r_2$ & The new sport centre will make the neighbour more & $\tau$ to $s_2$\\
 & attractive for wealthy residents because  & \\
  & they are more interested in leisure activities &\\
\hline

$r_3$ 
 &  A new community centre will attract more & $\tau$ to $s_3$\\
 &   businesses to the surrounding area.&\\
\hline

$r_4$
&  Having richer residents and more businesses will  & $\{s_2,s_3\}$ to $s_4$\\
&  increase the security measures around the neighbour&\\
&  and therefore reduce criminal activities.&\\
\hline

$r_5$ & The reduction of crime  will increase the price of  & $s_4$  to $s_5$ \\ & the houses in the neighbour & 
\\ 
\hline
$r_6$ & A new building will make collide the new   & $\tau$  to $s_1$\\
& architecture with the extended old nature of  &\\
&  the area, changing its character & \\

\hline
\comment{\hline
\hline
& \jordi{Below, possible candidates for new reasonings}&\\ 
\hline
\hline
$r_6$ & A new community center will give more relevance to the & $\tau$  to $s_5$\\ 
&   neighbour, increasing the house prices &\\
\hline
$r_7$ & Increased activity in the area makes will make it & $s_2,s_3$  to $s_4$\\ 
&  more crowded reducing then the possibilities to commit &\\
& unseen crimes & \jordi{second one}\\
\hline
$r_8$ & A new building will make collide the new architecture with & $\tau$  to $s_1$\\
& with the extended antique nature of the area, &\\
&  changing its character & \jordi{second one}\\}
\hline

\end{tabular}
    \caption{Reasoning for the sports centre example.}
    \label{tab:relationshipsDRF}
\end{table}
\normalsize

\comment{\begin{figure}
\centering
\begin{tikzpicture}[every node/.append style={inner sep=3mm}]

\node[rectangle,draw,line width=0.7mm] (T) {\large{$\tau$}};
\node(s1) [circle,draw,line width=0.3mm] at ($(T)+(-4,-2.5)$){$s_1$} ; 
\node(s2)[circle,draw,line width=0.3mm] at ($(T)+(0,-2.5)$){$s_2$} ; 
\node(s3)[circle,draw,line width=0.3mm] at ($(T)+(4,-2.5)$){$s_3$} ; 
\node(s4)[circle,draw,line width=0.3mm] at ($(T)+(2,-5)$){$s_4$} ; 
\node(s5)[circle,draw,line width=0.3mm] at ($(T)+(-2,-5)$){$s_5$} ; 

\node [] (r1) at ($(T)+(-2,-0.8)$){$r_1$};
\draw[->,thick,>=latex] (T) to[out=-90,in=90]  ($(s1)+(0,0.7)$);

\node [] (r2) at ($(T)+(0.4,-1.5)$){$r_2$};
\draw[->,thick,>=latex] (T) to[out=-90,in=90]  ($(s2)+(0,0.7)$);

\node [] (r3) at ($(T)+(2,-.8)$){$r_3$};
\draw[->,thick,>=latex] (T) to[out=-90,in=90] ($(s3)+(0,0.7)$);

\node [] (r4) at ($(s2)+(1,-1)$){$r_4$};
\draw[->,thick,>=latex] (s2) to[out=-10,in=90] (1.5,-3.5)  to[out=-90,in=90]   ($(s4)+(-0.5,0.55)$); 
\draw[->,thick,>=latex] (s3) to[out=170,in=90] (1.5,-3.5) to[out=-90,in=90,line width=0.6mm] ($(s4)+(-0.5,0.55)$);

\node [] (r5) at ($(s4)+(-2,0.3)$){$r_5$};
\draw[->,thick,>=latex] (s4) to[out=180,in=0]  ($(s5)+(0.7,0)$);

\node [] (r6) at ($(T)+(-2,-3)$){\jordi{$r_6$}};
\draw[->,thick,>=latex,ForestGreen] (T) to[out=-90,in=90]  ($(s5)+(0,0.7)$);

\node [] (r7) at ($(s3)+(-1,-1)$){\jordi{$r_7$}};
\draw[->,thick,>=latex,ForestGreen] (s2) to[out=10,in=90] (2.5,-3.5) to[out=-90,in=90] ($(s4)+(0.5,0.55)$);
\draw[->,thick,>=latex,ForestGreen] (s3) to[out=190,in=90] (2.5,-3.5) to[out=-90,in=90] ($(s4)+(0.5,0.55)$);

\node [] (r8) at ($(T)+(-2,-2)$){\jordi{$r_8$}};
\draw[->,thick,>=latex,ForestGreen] (T) to[out=-90,in=0]  ($(s1)+(0.7,0)$);

\end{tikzpicture}
\caption{Expanded example.}
\end{figure}}

\subsection{Formalising opinions}
\label{subsec:formalAssessments}

Now we address the formalisation of the opinions put forward by participants in a debate. We consider that opinions can be held both about statements and relationships. We therefore define two functions that capture the opinions of individuals: (i) a \emph{valuation function} over statements; and (ii) an \emph{acceptance function} over relationships. On the one hand, a valuation function conveys the subjective value that an individual places on each statement. On the other hand, an acceptance function expresses the subjective plausibility that an individual assigns to each relationship, representing a reasoning, connecting statements. Formally:

\begin{defi}[Valuation function]
Given a DRF $\langle \sen,\rel, T \rangle$, a \emph{valuation function} $v: \sen\longrightarrow I$ maps each statement to a value in $I=[-a,a]$, $a\in \R^+$.
\end{defi}

Given a statement $s \in \sen$: if $v(s)=a$ we say that $s$ counts on \emph{full positive valuation}; if $v(s)=-a$ we say that $s$ counts on \emph{full negative valuation}; and if $v(s)=0$ we say that $s$ has \emph{neutral valuation}.

\begin{defi}[Acceptance function]
Given a DRF $\langle \sen,\rel, T \rangle$, an \emph{acceptance function} maps each relationship to a value in $I^+$, $w: \rel\  \longrightarrow I^+=[0,a]$.
\end{defi}

Given a relationship $r\in\rel$ and an acceptance function $w$, we will refer to the value $w(r)$ as the \emph{acceptance degree} of $r$. If $w(r)=a$ we say that the acceptance function expresses \emph{full agreement} with the relationship, whereas if $w(r)=0$ we say that it expresses \emph{full disagreement}. Without loss of generality, henceforth we will set $a$ to 1, and hence $I=[-1,1]$ and $I^+=[0,1]$.

Considering our running example, graphically represented in figure \ref{pic:ex_DRF}, figures \ref{pic:ex_valuations} and \ref{pic:ex_acceptances} show the valuation functions and acceptance functions of agents 1, 2, and 3: $v_1$, $v_2$ and $v_3$ encode agents' valuations on statements, while $w_1$, $w_2$ and $w_3$ encode agents' acceptances of relationships.
We consider now the description of agents' opinions in section \ref{subsub:interpretation} to exemplify how they translate into valuations and acceptances. Thus, for instance, agent one is ``highly positive'' on the target $\tau$ ($v_1(\tau) = 0.9$), but neutral regarding statement $s_1$ ($v_1(s_1) = 0$). Furthermore, agent one considers that the plausibility of relationship $r_1$ is ``little'' ($w_1(r_1) = 0.2$).


\begin{figure}[H]
\centering
\begin{tikzpicture}[every node/.append style={inner sep=3mm},every place/.style={line width=0.3mm}]

\node[rectangle,draw,line width=0.7mm] (T) {\large{$\tau$}};
\node(s1) [circle,draw] at ($(T)+(-4,-2.5)$){$s_1$} ; 
\node(s2) [circle,draw] at ($(T)+(0,-2.5)$){$s_2$} ; 
\node(s3) [circle,draw] at ($(T)+(4,-2.5)$){$s_3$} ; 
\node(s4) [circle,draw] at ($(T)+(2,-5)$){$s_4$} ; 
\node(s5) [circle,draw] at ($(T)+(-2,-5)$){$s_5$} ; 

\draw[->,thick,>=latex] (T) to[out=-90,in=90]  ($(s1)+(0,0.7)$);

\draw[->,thick,>=latex] (T) to[out=-90,in=90]  ($(s2)+(0,0.7)$);

\draw[->,thick,>=latex] (T) to[out=-90,in=90] ($(s3)+(0,0.7)$);

\draw[->,thick,>=latex] (s2) to[out=0,in=90] (2,-3.5)  to[out=-90,in=90]   ($(s4)+(0,0.6)$); 
\draw[->,thick,>=latex] (s3) to[out=180,in=90] (2,-3.5) to[out=-90,in=90,line width=0.6mm] ($(s4)+(0,0.6)$);

\draw[->,thick,>=latex] (s4) to[out=180,in=0]  ($(s5)+(0.7,0)$);

\draw[->,thick,>=latex] (T) to[out=-90,in=0]  ($(s1)+(0.7,0)$);

\node [] (vT) at ($(T)+(1.6,0)$) {\scalebox{0.8}{$\arraycolsep=1pt\def\arraystretch{1}
\begin{array}{cc}
 v_1(\tau) =&0.9\\
     v_2(\tau) =&-0.5\\
    v_3(\tau) =&-0.5
\end{array}$}};

\node [] (v1) at ($(s1)+(-1.6,0)$) {\scalebox{0.8}{$\arraycolsep=1pt\def\arraystretch{1}
\begin{array}{cc}
 v_1(s_1) =& 0\\
     v_2(s_1) =&-1\\
    v_3(s_1) =&0
\end{array}$}};

\node [] (v2) at ($(s2)+(-1.6,-0.8)$) {\scalebox{0.8}{$\arraycolsep=1pt\def\arraystretch{1}
\begin{array}{cc}
 v_1(s_2) =&0.7\\
     v_2(s_2) =& 1\\
    v_3(s_2) =&-0.8
\end{array}$}};

\node [] (v3) at ($(s3)+(1.6,0)$) {\scalebox{0.8}{$\arraycolsep=1pt\def\arraystretch{1}
\begin{array}{cc}
 v_1(s_3) =&1\\
     v_2(s_3) =&0.5\\
    v_3(s_3) =&0.5
\end{array}$}};

\node [] (v4) at ($(s4)+(1.6,0)$) {\scalebox{0.8}{$\arraycolsep=1pt\def\arraystretch{1}
\begin{array}{cc}
 v_1(s_4) =&1\\
     v_2(s_4) =&1\\
    v_3(s_4) =&1
\end{array}$}};

\node [] (v5) at ($(s5)+(-1.6,0)$) {\scalebox{0.8}{$\arraycolsep=1pt\def\arraystretch{1}
\begin{array}{cc}
 v_1(s_5) =&-1\\
     v_2(s_5) =&1\\
    v_3(s_5) =&-1
\end{array}$}};

\end{tikzpicture}
\caption{Agents' valuation functions.}
\label{pic:ex_valuations}
\end{figure}
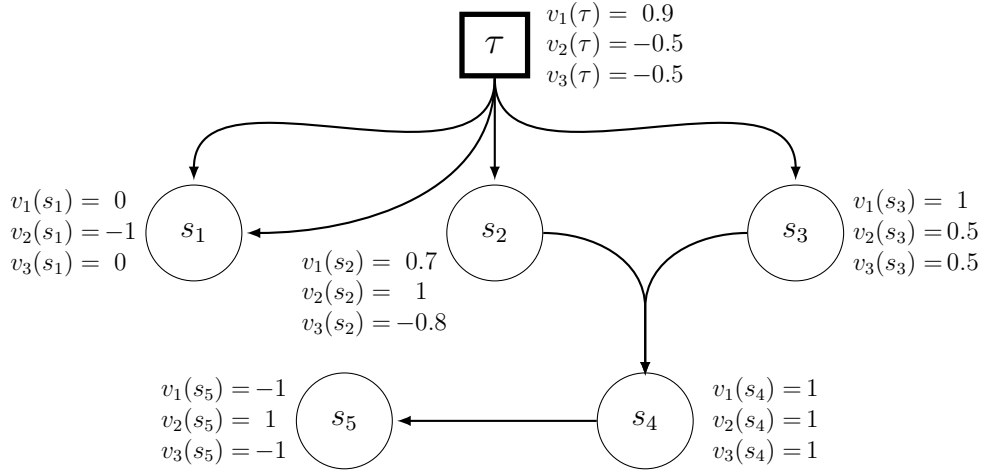

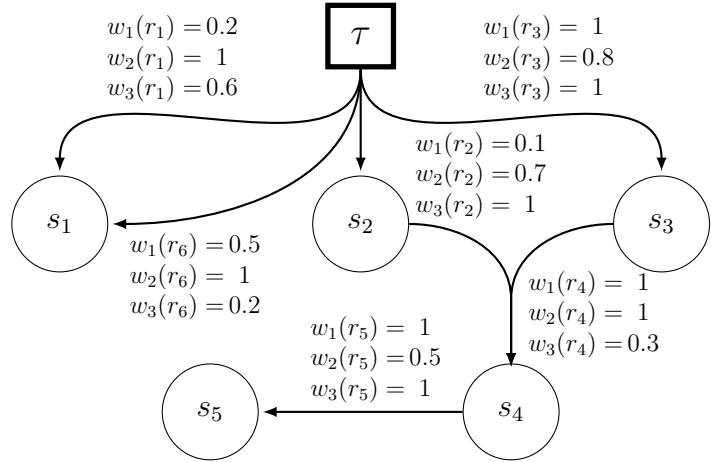
\begin{figure}[H]
\centering
\begin{tikzpicture}[every node/.append style={inner sep=3mm},every place/.style={line width=0.3mm}]

\node[rectangle,draw,line width=0.7mm] (T) {\large{$\tau$}};
\node(s1) [circle,draw] at ($(T)+(-4,-2.5)$){$s_1$} ; 
\node(s2) [circle,draw] at ($(T)+(0,-2.5)$){$s_2$} ; 
\node(s3) [circle,draw] at ($(T)+(4,-2.5)$){$s_3$} ; 
\node(s4) [circle,draw] at ($(T)+(2,-5)$){$s_4$} ; 
\node(s5) [circle,draw] at ($(T)+(-2,-5)$){$s_5$} ; 

\draw[->,thick,>=latex] (T) to[out=-90,in=90]  ($(s1)+(0,0.7)$);

\draw[->,thick,>=latex] (T) to[out=-90,in=90]  ($(s2)+(0,0.7)$);

\draw[->,thick,>=latex] (T) to[out=-90,in=90] ($(s3)+(0,0.7)$);

\draw[->,thick,>=latex] (s2) to[out=0,in=90] (2,-3.5)  to[out=-90,in=90]   ($(s4)+(0,0.6)$); 
\draw[->,thick,>=latex] (s3) to[out=180,in=90] (2,-3.5) to[out=-90,in=90,line width=0.6mm] ($(s4)+(0,0.6)$);

\draw[->,thick,>=latex] (s4) to[out=180,in=0]  ($(s5)+(0.7,0)$);

\draw[->,thick,>=latex] (T) to[out=-90,in=0]  ($(s1)+(0.7,0)$);

\node [] (w1) at ($(T)+(-2.5,-0.3)$) {\scalebox{0.8}{$\arraycolsep=1pt\def\arraystretch{1}
\begin{array}{cc}
 w_1(r_1) =&0.2\\
     w_2(r_1) =&1\\
    w_3(r_1) =&0.6
\end{array}$}};

\node [] (w2) at ($(s2)+(1.6,0.65)$) {\scalebox{0.8}{$\arraycolsep=1pt\def\arraystretch{1}
\begin{array}{cc}
 w_1(r_2) =&0.1\\
     w_2(r_2) =&0.7\\
    w_3(r_2) =&1
\end{array}$}};

\node [] (w3) at ($(T)+(2.5,-0.3)$) {\scalebox{0.8}{$\arraycolsep=1pt\def\arraystretch{1}
\begin{array}{cc}
 w_1(r_3) =&1\\
     w_2(r_3) =&0.8\\
    w_3(r_3) =&1
\end{array}$}};

\node [] (w4) at ($(s4)+(1.1,1.3)$) {\scalebox{0.8}{$\arraycolsep=1pt\def\arraystretch{1}
\begin{array}{cc}
 w_1(r_4) =&1\\
     w_2(r_4) =&1\\
    w_3(r_4) =&0.3
\end{array}$}};

\node [] (w5) at ($(s4)+(-1.8,0.7)$) {\scalebox{0.8}{$\arraycolsep=1pt\def\arraystretch{1}
\begin{array}{cc}
 w_1(r_5) =&1\\
     w_2(r_5) =&0.5\\
    w_3(r_5) =&1
\end{array}$}};

\node [] (w6) at ($(s1)+(1.8,-0.7)$) {\scalebox{0.8}{$\arraycolsep=1pt\def\arraystretch{1}
\begin{array}{cc}
 w_1(r_6) =&0.5\\
     w_2(r_6) =&1\\
    w_3(r_6) =&0.2
\end{array}$}};

\end{tikzpicture}
\caption{Agents' acceptance functions.}
\label{pic:ex_acceptances}
\end{figure}

At this point, we are ready to formally introduce the notion of individual opinion over a $DRF$.
\begin{defi}[Opinion]
Given a $DRF=\langle \sen,\rel,T\rangle$, an \emph{opinion} over the $DRF$ is a pair $O=(v,w)$ such that $v$ is a valuation function and $w$ is an acceptance degree.
\end{defi}
For practical and realistic purposes, we assume that for each relationship at least one agent values it different from 0. Otherwise, it would be the same, in practical terms, to have or not have the relationship in the debate.

Henceforth, we shall note the class of all opinions over a $DRF$ as $\mathbb{O}(DRF)$.

As depicted in figures \ref{pic:ex_valuations} and \ref{pic:ex_acceptances}, each agent $i$ involved in a debate will have its individual opinion $O_i=(v_i,w_i)$. Our goal will be to compute a collective opinion from the opinions issued by the agents.

\section{Characterising coherent opinions}\label{sec:Expectation_Coherence}
\

Previous work on the formal modelling of debates has placed restrictions on the opinions that individuals can put forward. For example, \cite{Awad:2015:JAM} interprets the opinions expressed by individuals as labels, in the sense of \cite{baroni:caminada:giacomini:ker}, for the arguments that they are expressing opinions about. Thus, an argument can be labelled $\lin$, meaning that the individual thinks that it holds, $\lout$, meaning that the individual thinks it does not hold, or $\lundec$, meaning that the individual is not sure whether it holds or not. These labellings are restricted to be \emph{complete} labellings \cite{baroni:caminada:giacomini:ker}, broadly meaning that they conform to a notion of rationality where arguments are $\lout$ if they are attacked by arguments that have been established to be $\lin$, and are $\lin$ if they are only attacked by arguments that are $\lout$. We believe that the restrictions imposed by these notions are too restrictive for modelling human debates, as humans may express opinions that are far from rational.

%
%

Instead, we impose weaker conditions for an individual opinion to be classified as reasonable or \textit{coherent}, along the lines of our former work in \cite{Ganzer2018}. Hence, given a statement we contrast opinions expressed about that statement, the \emph{direct opinion}, with the opinions expressed about the immediate descendants of the statement, what we call the \emph{indirect opinion}, and look for ways in which these may be made somewhat consistent.

Informally, what we do is the following. First, we compute an estimated opinion for a statement based on the indirect opinion about it. Then we say that a set of opinions about a statement are coherent if the estimated opinion for the statement aligns with the direct opinion about the statement. This will be the case when the opinion (valuations) about the descendants is close to the overall opinion (valuation) about the statement. Considering our example in figure \ref{pic:ex_valuations} again, consider statement $\tau$, its descendants ($s_1$, $s_2$, and $s_3$), and the opinion of agent 2 ($v_2$). We observe that although the direct opinion about $\tau$ is rather negative ($
v_2(\tau)=-0.5$), the valuations for its descendants are diverse: while the valuation for $s_1$ is also rather negative ($v_2(s_1)=-1$), and hence in line with $\tau$, the valuations on the other descendants are rather positive ($v_2(s_2) = 1$ and $v_2(s_3) = 0.5$), and hence not in line with $\tau$. Thus, at first sight\footnote{Note that we are not considering acceptances at this point.} it would seem that the overall estimated opinion is not in line with the direct opinion. 

In what follows, we first formalise our notion of  \emph{estimation} as an aggregated measure formed from the indirect opinion about a statement --- i.e., the collection of values on the descendants and their relationships. This will consider valuations and acceptance degrees related to a statement and its relationships
so that the more accepted a relation between a statement and its descendants, the more important the opinion about that descendant. Thereafter, we will formalise our notion of coherence by measuring how close the direct opinion about a statement is to the estimated opinion about that statement.

First, we introduce some concepts and notations that will aid us on later steps.
Given a $DRF=\langle \sen,\rel,T\rangle$, we define the \emph{set of relationships from $s\in \sen$} as the set of relationships having $s$ in the set of initial statements. Formally, 
\begin{equation} \label{eq:descendants}
R^+(s)=\{r=(\Sigma,s',c)\in\rel\ |\  s\in\Sigma\}.
\end{equation}
We will use the term \emph{descendants of a statement $s$, denoted by $D(s)$, to refer to} any statement $s_r$ connected to $s$ by a relationship $r$ that has $s$ as one of the initial statements and $s_r$ as final statement. Formally,
\[
D(s)=\{s_r\in \sen\  |\  r=(\Sigma,s_r,c)\in R^+(s)\}.
\]
Next, we define the concept that connects direct and indirect opinion in order to characterise our notion of coherence. We will use an \emph{estimation function} to compute an estimate of the direct opinion using the values gathered for the indirect opinion. 



\begin{defi} Given a $DRF=\langle \sen,\rel,T\rangle$ and $O=(v,w)$ and opinion over the $DRF$, the \emph{estimation function} is a valuation function:

$$\begin{array}{cccc}
   e:  & \sen & \longrightarrow & I \\
       & s & \longmapsto &   e(s)=f(IO(s))
\end{array}$$

where $IO(s)=\{(v(s_r),w(r))\ |\ r\in R(s),\ s_r\in D(s)\}$ and $f$ is a projection:
$$f:\ IO(s)\longrightarrow I$$

that computes an approximate value for $s$ from the valuations and acceptance degrees in the indirect opinion.

\begin{figure}[H]
\[
\begin{tikzpicture}[node distance=4cm, auto]
\node (A) {$e: \sen$};
\node (C) [below right = 1cm and 2cm of A] {$IO(s)$};
\node(B) [above right =1cm and 2cm of C] {$I$};
\draw[->,line width=0.4mm](A) to node {}(B);

\draw[->](A) to node [below=0.5ex] {}(C);
\draw[->](C) to node [below=0.5ex] {}(B);
\end{tikzpicture}
\]
\caption{A visualization of the estimation function.}
\label{fig:estimation}
\end{figure}
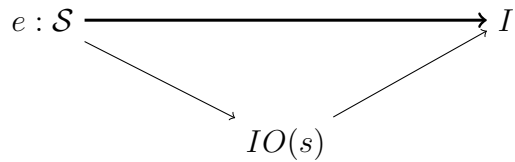
 
\end{defi}
In other words, the estimation function computes an estimated value for a statement using the valuations and acceptance degrees for indirect opinion about that statement. Figure~\ref{fig:estimation} is a diagrammatic representation of the estimation function.
This definition is designed to be an abstract definition that allows for many estimation functions to be defined to compute different approximations for the direct opinion. It therefore specifies a broad family of estimation functions rather than any specific function.

For the rest of this paper we will use a  specific estimation function to compute the estimate opinion on the results later on. We will use the weighted average of the valuations on the descendants, where the weights are the acceptance degrees on the relations leading to each descendant. In this manner, the more accepted a relation, the more valued the opinion on the descendant. Formally,
\[
e(s)=
  \left\{\begin{array}{cc}
   v(s)       & \text{if } R^+(s)=\emptyset \textrm{ or } \sum_{r\in R^+(s)}w(r)=0,  \\\\
   
   \frac {\sum_{r\in R^+(s)}w(r)v(s_r)} {\sum_{r\in R^+(s)}w(r)}  & \text{otherwise.}\\
  \end{array}\right.
\]
%
%
%
\comment{\jar{
\begin{align*}
\bar{d}_O(s)=&
  \begin{cases}
   $d_O(s)$        & \text{if } $D(s)=\emptyset $\\
   $\frac {\sum_{r\in R^+(s)}w(r)v(s_r)} {\sum_{r\in R^+(s)}w(r)}$  & \text{otherwise}\\
  \end{cases}
\end{align*}
}}
%
%
%

Notice that when a statement has no descendants we take the direct option as the estimated opinion, thus we cannot gather any value from an empty indirect opinion. Notice also that this specification of estimated opinion can be regarded as a general approach to compute the direct opinion using the indirect opinion. This is because we are allowing the relationships between statements to represent any kind of reasoning, so we cannot specify a specific behaviour for the estimated value. 

Informally, an opinion is characterised as coherent for a given statement when the value assigned by the participant (issuing the opinion) to the statement  (i.e., its direct opinion) is aligned with the values and plausibility assigned to its descendants (i.e., its estimate opinion). Furthermore, given the continuous values allowed on the opinion, we can choose the degree of coherence by using a parameter $\epsilon$. Formally:
\begin{defi}[Coherence]
\label{def:coherence}
Consider a $DRF=\langle \sen, \rel, T\rangle$ and an $\epsilon\in(0,1)$\footnote{We choose the interval (0,1) for the value of $\epsilon$ as the minimum interval that guarantees that if the direct opinion is 1 (or -1) then an opinion cannot be classified as coherent when the estimation function value is of the opposite sign, i.e., $e(s)\not\leq 0$ (or $e(s)\not\geq 0$ respectively)}  difference. We say that opinion $O=(v,w)$ is \emph{$\epsilon$-coherent} on $s\in\sen$ when 
$$|v(s)-e(s)|< \epsilon.$$

\end{defi}
In general, an opinion $O$ will be $\epsilon$-coherent if it is $\epsilon$-coherent for every statement in $\sen$. We will notate  as $\mathbb{C}_\epsilon(DRF)$ the class of all the $\epsilon$-coherent opinions. Thus, if $O$ is an $\epsilon$-coherent opinion then $O\in\mathbb{C}_\epsilon(DRF)$.

\begin{ex}
Following the example, now we can compare the values from the expectation function and the actual value given by participant 1 to each statement, its direct opinion, see figure \ref{pic:ex_coherence}. We can see that if $\epsilon\in(0.3,1)$ then the opinion of participant 1 for the statements $s_1,s_2,s_3$ and $s_5$ is $\epsilon$-coherent but not for statement $s_4$ due to the difference between direct opinion and estimated value, which is the maximum possible. Because of this statement $s_4$, the opinion of participant 1 cannot be classified as $\epsilon$-coherent for any $\epsilon\in(0,1)$.

\begin{figure}[H]
\centering
\begin{tikzpicture}[every node/.append style={inner sep=3mm},every place/.style={line width=0.3mm}]

\node[rectangle,draw,line width=0.7mm] (T) {\large{$\tau$}};
\node(s1) [circle,draw] at ($(T)+(-4,-2.5)$){$s_1$} ; 
\node(s2) [circle,draw] at ($(T)+(0,-2.5)$){$s_2$} ; 
\node(s3) [circle,draw] at ($(T)+(4,-2.5)$){$s_3$} ; 
\node(s4) [circle,draw] at ($(T)+(2,-5)$){$s_4$} ; 
\node(s5) [circle,draw] at ($(T)+(-2,-5)$){$s_5$} ; 

\draw[->,thick,>=latex] (T) to[out=-90,in=90]  ($(s1)+(0,0.7)$);

\draw[->,thick,>=latex] (T) to[out=-90,in=90]  ($(s2)+(0,0.7)$);

\draw[->,thick,>=latex] (T) to[out=-90,in=90] ($(s3)+(0,0.7)$);

\draw[->,thick,>=latex] (s2) to[out=0,in=90] (2,-3.5)  to[out=-90,in=90]   ($(s4)+(0,0.6)$); 
\draw[->,thick,>=latex] (s3) to[out=180,in=90] (2,-3.5) to[out=-90,in=90,line width=0.6mm] ($(s4)+(0,0.6)$);

\draw[->,thick,>=latex] (s4) to[out=180,in=0]  ($(s5)+(0.7,0)$);

\draw[->,thick,>=latex] (T) to[out=-90,in=0]  ($(s1)+(0.7,0)$);

\node [] (vT) at ($(T)+(2,0)$) {\scalebox{0.8}{$v_1(\tau)-e_1(\tau)=0.3$}};

\node [] (v1) at ($(s1)+(-0,-1)$) {\scalebox{0.8}{$v_1(s_1)-e_1(s_1)=0$}};

\node [] (v2) at ($(s2)+(0,-1)$) {\scalebox{0.8}{$v_1(s_2)-e_1(s_2)=-0.3$}};

\node [] (v3) at ($(s3)+(0,-1)$) {\scalebox{0.8}{$v_1(s_3)-e_1(s_3)=0$}};

\node [] (v4) at ($(s4)+(0,-1)$) {\scalebox{0.8}{$v_1(s_4)-e_1(s_4)=2$}};

\node [] (v5) at ($(s5)+(0,-1)$) {\scalebox{0.8}{$v_1(s_5)-e_1(s_5)=0$}};

\end{tikzpicture}
\caption{Coherence of Agent 1}
\label{pic:ex_coherence}
\end{figure}
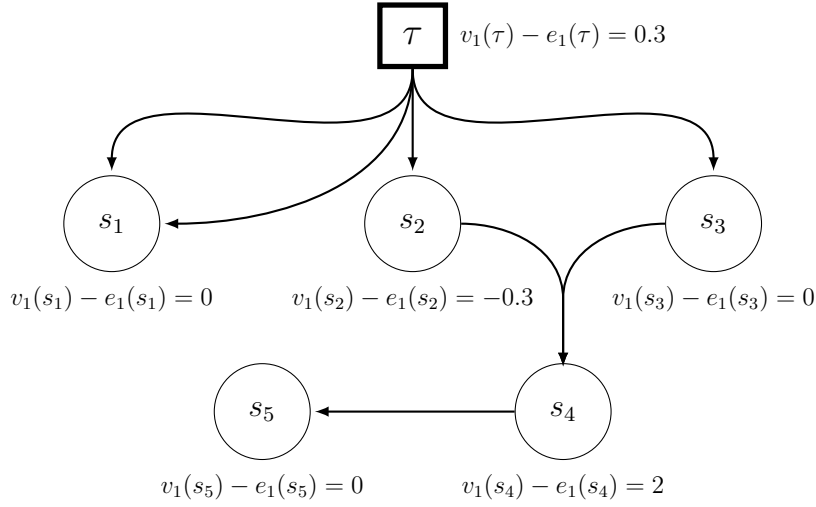

\end{ex}

\section{Formalising the collective decision making problem} \label{sec:Aggregation}

As stated above, our goal is to help agents reach a collective decision on target statements. This corresponds to step 4 of the protocol in Section \ref{subsec:collectiveWorkflow} 
and Figure \ref{pic:DebateSteps} (as well as to the third stage in Figure \ref{pic:RelationaModelComp}). 
In Section \ref{subsec:opinionProblem} we cast our goal as an opinion aggregation problem. We propose to solve such problem using an aggregation function that synthesises a single opinion out of all agents' opinions. 
Although opinions can be aggregated in different ways, here we follow \cite{Ganzer2018} in requiring that the outcome of an aggregation must be fair. 
In particular, Section \ref{subsec:socialChoiceProperties} introduces desirable social choice properties to help analyse and compare opinion aggregation functions.

\subsection{The opinion aggregation problem}
\label{subsec:opinionProblem}

First, we 
define our notion of an opinion profile, which brings together the opinions of the individuals involved in a debate. From hereon we use the term ``agent'' along with the term ``individual'' to refer to the participants in the debate. 
\begin{defi}[Opinion profile]
Let $Ag=\{1,\dotsc,n\}$ be a set of $n$ agents and a $DRF=\langle \sen,\rel,T\rangle$. An \emph{opinion profile} is a collection of opinions $(O_1=( v_1,w_1),\dotsc, O_n=( v_n,w_n) )\in \mathbb{O}(DRF)^n$ over the $DRF$ such that $O_i=(v_i,w_i)$ stands for the opinion of agent $i$.
\end{defi}
The problem at hand is how to aggregate the opinions in an opinion profile to produce a single opinion so  that single opinion is a reasonable summary of the opinions in the opinion profile. If the opinion profile represents the views expressed by individuals in a  debate, the combination should represent the collective opinion of all the individuals. The opinion aggregation function, which we formalise below, is the mechanism for establishing this collective opinion.
\begin{defi}[Opinion aggregation function]
Given a $DRF$ and a set of $n$ agents $Ag$, a function $F:\mathcal{D}\subset \mathbb{O}(DRF)^n\longrightarrow \mathbb{O}(DRF)$ mapping an opinion profile to a single opinion is called an \emph{opinion aggregation function}. Given an opinion profile $P$ in the domain $\mathcal{D}$, $F(P)$ is called the \emph{collective opinion} by $F$ and it will be noted as $F(P)=( v_{F(P)},w_{F(P)})$.
\end{defi}
Thus, in terms of the components of our model, the collective opinion output by an opinion aggregation function combines the collective valuations over statements and the collective acceptances over the relationships linking them. In section \ref{sec:OpAggFunc} we define specific opinion aggregation functions that compute a collective opinion. Before that we introduce the properties that we will use to analyse the aggregation functions.

\subsection{Social choice properties}
\label{subsec:socialChoiceProperties}

Social choice theory provides formal properties to characterise aggregation methods in terms of outcome  fairness \cite{dietrich:2007:generalised}.
In what follows, we formally adapt some of the desirable social properties of an aggregation function introduced in \cite{Awad:2015:JAM} and \cite{Ganzer2018}.
Besides adapting properties, we define some novel properties that characterise aggregation functions motivated by the fact that here we are considering opinions to be continuous-valued in contrast to the discrete-valued opinions considered in \cite{Awad:2015:JAM} and \cite{Ganzer2018}.

First, we characterise aggregation functions in terms of the opinion profiles that they can take as input. Thus, we adapt from \cite{Awad:2015:JAM} the notion of \emph{exhaustive domain} to characterise opinion aggregation functions that are defined for any opinion profile. Thereafter, we modify this property to limit an opinion aggregation function to operate with $\epsilon$-coherent opinion profiles.

\begin{description}
\item{\textbf{Exhaustive Domain (ED) .}} An opinion aggregation function $F$ satisfies \emph{exhaustive domain} if its domain is $\mathcal{D} = \mathbb{O}(DRF)^n$, namely if the function can operate over  all profiles.

\item{\textbf{$\epsilon$-Coherent Domain ($\epsilon$-CD).}}
An opinion aggregation function $F$ satisfies \emph{$\epsilon$-coherent domain} if its domain $\mathcal{D}$ contains all $\epsilon$-coherent opinion profiles, namely $\mathbb{C}_{\epsilon}(DRF)^n\subseteq\mathcal{D}$.
\end{description}
Note that we will sometimes refer to $\epsilon$-Coherent Domain as ``coherent domain''.
\begin{lemma}
\label{lemma:ED-CD}
An opinion aggregation function satisfying exhaustive domain also satisfies $\epsilon$-coherent domain.
\end{lemma}
\begin{proof}
Straightforward, since an aggregation function taking any opinion profile will also take $\epsilon$-coherent opinion profiles.
\end{proof}

Moreover, we also define \emph{collective $\epsilon$-coherence} as a property characterising opinion aggregation functions that produce $\epsilon$-coherent collective opinions. Therefore, our notion of collective $\epsilon$-coherence here is more relaxed than the crisp notion of coherence in \cite{Ganzer2018}.

\begin{description}
\item{\textbf{Collective $\epsilon$-coherence ($\epsilon$-CC).}} An opinion aggregation function $F$ satisfies \emph{$\epsilon$-collective coherence} if for all $P\in\mathcal{D}$, then $F(P)\in \mathbb{C}_\epsilon(DRF)$.
\end{description}
In accordance with \cite{Ganzer2018}, here we consider $\epsilon$-CC as the most desirable property that can be satisfied by an aggregation function, since collective coherence is the foundation of the acceptability of collective decisions \cite{thagard2002coherence}. Notice also that, as in \cite{Ganzer2018}, collective coherence can be regarded as the counterpart of the notion of \emph{collective rationality} in \cite{Awad:2015:JAM}.

Next, anonymity and non-dictatorship characterise the importance of the agents involved in a debate that yields a collective opinion. On the one hand, anonymity is a social choice property requiring that the opinions of all the agents involved in a debate are considered to be equally significant. On the other hand, non-dictatorship requires that no agent overrules the opinions of the rest of agents.

\begin{description}
\item{\textbf{Anonymity (A)}} Let $P=(O_1,\dotsc,O_n)$ be an opinion profile in $\mathcal{D}$, $\sigma$ a permutation over $Ag$, and $P'=(O_{\sigma(1)},\dotsc,\allowbreak O_{\sigma(n)})$ the opinion profile resulting from applying $\sigma$ over $P$. An opinion aggregation function $F$ satisfies \emph{anonymity} if $F(P)=F(P')$.
\end{description}

\begin{description}
\item{\textbf{Non-Dictatorship (ND).}} 
An opinion aggregation function $F$ satisfies \emph{non-dictatorship} if no agent $i \in Ag$ satisfies that $F(P)=O_i$ for every opinion profile $P\in \mathcal{D}$.
\end{description}
Notice that non-dictatorship is a weaker version of anonymity since it follows directly 
from it --- any aggregation function that satisfies anonymity will satisfy non-dictatorship.

Now we consider how an opinion aggregation function behaves when agents agree on their opinions about statements.
Unanimity is the social choice property that characterises the behaviour of aggregation functions when there is agreement among agents' opinions. Here we will adapt the classic notion of unanimity in \cite{Awad:2015:JAM}, here called \emph{narrow unanimity}, and the notion of \emph{endorsed unanimity} from \cite{Ganzer2018} that helps exploit dependencies between statements. Since the notion of narrow unanimity is rather rigid, we also define two new, more flexible unanimity properties. Therefore, overall we define a family of unanimity properties, and below we study the relationships between them.

\begin{description}
\item{\textbf{Narrow Unanimity (NU).}}
Let $P=(O_1,\dotsc,O_n)$ be an opinion profile, where $ P \in \mathcal{D}$. An opinion aggregation profile $F$ satisfies \emph{narrow unanimity} if, for any $s\in S$ such that $v_i(s)=\lambda$ for all $i\in\{1,\dotsc,n\}$, then $v_{F(P)}(s)=\lambda$ holds.
\end{description}
Narrow unanimity defines unanimity as being when all agents share the very same opinion. While this is very possible in settings where agents only have a few discrete possibilities for expressing their opinion, as in \cite{Awad:2015:JAM,Ganzer2018}, it is not likely to occur in the setting we are studying here, where opinions can take a wide range of values. As a result, we propose some relaxed variations which are more useful for the setting we consider. First, we say that \emph{sided unanimity} will hold when, for each statement, either all opinions on it are positive or negative. Formally,
\begin{description}
\item{\textbf{Sided Unanimity (SU).}}
Let $P=(O_1,\dotsc,O_n)$ be an opinion profile, where $P \in \mathcal{D}$. An opinion aggregation profile $F$ satisfies \emph{sided-unanimity} if for every $s\in S$:
\begin{itemize}
\item  if $v_i(s)>0$ for all $i\in Ag$ then $v_{F(P)}(s)>0$; 
\item  if $v_i(s)<0$ for all $i\in Ag$ then $v_{F(P)}(s)<0$. 
\end{itemize} 
\end{description}
We also find a weaker version of sided unanimity to be worth distinguishing:
\begin{description}
\item{\textbf{Weak Unanimity (WU).}}
Let $P=(O_1,\dotsc,O_n)$ be an opinion profile, where $P \in \mathcal{D}$. An aggregation profile $F$ satisfies \emph{Weak unanimity} if, for every  $s\in S$:  
\begin{itemize}
\item  if $v_i(s)=1$ for all $i\in Ag$ then $v_{F(P)}(s)>0$; 
\item  if $v_i(s)=-1$ for all $i\in Ag$ then $v_{F(P)}(s)<0$. 
\end{itemize}  
\end{description}


Although WU requires that all agents agree on fully positive (1) or fully negative (-1) valuations on statements, it does not require that the output of the opinion aggregation function takes on those same values, as required by narrow unanimity. This property has value when translating valuations expressed in a discrete model such as those in \cite{Awad:2015:JAM,Ganzer2018} into our model, and so has value in allowing us to relate our model to those which came before.

From the definitions above, it follows that the three notions of unanimity are related.

\begin{prop}[Unanimity relationships]
\label{prop:unanimity}
If an opinion aggregation function satisfies Sided Unanimity then it satisfies Weak Unanimity. If an opinion aggregation function satisfies Narrow Unanimity then it satisfies Weak Unanimity.
\end{prop}
\begin{proof}
Clearly, if an aggregation function cannot hold the sign of the aggregation when the assumptions of Weak Unanimity are satisfied, then it is straightforward to see that will fail to satisfy Sided Unanimity. 

Obviously, if an aggregation function returns a value $\lambda$ when all the agents  valued the statement as $\lambda$, it will return the value when it is $1$ or $-1$.

\end{proof}

Below, we will show that we need further assumptions to prove that NU implies SU.

As a final unanimity property, we adapt the notion of endorsed unanimity from \cite{Ganzer2018} to consider unanimity based on indirect opinions. In short, an opinion aggregation function will satisfy \emph{endorsed unanimity} if, for each statement, the collective opinion on the statement is in line with the unanimous indirect opinion on it. Formally,
\begin{description}
\item {\textbf{Endorsed  Unanimity (EU).}}
Let $P=(O_1,\dotsc,O_n)$ be an opinion profile such that $ P \in \mathcal{D}$. An aggregation profile $F$ satisfies \emph{endorsed unanimity} if for every $s\in\sen$: 
\begin{enumerate}[(i)]
\item if  $v_i(s_d)=1$ for any $i\in Ag$ and $s_d\in D(s)$ (called \emph{full positive support}), then $v_{F(P)}(s)>0$; and
\item if  $v_i(s_d)=-1$ for any $i\in Ag$ and $s_d\in D(s)$ (called \emph{full negative support}), then $v_{F(P)}(s)<0$.
\end{enumerate}
\end{description}
Note that, this property is closely related to the notion of coherence. We will show below that restricting the domain to coherent opinion profiles will help to fulfil coherence.

Next, we introduce monotonicity properties to study how the result of an opinion aggregation function changes as opinions change. First, we adapt the notion of monotonicity from \cite{Awad:2015:JAM}: if some of the direct opinions about a statement increase (or decrease) the collective opinion should increase (or decrease) accordingly.
\begin{description}
\item{\textbf{Monotonicity (M) } }
Let $s\in S$ be a statement, and $P=(O_1,\dotsc,O_n)$ and $P'=(O'_1,\dotsc,O'_n)$ such that for every $i$ $v_i(s)\leq v'_i(s)$. We say that an opinion aggregation function $F$ is \emph{monotonic} if $v_{F(P)}(s)\leq v_{F(P')}(s)$.
\end{description}
Notice that M only takes into account the direct opinion about each statement. Since we aim at handling opinion aggregation functions that merge both direct and indirect opinions, we define a variation of monotoniciy taking this into account.
To do this we adapt the notion of \emph{familiar monotonicity}\footnote{The name derives from the fact that this form of monotonicity takes into account opinion about the descendents of a statement which make up its family.} in \cite{Ganzer2018}. In our case, familiar monotonicity requires that when the direct opinion on a statement increases, the collective opinion does not decrease provided that the opinions on the descendants of the statement do not change either. Formally:


\begin{description}
\item{\textbf{Familiar Monotonicity (FM).}} 
Let $s\in S$ be 
a statement, and $P=(O_1,\dotsc,O_n)$ and $P'=(O'_1,\dotsc,O'_n)$ such that for every opinion $i$ satisfies that $v_i(s)\leq v'_i(s)$, and, $w_i(r)=w'_i(r)$ and $v_i(s_r)=v'_i(s_r)$ for every relationship $r \in R(s)$ and its associated descendant $s_r\in D(s)$. We say that an opinion aggregation function $F$ satisfies FM if $v_{F(P)}(s)\leq v_{F(P')}(s)$.
\end{description}
The following lemma establishes the relationship between monotonicity properties.
\begin{lemma}
\label{lem:monotonicity}
An opinion aggregation function that satisfies monotonicity also satisfies familiar monotonicity.
\end{lemma}
\begin{proof}
Since FM assumes a restriction on the descendants' opinions that monotonicity does not, clearly fulfilling monotonicity implies the fulfilment of familiar monotonicity.
\end{proof}
Now we are ready to prove the relationship between narrow unanimity and sided unanimity via monotonicity.
\begin{prop}\label{NarrowMono:Sided}\label{prop:NU-SU}
An opinion aggregation function that satisfies Narrow unanimity and Monotonicity also satisfies Sided unanimity.
\end{prop}
\begin{proof}
Let $F$ be an opinion aggregation function that fulfils Narrow unanimity and Monotonicity. Let $P=(O_1,\dotsc,O_n)$ be an opinion profile over a $DRF$ and $s\in\sen$. Assume that for any $i\in Ag$, $v_i(s)\geq \lambda$ for a certain $\lambda>0$. Consider the opinion profile $P'$ such that for any $i$ $v'_i(s)=\lambda$. Then, by Monotonicity of $F$, $v_{F(P)}(s)\geq v_{F(P')}(s)$ holds, and by Narrow unanimity of $F$, $v_{F(P')}(s)=\lambda$. Then $v_{F(P)}(s)\geq \lambda >0$, proving that also satisfies Sided Unanimity. The proof for the negative case is analogous.
\end{proof}
Finally, we introduce the property of independence, which will be used to emphasise the difference between those opinion aggregation functions that exploit direct and indirect opinions and the ones that do not. Essentially, the property states that the collective opinion on a statement will only depend on the direct opinions on it. Therefore, independence disregards indirect opinions.
\begin{description}
\item{\textbf{Independence (I)}}
Let be two profiles $P=(O_1,\dotsc,O_n)$ and $P'=(O'_1,\dotsc,O'_n)$, such that $P,P' \in \mathcal{D}$ ; and $s\in \sen$ 
a statement, such that for all agents $i\in \{1,\dotsc,n\}$ $v_i(s)=v'_i(s)$. An aggregation function $F$ satisfies \emph{independence} if $v_{F(P)}(s)=v_{F(P')}(s)$.
\end{description}
Next result shows the relationship between monotonicity and independence.
\begin{prop}\label{prop:MtoI}
An opinion aggregation function that satisfies monotonicity also satisfies independence.
\end{prop}
\begin{proof}
Let $s\in\sen$ be a statement and $P=(O_1,\dotsc,O_n)$, $P'=(O'_1,\dotsc,O'_n)$  two opinion profiles satisfying the assumptions of independence on $s$, i.e., for every $i\in Ag$ $v_i(s)=v_i'(s)$, and $F$ an aggregation function satisfying monotonicity.

For each $i$, the equality $v_i(s)=v_i'(s)$ is equivalent to (a): $v_i(s)\geq v_i'(s)$, and, (b): $v_i(s)\leq v_i'(s)$. Thus, assuming monotonicity from (a) we can deduce $v_{F(P)}(s)\geq v_{F(P')}(s)$, and from (b) we can deduce that $v_{F(P)}(s)\leq v_{F(P')}(s)$. Hence, we conclude that $v_{F(P)}(s)= v_{F(P')}(s)$ proving that $F$ satisfies independence.
\end{proof}



Having listed these properties, it is important to note that they are not all equal. For a multi-party discussion, we believe that the most important property is collective coherence. If an aggregation function is collectively coherent, the resulting combined opinion will be coherent regardless of the coherence of the initial opinions that are being merged.
In other words, an aggregation function that satisfies collective coherence will always discover a coherent overall opinion no matter how incoherent are the opinions on which it is based.
Along with collective coherence, the properties that we would like to see for an aggregation function are the two domain related properties --- exhaustive domain and coherent domain --- because they allow for broad applicability of the function, and naturally we would prefer exhaustive domain because of its wider reach.
Finally, we regard the usual social choice properties of anonymity and non-dictatorship as essential. 

Among the unanimity properties, we find sided, weak and endorsed unanimity,  which allow more maneuverability when using dependencies to build the collective opinion than narrow unanimity, to be more desirable than narrow unanimity. We do not consider narrow unanimity to be desirable because of its close relationship to independence\footnote{Though it is not directly related --- we need to add further minor assumptions to narrow unanimity for it to imply independence.}. Thus, an aggregation function satisfying narrow unanimity would forbid the use of the indirect opinion the way we consider to be necessary.

Though it is natural to require some form of monotonicity, we do not consider the property M to be desirable because of its relationship to independence and the discarding of indirect opinion. Thus, in its place we prefer familiar monotonicity, which takes into consideration indirect opinion.

Finally, although we focus the design of aggregation functions towards the use of both direct and indirect opinion, we include independence, monotonicity and narrow unanimity in the set of properties we consider in order to to emphasize whether aggregation functions take account of indirect opinion or not.


\section{Aggregation functions to enact collective decision making}\label{sec:OpAggFunc}

This section defines a family of aggregation functions for the relational model. All of the aggregation functions that we define use some combination of direct and indirect opinions to generate a collective opinion. Our aim is to explore the spectrum of aggregation functions from functions that only employ direct opinions to functions that only employ indirect opinions, so that we can later on (in section \ref{sec:Analysis}) compare the benefits that they yield in terms of social choice properties. This will allow us to learn how to best exploit indirect opinions to obtain collective opinions.

We start by defining an aggregation function that only aggregates direct opinions, and thus disregards indirect opinions.
This function will obtain a collective opinion by averaging valuations per statement and acceptance degrees per relation from the individual opinions in an opinion profile. Formally:
\begin{defi}[Direct aggregation]
Let $\langle \sen, \rel, T \rangle$ be a $DRF$ and $P=(O_1,\dotsc ,O_n)$ an opinion profile over the $DRF$. The \emph{direct aggregation} of $P$ over the $DRF$ is defined as a function $D(P)=(v_{D(P)},w_{D(P)})$,
where $v_{D(P)}(s)=\frac1n\sum_{i=1}^n v_i(s)$ and  $w_{D(P)}(r)=\frac1n\sum_{i=1}^n w_i(r)$ for any statement $s\in \sen$ and relationship $r\in \rel$.
\end{defi}
\begin{ex}
Figure \ref{pic:direct} 
shows the result of applying the direct function to the opinion profile of our running example, shown in figures \ref{pic:ex_valuations} and \ref{pic:ex_acceptances}.

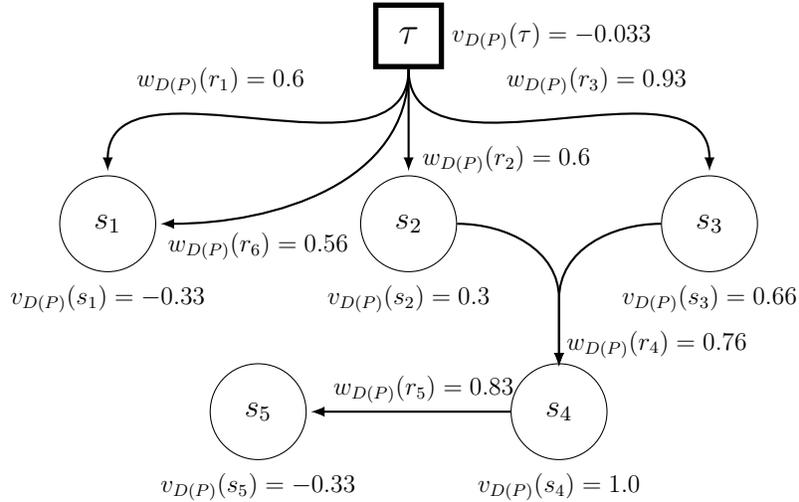
\begin{figure}[H]
\centering
\begin{tikzpicture}[every node/.append style={inner sep=3mm},every place/.style={line width=0.3mm}]

\node[rectangle,draw,line width=0.7mm] (T) {\large{$\tau$}};
\node(s1) [circle,draw] at ($(T)+(-4,-2.5)$){$s_1$} ; 
\node(s2) [circle,draw] at ($(T)+(0,-2.5)$){$s_2$} ; 
\node(s3) [circle,draw] at ($(T)+(4,-2.5)$){$s_3$} ; 
\node(s4) [circle,draw] at ($(T)+(2,-5)$){$s_4$} ; 
\node(s5) [circle,draw] at ($(T)+(-2,-5)$){$s_5$} ; 

\draw[->,thick,>=latex] (T) to[out=-90,in=90]  ($(s1)+(0,0.7)$);

\draw[->,thick,>=latex] (T) to[out=-90,in=90]  ($(s2)+(0,0.7)$);

\draw[->,thick,>=latex] (T) to[out=-90,in=90] ($(s3)+(0,0.7)$);

\draw[->,thick,>=latex] (s2) to[out=0,in=90] (2,-3.5)  to[out=-90,in=90]   ($(s4)+(0,0.6)$); 
\draw[->,thick,>=latex] (s3) to[out=180,in=90] (2,-3.5) to[out=-90,in=90,line width=0.6mm] ($(s4)+(0,0.6)$);

\draw[->,thick,>=latex] (s4) to[out=180,in=0]  ($(s5)+(0.7,0)$);

\draw[->,thick,>=latex] (T) to[out=-90,in=0]  ($(s1)+(0.7,0)$);

\node [] (vT) at ($(T)+(1.9,0)$) {\scalebox{0.8}{$v_{D(P)}(\tau)=-0.033$}};

\node [] (v1) at ($(s1)+(-0,-1)$) {\scalebox{0.8}{$v_{D(P)}(s_1)=-0.33$}};

\node [] (v2) at ($(s2)+(0,-1)$) {\scalebox{0.8}{$v_{D(P)}(s_2)=0.3$}};

\node [] (v3) at ($(s3)+(0,-1)$) {\scalebox{0.8}{$v_{D(P)}(s_3)=0.66$}};

\node [] (v4) at ($(s4)+(0,-1)$) {\scalebox{0.8}{$v_{D(P)}(s_4)=1.0$}};

\node [] (v5) at ($(s5)+(0,-1)$) {\scalebox{0.8}{$v_{D(P)}(s_5)=-0.33$}};

\node [] (w1) at ($(T)+(-2.5,-0.6)$) {\scalebox{0.8}{$w_{D(P)}(r_1)=0.6$}};

\node [] (w2) at ($(s2)+(1.3,0.85)$) {\scalebox{0.8}{$w_{D(P)}(r_2)=0.6$}};

\node [] (w3) at ($(T)+(2.5,-0.6)$) {\scalebox{0.8}{$w_{D(P)}(r_3)=0.93$}};

\node [] (w4) at ($(s4)+(1.3,0.9)$) {\scalebox{0.8}{$w_{D(P)}(r_4)=0.76$}};

\node [] (w5) at ($(s4)+(-1.8,0.3)$) {\scalebox{0.8}{$w_{D(P)}(r_5)=0.83$}};

\node [] (w6) at ($(s1)+(2,-0.3)$) {\scalebox{0.8}{$w_{D(P)}(r_6)=0.56$}};

\end{tikzpicture}
\caption{Direct function: aggregated valuations}
\label{pic:direct}
\end{figure}

\comment{\begin{figure}
\begin{minipage}{.49\textwidth}
\centering
    \includegraphics[width=6.5cm]{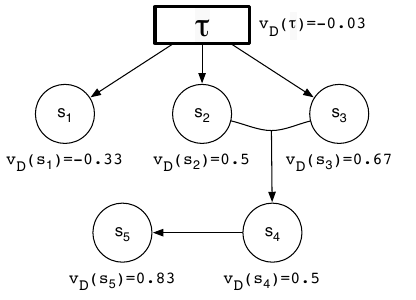}
    \caption{Direct function: aggregated  valuations.}
    \label{pic:direct}
\end{minipage}%
\begin{minipage}{.49\textwidth}
\centering
    \includegraphics[width=5.5cm]{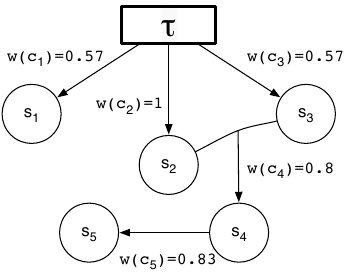}
    \caption{Direct function: aggregated  acceptances.}
    \label{pic:aggregation_weights}
\end{minipage}
\end{figure}}
\end{ex}
Next, we define an aggregation function that only aggregates indirect opinions, and disregards direct opinions. It is thus the converse of the direct aggregation function defined above. The aggregation of indirect opinions is based on the aggregation of the estimation functions introduced in section \ref{sec:Expectation_Coherence}. Formally:
\begin{defi}[Indirect aggregation]
$\langle \sen, \rel, T \rangle$ be a $DRF$ and $P=(O_1,\dotsc \allowbreak ,O_n)$ an opinion profile over the $DRF$.
The \emph{indirect aggregation} of $P$ over the $DRF$ is defined as a function $I(P)=(v_{I(P)},w_{I(P)})$,
where $v_{I(P)}(s)=\frac1n\sum_{i=1}^n e_i(s)$, where $e_i$ is an estimation function, and  $w_{I(P)}(r)=\frac1n\sum_{i=1}^n w_i(r)$ for any statement $s\in \sen$ and relationship $r\in \rel$.
\end{defi}
Notice that, while the direct aggregation function computes the average of individuals' direct opinions, the indirect aggregation function computes the average of individuals' indirect opinions. This is achieved by aggregating individuals' estimated opinions using an estimation function.
We observe also that both functions calculate in the same way the aggregation of acceptance degrees.
This will be the case for all the aggregation functions defined in this section, and hence the difference between them will be the way in which they aggregate valuations.
\begin{ex}
Figure \ref{pic:indirect} shows the aggregated (collective) valuations obtained by the indirect aggregation function for the opinion profile of our running example shown in figure \ref{pic:ex_valuations}. For the sake of illustrating the computation, notice that $v_I(\tau)=(e_1(\tau)+ e_2(\tau)+e_3(\tau))/3=0.079$ where 
$e_1(\tau)=(0.2\cdot 0 + 0.5 \cdot 0 + 0.1 \cdot 0.7 + 1 \cdot 1)/1.8 = 0.59,$
$ e_2(\tau) = -0.257$, and $e_3(\tau)=-0.107$. The aggregated acceptance degrees are the same as in figure \ref{pic:direct}.

\begin{figure}[H]
\centering
\begin{tikzpicture}[every node/.append style={inner sep=3mm},every place/.style={line width=0.3mm}]

\node[rectangle,draw,line width=0.7mm] (T) {\large{$\tau$}};
\node(s1) [circle,draw] at ($(T)+(-4,-2.5)$){$s_1$} ; 
\node(s2) [circle,draw] at ($(T)+(0,-2.5)$){$s_2$} ; 
\node(s3) [circle,draw] at ($(T)+(4,-2.5)$){$s_3$} ; 
\node(s4) [circle,draw] at ($(T)+(2,-5)$){$s_4$} ; 
\node(s5) [circle,draw] at ($(T)+(-2,-5)$){$s_5$} ; 

\draw[->,thick,>=latex] (T) to[out=-90,in=90]  ($(s1)+(0,0.7)$);

\draw[->,thick,>=latex] (T) to[out=-90,in=90]  ($(s2)+(0,0.7)$);

\draw[->,thick,>=latex] (T) to[out=-90,in=90] ($(s3)+(0,0.7)$);

\draw[->,thick,>=latex] (s2) to[out=0,in=90] (2,-3.5)  to[out=-90,in=90]   ($(s4)+(0,0.6)$); 
\draw[->,thick,>=latex] (s3) to[out=180,in=90] (2,-3.5) to[out=-90,in=90,line width=0.6mm] ($(s4)+(0,0.6)$);

\draw[->,thick,>=latex] (s4) to[out=180,in=0]  ($(s5)+(0.7,0)$);

\draw[->,thick,>=latex] (T) to[out=-90,in=0]  ($(s1)+(0.7,0)$);

\node [] (vT) at ($(T)+(1.7,0)$) {\scalebox{0.8}{$v_{I(P)}(\tau)=0.076$}};

\node [] (v1) at ($(s1)+(-0,-1)$) {\scalebox{0.8}{$v_{I(P)}(s_1)=-0.333$}};

\node [] (v2) at ($(s2)+(0,-1)$) {\scalebox{0.8}{$v_{I(P)}(s_2)=1.0$}};

\node [] (v3) at ($(s3)+(0,-1)$) {\scalebox{0.8}{$v_{I(P)}(s_3)=1.0$}};

\node [] (v4) at ($(s4)+(0,-1)$) {\scalebox{0.8}{$v_{I(P)}(s_4)=-0.333$}};

\node [] (v5) at ($(s5)+(0,-1)$) {\scalebox{0.8}{$v_{I(P)}(s_5)=-0.333$}};

\end{tikzpicture}
\caption{Indirect function: aggregated valuations}
\label{pic:indirect}
\end{figure}
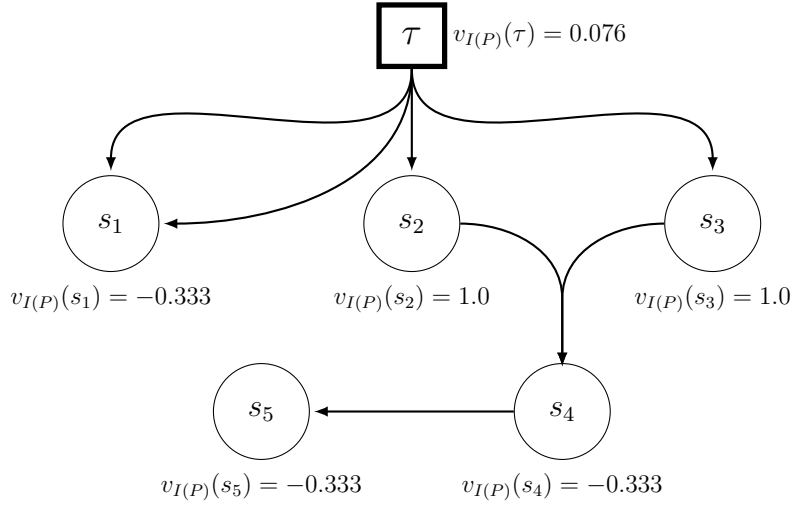

\end{ex}
%
%
%

Having defined two extremes of our spectrum of functions, we introduce a family of aggregation functions based on a linear combination of the direct and indirect aggregation functions.
\begin{defi}[$\alpha$-Balanced aggregation] 
\label{def:balanced-aggregation}
Let $\langle \sen, \rel, T \rangle$ be a $DRF$ and $P=(O_1,\dotsc ,O_n)$ an opinion profile over the $DRF$.
Given the direct aggregation $D(P)=(v_{D(P)},w_{D(P)})$, the indirect aggregation $I(P)=(v_{I(P)},w_{I(P)})$, and $\alpha \in [0,1]$, we define the aggregation function 
$B_\alpha(P)=(v_{B_\alpha(P)},w_{B_\alpha(P)})$, where:
\begin{align*}
v_{B_\alpha(P)} & = \alpha \cdot v_{D(P)}+(1-\alpha) \cdot v_{I(P)}\\ 
w_{B_\alpha(P)}(r) & =\frac1n\sum_{i=1}^n w_i(r)
\end{align*}
for any statement $s\in \sen$ and relationship $r\in \rel$. We say that $B_\alpha$ is an $\alpha$-balanced aggregation function.
\end{defi}

By changing the value of $\alpha$ we set the importance of the direct opinion with respect to the indirect opinion. The functions resulting from definition \ref{def:balanced-aggregation} form a family of balanced aggregation functions: $\{B_\alpha\}_{\alpha\in[0,1]}$. In particular, by setting $\alpha$ to $0$ we obtain the indirect aggregation function, and by setting it to $1$ we obtain the direct aggregation function. 

Next, we define an aggregation function that exploits indirect opinions differently. For a given statement, the so-called \emph{recursive aggregation} function calculates its aggregated valuation by using the collective opinion on its descendants, which in turn is recursively computed from their descendants, and so on. This recursive computing ends up when reaching statements without descendants whose indirect opinion is empty. Therefore, the recursive aggregation, unlike balanced aggregations, disregards individual valuations in the indirect opinion, and employs their collective opinions instead. 
%
\begin{defi}[Recursive aggregation] 
Let $\langle \sen, \rel, T \rangle$ be a $DRF$ and $P=(O_1,\dotsc ,O_n)$ an opinion profile over the $DRF$.
The \emph{recursive aggregation} of $P$ over the $DRF$ is defined as a function $R(P)=(v_{R(P)},w_{R(P)})$,
where
\[
v_{R(P)}(s)=\left\{\begin{array}{cc}
\frac{1}{\sum_{r\in R^+(s)} w_{R(P)}(r)} \sum_{r\in R^+(s)} v_{R(P)}(s_r) \cdot w_{R(P)}(r) & \text{if } R^+(s)\neq\emptyset \\
    v_{D(P)}(s) & \text{otherwise} 
\end{array}\right.
\]
\comment{\begin{align*}\label{eq:sign}
v_{R(P)}(s) &= 
   \begin{cases}
    $\frac{1}{\sum_{r\in R^+(s)} w_{R(P)}(r)} \sum_{r\in R^+(s)} v_{R(P)}(s_r)w_{R(P)}(r)$ & \text{if } $R^+(s)\neq\emptyset$ \\
    $v_{D(P)}(s)$ & \text{otherwise} 
\end{cases}
\end{align*}}
and $w_{R(P)}(r)=\frac1n\sum_{i=1}^n w_i(r)$
for any statement $s\in \sen$ and relationship $r\in \rel$.
\end{defi}

Recall that $R^+(s)$ stands for the relationships connecting $s$ to a descendant $s_r$ of $s$ through the relationship $r$.



Thus, the recursive function computes the average of the indirect collective opinion computed so far. In fact, we could say that, due to its recursive character, the function computes the estimated opinion for each statement in a bottom-up manner. Thus, the aggregation of opinions starts considering the direct opinions at the ``leaves'' of the debate, namely at the statements with no descendants, and moves up until reaching the targets.


\begin{ex} 
Figure \ref{pic:recursive} shows the aggregated (collective) valuations obtained by the recursive aggregation function for the opinion profile of our running example shown in figure \ref{pic:ex_valuations}. The aggregated acceptance degrees are the same as in figure \ref{pic:direct}. Again, for the sake of illustrating the computation, please notice that we start by computing $v_R(s_1)=v_D(s_1)=-0.33$ and $v_R(s_5)=v_D(s_5)=-0.33$ and, from these, we can compute $v_R(s_4)=v_R(s_5) \cdot w(r_5)/w(r_5)=-0.33$, $v_R(s_2)=v_R(s_4)\cdot w(r_4)/w(r_4)=-0.33=v_R(s_3)$ so to finally compute $v_R(\tau)=(-0.33\cdot0.6 -0.33\cdot 0.56 - 0.33\cdot 0.6 - 0.33\cdot0.93)/2.69=-0.33$.

\begin{figure}[H]
\centering
\begin{tikzpicture}[every node/.append style={inner sep=3mm},every place/.style={line width=0.3mm}]

\node[rectangle,draw,line width=0.7mm] (T) {\large{$\tau$}};
\node(s1) [circle,draw] at ($(T)+(-4,-2.5)$){$s_1$} ; 
\node(s2) [circle,draw] at ($(T)+(0,-2.5)$){$s_2$} ; 
\node(s3) [circle,draw] at ($(T)+(4,-2.5)$){$s_3$} ; 
\node(s4) [circle,draw] at ($(T)+(2,-5)$){$s_4$} ; 
\node(s5) [circle,draw] at ($(T)+(-2,-5)$){$s_5$} ; 

\draw[->,thick,>=latex] (T) to[out=-90,in=90]  ($(s1)+(0,0.7)$);

\draw[->,thick,>=latex] (T) to[out=-90,in=90]  ($(s2)+(0,0.7)$);

\draw[->,thick,>=latex] (T) to[out=-90,in=90] ($(s3)+(0,0.7)$);

\draw[->,thick,>=latex] (s2) to[out=0,in=90] (2,-3.5)  to[out=-90,in=90]   ($(s4)+(0,0.6)$); 
\draw[->,thick,>=latex] (s3) to[out=180,in=90] (2,-3.5) to[out=-90,in=90,line width=0.6mm] ($(s4)+(0,0.6)$);

\draw[->,thick,>=latex] (s4) to[out=180,in=0]  ($(s5)+(0.7,0)$);

\draw[->,thick,>=latex] (T) to[out=-90,in=0]  ($(s1)+(0.7,0)$);

\node [] (vT) at ($(T)+(1.8,0)$) {\scalebox{0.8}{$v_{R(P)}(\tau)=-0.33$}};

\node [] (v1) at ($(s1)+(-0,-1)$) {\scalebox{0.8}{$v_{R(P)}(s_1)=-0.33$}};

\node [] (v2) at ($(s2)+(0,-1)$) {\scalebox{0.8}{$v_{R(P)}(s_2)=-0.33$}};

\node [] (v3) at ($(s3)+(0,-1)$) {\scalebox{0.8}{$v_{R(P)}(s_3)=-0.33$}};

\node [] (v4) at ($(s4)+(0,-1)$) {\scalebox{0.8}{$v_{R(P)}(s_4)=-0.33$}};

\node [] (v5) at ($(s5)+(0,-1)$) {\scalebox{0.8}{$v_{R(P)}(s_5)=-0.33$}};

\end{tikzpicture}
\caption{Recursive function: aggregated valuations}
\label{pic:recursive}
\end{figure}
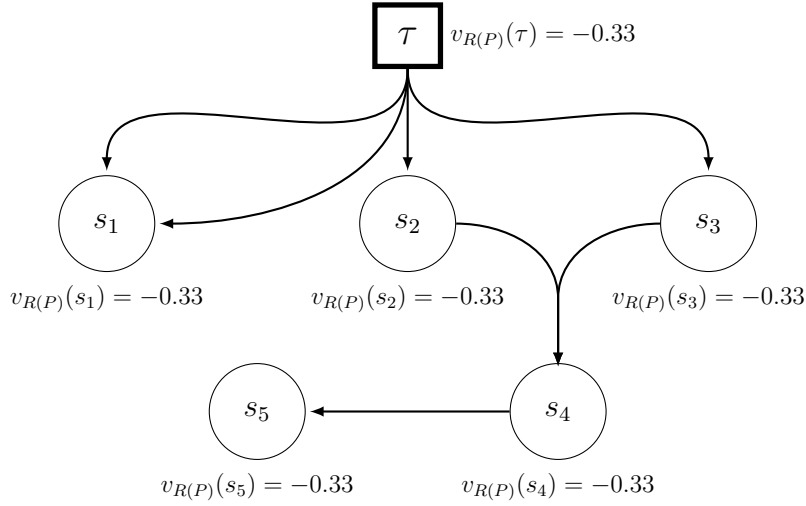
\end{ex}
Similarly to what we did for the balanced function above, next we define a family of aggregation functions based on combining  the direct and recursive aggregation functions.
%
%
%
\begin{defi}[$\alpha$-recursive aggregation] Let $\langle \sen, \rel, T \rangle$ be a $DRF$ and $P=(O_1,\dotsc ,O_n)$ an opinion profile over the $DRF$. Given the direct aggregation $D(P)=(v_{D(P)},w_{D(P)})$, the recursive aggregation $R(P)=(v_{R(P)},w_{R(P)})$, and $\alpha \in [0,1]$, we define the aggregation function $R_\alpha(P)=(v_{R_\alpha(P)},w_{R_\alpha(P)})$, where:
\begin{align*}
    v_{R_\alpha(P)} & = \alpha \cdot v_{D(P)}+(1-\alpha) \cdot v_{R(P)} \\
    w_{R_\alpha(P)}(r) & = \frac1n\sum_{i=1}^n w_i(r)
\end{align*} for any statement $s\in \sen$ and relationship $r\in \rel$. We say that $R_\alpha$ is an $\alpha$-recursive aggregation function.
\end{defi}

\section{Analysing opinion aggregation functions}
\label{sec:Analysis}

In this section we compare the aggregation functions introduced in section \ref{sec:OpAggFunc} in terms of their satisfaction, or otherwise, of the social choice properties introduced in section \ref{sec:Aggregation}. Our analysis will run along two dimensions: (1) the \emph{coherence} of an opinion profile; and (2) the \emph{consensus} on the acceptance degrees of an opinion profile. Thus, we will consider whether agents' opinions are constrained to be coherent (the opinion profile is coherent) or not 
, and whether agents agree on acceptance degrees (there is consensus on acceptance degrees) or not. This results in four debate scenarios to analyse:
\begin{enumerate}
    \item Unconstrained opinion profiles;
    \item Constrained opinion profiles: assuming consensus on acceptance degrees;
    \item Constrained opinion profiles: assuming coherent profiles; and
    \item Constrained opinion profiles: assuming consensus on acceptance degrees and coherent profiles.
\end{enumerate}
The analysis of these scenarios will help us assess the price that must be paid if the opinions stated by participating agents are not necessarily coherent. The scenarios will also help us assess that must be paid when the relationships between statements are opened for discussion by means of acceptance degrees.
    
Note that, in the interests of readability, we do not include the formal analysis in the body of the paper. Rather we present the results of that analysis. The formal analysis, both statements of results and the proofs of the results, can be found in  \ref{sec:Proofs}.



\subsection{Unconstrained opinion profiles}\label{subsec:generalResults}

This is the most general scenario we can consider for analysis. We assume unconstrained opinion profiles, which means that \emph{any} opinion profile is deemed to be possible input for the aggregation functions introduced in Section \ref{sec:OpAggFunc}. In other words, the domain of our aggregation functions is the class $\mathbb{O}(DRF)^n$ itself, and hence opinions need not to be coherent nor have consensus on the acceptance degrees. 



Table \ref{tableGeneralResults} displays the social choice properties fulfilled by the functions defined in section \ref{sec:OpAggFunc} in this general case. There is one column per aggregation function and one row per social choice property.
Notice that we distinguish between desirable social choice properties and other properties as discussed in section \ref{subsec:socialChoiceProperties}.
In the table, a green square (with a tick) indicates that a property is fulfilled, while a red square (with a cross) indicates that a property is not fulfilled. As to the more general aggregation functions, $\alpha$-B(alanced), and $\alpha$-R(ecursive), in some cases we specify the values $\alpha$ for which a given property holds. Notice that, for both families, we show the results considering $\alpha\in(0,1)$, not considering 0 or 1, thus the cases for the extreme values represent aggregations functions already displayed in other columns.


\begin{table}[h]
\centering
\footnotesize
\begin{tabular}{ | l | >{\centering}m{0.9cm} | >{\centering}p{0.9cm} 
| >{\centering}p{0.9cm} | >{\centering}p{1.3cm} | c |}
\hline
\rowcolor{lightgray} Desirable properties  &  D & I & R & $\alpha$-B & \textbf{$\alpha$-R}\\
\hline
\textbf{Collective coherence}\score{1}{1}
& \cellcolor[HTML]{FF0000} \textcolor{white}{\xmark} & \cellcolor[HTML]{FF0000} \textcolor{white}{\xmark}  \comment{ & \cellcolor[HTML]{FF0000} \textcolor{white}{\xmark}} & \cellcolor[HTML]{04B45F} \checkmark & \cellcolor[HTML]{FF0000} \textcolor{white}{\xmark} & \cellcolor[HTML]{04B45F} $\alpha<\epsilon/2$ \\
\hline
Exhaustive domain & \cellcolor[HTML]{04B45F} \checkmark & \cellcolor[HTML]{04B45F} \checkmark  \comment{& \cellcolor[HTML]{04B45F}  \checkmark} & \cellcolor[HTML]{04B45F} \checkmark & \cellcolor[HTML]{04B45F}  \checkmark & \cellcolor[HTML]{04B45F} \checkmark\\
\hline
Coherent domain & \cellcolor[HTML]{04B45F} \checkmark & \cellcolor[HTML]{04B45F} \checkmark  \comment{& \cellcolor[HTML]{04B45F}  \checkmark} & \cellcolor[HTML]{04B45F} \checkmark & \cellcolor[HTML]{04B45F}  \checkmark & \cellcolor[HTML]{04B45F} \checkmark\\
\hline
Anonymity  & \cellcolor[HTML]{04B45F} \checkmark & \cellcolor[HTML]{04B45F} \checkmark  \comment{& \cellcolor[HTML]{04B45F}  \checkmark} & \cellcolor[HTML]{04B45F} \checkmark & \cellcolor[HTML]{04B45F}  \checkmark & \cellcolor[HTML]{04B45F} \checkmark\\
\hline
Non-dictatorship & \cellcolor[HTML]{04B45F} \checkmark & \cellcolor[HTML]{04B45F} \checkmark  \comment{& \cellcolor[HTML]{04B45F}  \checkmark} & \cellcolor[HTML]{04B45F} \checkmark & \cellcolor[HTML]{04B45F}  \checkmark & \cellcolor[HTML]{04B45F} \checkmark\\
\hline

Sided Unanimity &  \cellcolor[HTML]{04B45F} \checkmark  & \cellcolor[HTML]{FF0000} \textcolor{white}{\xmark}  & \cellcolor[HTML]{FF0000} \textcolor{white}{\xmark} & \cellcolor[HTML]{FF0000} \textcolor{white}{\xmark} & \cellcolor[HTML]{FF0000} \textcolor{white}{\xmark}\\
\hline
Weak Unanimity &  \cellcolor[HTML]{04B45F} \checkmark  &\cellcolor[HTML]{FF0000} \textcolor{white}{\xmark} & \cellcolor[HTML]{FF0000} \textcolor{white}{\xmark} & \cellcolor[HTML]{04B45F} $\alpha > 1/2$ & \cellcolor[HTML]{04B45F} $\alpha>1/2$\\
\hline
Endorsed Unanimity &  \cellcolor[HTML]{FF0000} \textcolor{white}{\xmark}  & \cellcolor[HTML]{04B45F} \checkmark  & \cellcolor[HTML]{FF0000} \textcolor{white}{\xmark} & \cellcolor[HTML]{04B45F} $\alpha < 1/2$ & \cellcolor[HTML]{FF0000} \textcolor{white}{\xmark}\\
\hline

Familiar monotonicity & \cellcolor[HTML]{04B45F} \checkmark & \cellcolor[HTML]{04B45F} \checkmark  & \cellcolor[HTML]{FF0000} \textcolor{white}{\xmark} & \cellcolor[HTML]{04B45F} \checkmark & \cellcolor[HTML]{FF0000} \textcolor{white}{\xmark}\\

\hline\hline \rowcolor{lightgray} Other properties & & & & &\\ \hline
Monotonicity & \cellcolor[HTML]{04B45F} \checkmark & \cellcolor[HTML]{FF0000} \textcolor{white}{\xmark} \comment{& \cellcolor[HTML]{FF0000} \textcolor{white}{\xmark}} & \cellcolor[HTML]{FF0000} \textcolor{white}{\xmark} & \cellcolor[HTML]{FF0000} \textcolor{white}{\xmark} & \cellcolor[HTML]{FF0000} \textcolor{white}{\xmark}\\
\hline
Narrow Unanimity &  \cellcolor[HTML]{04B45F} \checkmark  &\cellcolor[HTML]{FF0000} \textcolor{white}{ \xmark} \comment{& \cellcolor[HTML]{FF0000} \textcolor{white}{ \xmark}} & \cellcolor[HTML]{FF0000} \textcolor{white}{\xmark} & \cellcolor[HTML]{FF0000} \textcolor{white}{ \xmark} & \cellcolor[HTML]{FF0000} \textcolor{white}{ \xmark}\\
\hline
Independence & \cellcolor[HTML]{04B45F} \checkmark & \cellcolor[HTML]{FF0000} \textcolor{white}{\xmark} & \cellcolor[HTML]{FF0000} \textcolor{white}{\xmark} & \cellcolor[HTML]{FF0000} \textcolor{white}{\xmark} & \cellcolor[HTML]{FF0000} \textcolor{white}{\xmark}\\
\hline
\end{tabular}
\caption{Social choice properties satisfied by aggregation functions D(irect), I(ndirect), R(ecursive), $\alpha$-B(alanced), and $\alpha$-R(ecursive) for: (i) a general scenario considering unconstrained opinion profiles; (ii) a scenario considering constrained opinion profiles: consensus on acceptance degrees.}
\label{tableGeneralResults}
\end{table}
\normalsize

\noindent
\textbf{Domain and anonymity.} Table \ref{tableGeneralResults} shows that Exhaustive Domain (ED), Coherent Domain, Anonymity and Non-dictatorship are fulfilled by all the proposed opinion aggregation functions. This is because of the agnostic treatment of opinion profiles implemented by the aggregation functions. Since no constraints are imposed on opinion profiles received as input, ED is satisfied, and since no agent in an opinion profile receives a special treatment, Anonymity holds. Satisfying this family of properties is important. On the one hand, fulfilling ED ensures that any opinion profile is valid as input to a collective opinion, that is, the aggregation functions do not filter out participants' opinions prior to computing a collective opinion. On the other hand, satisfying Anonymity  guarantees that all participants are equally important when calculating a collective opinion. 


\noindent
\textbf{Collective coherence.} The (D)irect and (I)ndirect functions do not satisfy collective coherence. As a result, neither does any $\alpha$-Balanced aggregation function because it results from the linear combination of D and I. The result of such aggregation methods largely depends on the coherence of the opinion profile at hand, which in this scenario can be as incoherent as possible. More positively, the (R)ecursive aggregation function does satisfy Collective Coherence (CC). Out of the family of recursive aggregation functions ($\alpha$-R), which relies on D and R, those for which  
$\alpha<\epsilon/2$, where $\epsilon$ is set to assess the coherence of the output, also satisfy CC. This tells us that the closer is $\alpha$ to 0 (the less the use of the direct opinion), the more coherent the collective opinion obtained by an $\alpha$-R function will be. The closer $\alpha$ is to $\epsilon/2$, the less coherent the collective opinion obtained by an $\alpha$-R function will be. When $\alpha$ goes beyond $\epsilon/2$, the $\alpha$-R function depends too much on the direct opinion (which does not satisfy CC) and CC does not hold.


\noindent
\textbf{Unanimity.} 
Narrow, Sided and Weak Unanimity are not satisfied by the Indirect and Recursive aggregation functions. This is because the indirect opinion, employed by all these aggregation functions, ignores unanimity on the direct opinion of a statement and in some cases these functions can produce a result in the opposite direction. On the other hand, the Direct function, which only depends on the direct opinions of a statement, does satisfy all the unanimity properties. This benefits the Balanced and Recursive families, which satisfy Weak unanimity for some values of $\alpha$. Notice that only Balanced and Recursive aggregation functions for which $\alpha$ is greater than $1/2$ satisfy Weak unanimity. This is to lessen the influence of the \mbox{indirect} opinion and sway the result towards the Direct aggregation function, which does satisfy the property. Regarding the Narrow and Sided unanimity properties, not even the influence of the Direct function is enough to guarantee that unanimity is preserved, and therefore no aggregation function in the Balanced or Recursive families fulfil them for any value of $\alpha$.


On the other hand, regarding  Endorsed unanimity, the situation changes for the Direct and Indirect functions. They flip sides so that the Direct function does not fulfill Endorsed unanimity, but the Indirect function does. This is because the unanimity in this case resides in indirect opinions, and hence it is in line with the Indirect function, which only depends on indirect opinions. However, this goes against the Direct function, which disregards indirect opinions, and hence unanimity on its values. Conversely to the Weak unanimity case for the Balanced family, now we need that the values of $\alpha$ are less than $1/2$ to sway the balanced aggregation towards the Indirect function, and hence, satisfy Endorsed unanimity. Next, although it might seem reasonable that aggregation functions in the Recursive family also fulfil Endorsed unanimity, they do not. This is caused by the the recursive behaviour of these aggregation functions, which can  overlook unanimity on indirect opinions to use instead opinions deep in the debate on which there might be no unanimity. And last, due to the failure of the Direct and Recursive aggregation functions to fulfil Endorsed unanimity, so do all the aggregation functions in the Recursive family, no matter the value of $\alpha$.

\noindent
\textbf{Monotonicity.}
The Familiar Monotonicity property is fulfilled by the Direct function (as a consequence of fulfilling Monotonicity), the Indirect function, and therefore by the whole family of Balanced functions that are combinations of the Direct and Indirect functions. The Recursive function, and hence the Recursive family, fails to satisfy Familiar Monotonicity because given a statement, the aggregated opinion about its descendants does not solely depend on the valuations on these descendants alone. Instead, the  aggregated opinion about its descendants recursively depends on descendants down the relational framework. Thus, changes of opinion on "grandchildren" statements can cause a change of opinion independently of any change of the direct opinion.


\noindent
\textbf{Other properties.} For the sake of completeness, next we also analyse the fulfilment of further, non-desirable, properties. As expected, Independence is not fulfilled by any of the function making use, at any degree, of indirect opinion, namely the Indirect function, the Recursive function, and the Balanced and Recursive  families (for any  $\alpha<1$). Also related to the use of indirect opinions, we observe that Narrow Unanimity and Monotonicity are not fulfilled by any of these four functions. Overall, these properties can only be satisfied when employing the direct opinion alone to obtain the aggregated collective opinion. This reinforces our discussion in section \ref{subsec:socialChoiceProperties} about disregarding these properties to consider alternative properties more orientated to aggregation functions that take account of dependencies.




\subsection{Constrained opinions: assuming consensus on acceptance degrees}\label{subsec:consensusResults}


Assuming consensus on acceptances does not lead to any gain in the fulfilment of social choice properties with respect to those already claimed in section \ref{subsec:generalResults}, so previous table \ref{tableGeneralResults} show the results for this scenario as well. Nonetheless, we deemed worth analysing this debate scenario because of the multiple already-in-use participation systems that do not allow participants to value differently the relationships between sentences in a debate. In our case, if we assume that participants agree on acceptance degrees, the collective opinion will only depend on valuations over sentences. This is equivalent to considering a debate where participants are allowed to value sentences, but do not express their opinions on the relationships between them. 


We refer the reader to \ref{sec:Proofs} for the proofs regarding the fulfilment of social choice properties in section \ref{subsec:generalResults}. By analyising such proofs, we observe that assuming consensus on acceptances does not yield any further benefit that allows the hail the fulfilment of any social choice property that did not hold in the analysis of section \ref{subsec:generalResults}. This motivates that we have decided not to include any additional proofs for this scenario in \ref{Proofs:Uniform}. Hence, the proofs in section \ref{subsec:generalResults} are also valid for this scenario.


\subsection{Constrained opinions: assuming coherent profiles}\label{sec:coherentResults}


In the following, we assume that the opinion profile in a debate is constrained to be coherent at some degree (according to some value $\epsilon\in(0,1)$), so that each of the opinions in the profile is always coherent. Recall that we consider that coherence occurs when the direct and indirect opinions are in line. Therefore, assuming coherence is expected to have a positive impact on aggregation functions that exploit  indirect opinions to compute a collective opinion.  Here we assess the gain.

Table \ref{tableCoherentResults} shows the desirable properties satisfied by our aggregation functions when assuming coherence. The light 
green squares with check marks identify properties that are now satisfied, but were not (in table \ref{tableGeneralResults}) when not imposing coherence. Therefore, assuming coherence yields new positive results. More precisely, table \ref{tableCoherentResults} shows that assuming coherence leads to the satisfaction of desirable unanimity properties for several of the functions.
First, given the coherence assumption, the unanimity on the direct opinion drags the indirect opinion to become more similar to it, and therefore the Indirect function gains Weak unanimity.  Now, since the Direct function also satisfies it, it follows that all $\alpha$-Balanced functions now fulfil it too. Furthermore, thanks to the alignment that the coherence assumption brings between the direct and indirect opinions, the Direct function fulfils the Endorsed unanimity property. Therefore, having Endorsed unanimity fulfilled now by the Indirect and Direct functions, the aggregation functions in the Balanced family also fulfil it for any $\alpha$.

Observe that unanimity and the coherence assumption work well together.  Unanimity on one sentence 
brings together its direct and indirect opinions, making it impossible for both to be far apart, and therefore allowing the Direct and Indirect functions to fulfill more unanimity properties.

Finally, the family of Recursive function now fulfils Endorsed unanimity, though, not for any $\alpha$. Depending on the degree of coherence allowed in the opinion profile, i.e. the value of $\epsilon$, the interval of $\alpha$ values allowing $R_\alpha$ to fulfil Endorsed unanimity will change. In this case, $\alpha$ has to be greater than  $1/(2-\epsilon)$, representing the need to overcome the bad result obtained by the Recursive function with respect to the Endorsed unanimity property. 

\begin{table}[h]
\centering
\footnotesize
\begin{tabular}{ | l | >{\centering}m{1cm} | >{\centering}p{1cm} 
| >{\centering}p{1cm} | >{\centering}p{1cm} | c |}
\hline
\rowcolor{lightgray} Desirable properties  &  D & I & R & $\alpha$-B & \textbf{$\alpha$-R}\\
\hline
\textbf{Collective coherence}\score{1}{1}
& \cellcolor[HTML]{FF0000} \textcolor{white}{\xmark} & \cellcolor[HTML]{FF0000} \textcolor{white}{\xmark}  \comment{ & \cellcolor[HTML]{FF0000} \textcolor{white}{\xmark}} & \cellcolor[HTML]{04B45F} \checkmark & \cellcolor[HTML]{FF0000} \textcolor{white}{\xmark} & \cellcolor[HTML]{04B45F} $\alpha<\epsilon/2$ \\
\hline
Exhaustive domain & \cellcolor[HTML]{04B45F} \checkmark & \cellcolor[HTML]{04B45F} \checkmark  \comment{& \cellcolor[HTML]{04B45F}  \checkmark} & \cellcolor[HTML]{04B45F} \checkmark & \cellcolor[HTML]{04B45F}  \checkmark & \cellcolor[HTML]{04B45F} \checkmark\\
\hline
Coherent domain & \cellcolor[HTML]{04B45F} \checkmark & \cellcolor[HTML]{04B45F} \checkmark  \comment{& \cellcolor[HTML]{04B45F}  \checkmark} & \cellcolor[HTML]{04B45F} \checkmark & \cellcolor[HTML]{04B45F}  \checkmark & \cellcolor[HTML]{04B45F} \checkmark\\
\hline
Anonymity  & \cellcolor[HTML]{04B45F} \checkmark & \cellcolor[HTML]{04B45F} \checkmark  \comment{& \cellcolor[HTML]{04B45F}  \checkmark} & \cellcolor[HTML]{04B45F} \checkmark & \cellcolor[HTML]{04B45F}  \checkmark & \cellcolor[HTML]{04B45F} \checkmark\\
\hline
Non-dictatorship & \cellcolor[HTML]{04B45F} \checkmark & \cellcolor[HTML]{04B45F} \checkmark  \comment{& \cellcolor[HTML]{04B45F}  \checkmark} & \cellcolor[HTML]{04B45F} \checkmark & \cellcolor[HTML]{04B45F}  \checkmark & \cellcolor[HTML]{04B45F} \checkmark\\
\hline

Sided Unanimity &  \cellcolor[HTML]{04B45F} \checkmark  &  \cellcolor[HTML]{FF0000} \textcolor{white}{\xmark}  & \cellcolor[HTML]{FF0000} \textcolor{white}{\xmark} & \cellcolor[HTML]{FF0000} \textcolor{white}{\xmark} & \cellcolor[HTML]{FF0000} \textcolor{white}{\xmark}\\
\hline
\cellcolor[HTML]{DBFB8B} \textbf{Weak Unanimity} &  \cellcolor[HTML]{04B45F} \checkmark  &\cellcolor[HTML]{99F7AB} \checkmark & \cellcolor[HTML]{FF0000} \textcolor{white}{\xmark} & \cellcolor[HTML]{99F7AB} \checkmark & \cellcolor[HTML]{04B45F} $\alpha>1/2$\\
\hline
\cellcolor[HTML]{DBFB8B} \textbf{Endorsed Unanimity} &  
\cellcolor[HTML]{99F7AB} \checkmark & \cellcolor[HTML]{04B45F} \checkmark  & \cellcolor[HTML]{FF0000} \textcolor{white}{\xmark} &
\cellcolor[HTML]{99F7AB} \checkmark &
\cellcolor[HTML]{99F7AB} $\alpha>\frac 1 {2-\epsilon}$\\
\hline

Familiar monotonicity & \cellcolor[HTML]{04B45F} \checkmark & \cellcolor[HTML]{04B45F} \checkmark  & \cellcolor[HTML]{FF0000} \textcolor{white}{\xmark} &
\cellcolor[HTML]{04B45F} \checkmark &
\cellcolor[HTML]{FF0000} \textcolor{white}{\xmark}\\

\hline
\end{tabular}
\caption{Highlighted, in light color, the fulfilment of additional desirable properties, on top of those in Table~\ref{tableGeneralResults}, when assuming coherent opinions. }
\label{tableCoherentResults}
\end{table}
\normalsize

\subsection{Constrained opinions: assuming consensus on acceptance degrees and coherent profiles}\label{subsec:sameAcceptanceResults}

In what follows we assume both previous constraints on the opinion profiles: coherence on the opinions and consensus on acceptances degrees. First, consensus on acceptance degrees on relationships represents a more simplified debate where participants only provide their opinions on sentences. Second, the coherence assumed on opinions aligns direct and indirect opinions. Overall, both assumptions yield major benefits in terms of the satisfaction of desired social choice properties, as we discuss next.

Table \ref{tableUniformAcceptance} shows the gain in fulfilment of desirable properties with respect to table \ref{tableCoherentResults}. The light 
green squares with check marks identify properties that are now satisfied, but were not (in table \ref{tableCoherentResults}) when not imposing consensus on acceptance degrees. Now, besides the aggregation functions in the Recursive family, which now satisfies Collective coherence for any $\alpha$, the rest of aggregation functions under study also satisfy $\epsilon$-Collective coherence when the opinion profiles are $\delta$-coherent, $0<\delta\leq\epsilon$.  

This major improvement is because the consensus on acceptance degrees forbids the participants to value a relationship as 0, which is key to ensure collective coherence for the Direct and Indirect functions when the opinion profiles are coherent. We assume that for each relationship at least one agent has valued it other than 0, because otherwise it would be as if the relationship did not exist, and this forces all the participants to have a positive value too.


In this manner, if we have both $\epsilon$-coherent opinions and consensus, then all our aggregation functions can guarantee $\epsilon$-coherent aggregated opinions, which may increase the acceptability of the results from the participants.


Notice that assuming consensus on acceptance degrees is quite reasonable. Indeed, we believe that such consensus is actually very likely to be found in many debates where the relationships are classified first and afterwards the participants are allowed to give their opinions. Our model can fit perfectly with these kinds of scenario by setting all the acceptance degrees as a constant value for every participant. Furthermore, the procedure to create the debate, and therefore the DRF, could be adapted so there is a first stage in which a collective value is established for each of the relationships, and then there is a second stage in which values are assigned to the statements. 


\begin{table}[h]
\centering
\footnotesize
\begin{tabular}{ | l | >{\centering}m{1cm} | >{\centering}p{1cm} 
| >{\centering}p{1cm} | >{\centering}p{1cm} | c |}
\hline
\rowcolor{lightgray} Desirable properties  &  D & I & R & $\alpha$-B & \textbf{$\alpha$-R}\\
\hline
\cellcolor[HTML]{DBFB8B} \textbf{Collective coherence}\score{1}{1} &
\cellcolor[HTML]{99F7AB} \checkmark & 
\cellcolor[HTML]{99F7AB} \checkmark & \cellcolor[HTML]{04B45F} \checkmark & \cellcolor[HTML]{99F7AB} \checkmark & \cellcolor[HTML]{99F7AB} \checkmark \\
\hline
Exhaustive domain & \cellcolor[HTML]{04B45F} \checkmark & \cellcolor[HTML]{04B45F} \checkmark  \comment{& \cellcolor[HTML]{04B45F}  \checkmark} & \cellcolor[HTML]{04B45F} \checkmark & \cellcolor[HTML]{04B45F}  \checkmark & \cellcolor[HTML]{04B45F} \checkmark\\
\hline
Coherent domain & \cellcolor[HTML]{04B45F} \checkmark & \cellcolor[HTML]{04B45F} \checkmark  \comment{& \cellcolor[HTML]{04B45F}  \checkmark} & \cellcolor[HTML]{04B45F} \checkmark & \cellcolor[HTML]{04B45F}  \checkmark & \cellcolor[HTML]{04B45F} \checkmark\\
\hline
Anonymity  & \cellcolor[HTML]{04B45F} \checkmark & \cellcolor[HTML]{04B45F} \checkmark  \comment{& \cellcolor[HTML]{04B45F}  \checkmark} & \cellcolor[HTML]{04B45F} \checkmark & \cellcolor[HTML]{04B45F}  \checkmark & \cellcolor[HTML]{04B45F} \checkmark\\
\hline
Non-dictatorship & \cellcolor[HTML]{04B45F} \checkmark & \cellcolor[HTML]{04B45F} \checkmark  \comment{& \cellcolor[HTML]{04B45F}  \checkmark} & \cellcolor[HTML]{04B45F} \checkmark & \cellcolor[HTML]{04B45F}  \checkmark & \cellcolor[HTML]{04B45F} \checkmark\\
\hline

Sided Unanimity &  \cellcolor[HTML]{04B45F} \checkmark  &  \cellcolor[HTML]{FF0000} \textcolor{white}{\xmark}  & \cellcolor[HTML]{FF0000} \textcolor{white}{\xmark} & \cellcolor[HTML]{FF0000} \textcolor{white}{\xmark} & \cellcolor[HTML]{FF0000} \textcolor{white}{\xmark}\\
\hline

Weak Unanimity &  \cellcolor[HTML]{04B45F} \checkmark  &\cellcolor[HTML]{04B45F} \checkmark & \cellcolor[HTML]{FF0000} \textcolor{white}{\xmark} & \cellcolor[HTML]{04B45F} \checkmark & \cellcolor[HTML]{04B45F} $\alpha>1/2$\\
\hline
Endorsed Unanimity &  \cellcolor[HTML]{04B45F} \checkmark & \cellcolor[HTML]{04B45F} \checkmark  & \cellcolor[HTML]{FF0000} \textcolor{white}{\xmark} & \cellcolor[HTML]{04B45F} \checkmark & \cellcolor[HTML]{04B45F} $\alpha>\frac 1 {2-\epsilon}$\\
\hline

Familiar monotonicity & \cellcolor[HTML]{04B45F} \checkmark & \cellcolor[HTML]{04B45F} \checkmark  & \cellcolor[HTML]{FF0000} \textcolor{white}{\xmark} & \cellcolor[HTML]{04B45F} \checkmark & \cellcolor[HTML]{FF0000} \textcolor{white}{\xmark}\\

\hline
\end{tabular}
\caption{Highlighted, in light color, the fulfilment of additional desirable properties, on top of those in Table~\ref{tableCoherentResults}, when assuming coherent opinions and consensus on acceptance degrees.
}
\label{tableUniformAcceptance}
\end{table}
\normalsize


\comment{
The next table \ref{tableCompleteComparisonDRF} displays the results obtained for all the functions including the generalised families. \todo{TO DISCUSS: \maite{What is the aim of this table?  I THINK THIS (and the rest of this subsection) IS REDUNDANT... .........
}}

The results are shown in different ways. The green squares (or those marked by a check mark), represent those row\maite{(?)} properties fulfilled by the \natalia{aggregation} functions. In the case of the generalised functions, if needed, it is specified the interval of values for $\alpha$ needed to satisfy the property. Next, 
\maite{light green}
squares (tick between parentheses) represent those properties which \natalia{are fulfilled under the assumption of coherent opinion profiles. Again, in} the case of the generalised families it is specified for which values of $\alpha$ the property is fulfilled. Finally, a red square (or cross mark) simply means that the column function doesn't satisfy the row property. \maite{(PENDING: table notation has been explained before) }


\begin{table}[h]
\centering
\begin{tabular}{ | l | >{\centering}m{1cm} | >{\centering}p{1cm} 
| >{\centering}p{1cm} | >{\centering}p{1cm} | c |}
\hline
\rowcolor{lightgray} Properties  &  $D$ & $I$ \comment{& $B$} & $R$ & $\alpha-I$ & $\alpha-R$\\
\hline
\textbf{Collective coherence}
& \cellcolor[HTML]{FF0000} \textcolor{white}{\xmark} & \cellcolor[HTML]{FF0000} \textcolor{white}{\xmark}  \comment{ & \cellcolor[HTML]{FF0000} \textcolor{white}{\xmark}} & \cellcolor[HTML]{04B45F} \checkmark & \cellcolor[HTML]{FF0000} \textcolor{white}{\xmark} & \cellcolor[HTML]{04B45F} $\alpha<\epsilon/2$ \\
\hline
Exhaustive domain & \cellcolor[HTML]{04B45F} \checkmark & \cellcolor[HTML]{04B45F} \checkmark  \comment{& \cellcolor[HTML]{04B45F}  \checkmark} & \cellcolor[HTML]{04B45F} \checkmark & \cellcolor[HTML]{04B45F}  \checkmark & \cellcolor[HTML]{04B45F} \checkmark\\
\hline
Coherent domain & \cellcolor[HTML]{04B45F} \checkmark & \cellcolor[HTML]{04B45F} \checkmark  \comment{& \cellcolor[HTML]{04B45F}  \checkmark} & \cellcolor[HTML]{04B45F} \checkmark & \cellcolor[HTML]{04B45F}  \checkmark & \cellcolor[HTML]{04B45F} \checkmark\\
\hline
Anonymity  & \cellcolor[HTML]{04B45F} \checkmark & \cellcolor[HTML]{04B45F} \checkmark  \comment{& \cellcolor[HTML]{04B45F}  \checkmark} & \cellcolor[HTML]{04B45F} \checkmark & \cellcolor[HTML]{04B45F}  \checkmark & \cellcolor[HTML]{04B45F} \checkmark\\
\hline
Non-dictatorship & \cellcolor[HTML]{04B45F} \checkmark & \cellcolor[HTML]{04B45F} \checkmark  \comment{& \cellcolor[HTML]{04B45F}  \checkmark} & \cellcolor[HTML]{04B45F} \checkmark & \cellcolor[HTML]{04B45F}  \checkmark & \cellcolor[HTML]{04B45F} \checkmark\\
\hline
Narrow Unanimity &  \cellcolor[HTML]{04B45F} \checkmark  &\cellcolor[HTML]{FF0000} \textcolor{white}{ \xmark} \comment{& \cellcolor[HTML]{FF0000} \textcolor{white}{ \xmark}} & \cellcolor[HTML]{FF0000} \textcolor{white}{\xmark} & \cellcolor[HTML]{FF0000} \textcolor{white}{ \xmark} & \cellcolor[HTML]{FF0000} \textcolor{white}{ \xmark}\\
\hline
Sided Unanimity &  \cellcolor[HTML]{04B45F} \checkmark  &\cellcolor[HTML]{DBFB8B} (\checkmark) \comment{& \cellcolor[HTML]{DBFB8B} (\checkmark)} & \cellcolor[HTML]{FF0000} \textcolor{white}{\xmark} & \cellcolor[HTML]{DBFB8B} (\checkmark) & \cellcolor[HTML]{04B45F} $\alpha>\frac 1 {\min_i v_i(s) +1}$\\
\hline
Weak Unanimity &  \cellcolor[HTML]{04B45F} \checkmark  &\cellcolor[HTML]{DBFB8B} (\checkmark) \comment{& \cellcolor[HTML]{DBFB8B} (\checkmark)} & \cellcolor[HTML]{FF0000} \textcolor{white}{\xmark} & \cellcolor[HTML]{DBFB8B} (\checkmark) & \cellcolor[HTML]{04B45F} $\alpha>1/2$\\
\hline
Endorsed Unanimity &  \cellcolor[HTML]{DBFB8B} (\checkmark)  &\cellcolor[HTML]{04B45F} \checkmark \comment{& \cellcolor[HTML]{DBFB8B} (\checkmark)} & \cellcolor[HTML]{FF0000} \textcolor{white}{\xmark} & \cellcolor[HTML]{DBFB8B} (\checkmark) & \cellcolor[HTML]{DBFB8B} $\alpha>\frac 1 {\min_i v_i(s) +1}$\\
\hline
Monotonicity & \cellcolor[HTML]{04B45F} \checkmark & \cellcolor[HTML]{FF0000} \textcolor{white}{\xmark} \comment{& \cellcolor[HTML]{FF0000} \textcolor{white}{\xmark}} & \cellcolor[HTML]{FF0000} \textcolor{white}{\xmark} & \cellcolor[HTML]{FF0000} \textcolor{white}{\xmark} & \cellcolor[HTML]{FF0000} \textcolor{white}{\xmark}\\
\hline

Familiar monotonicity & \cellcolor[HTML]{04B45F} \checkmark & \cellcolor[HTML]{04B45F} \checkmark  & \cellcolor[HTML]{FF0000} \textcolor{white}{\xmark} & \cellcolor[HTML]{04B45F} \checkmark & \cellcolor[HTML]{FF0000} \textcolor{white}{\xmark}\\
\hline

Independence & \cellcolor[HTML]{04B45F} \checkmark & \cellcolor[HTML]{FF0000} \textcolor{white}{\xmark} & \cellcolor[HTML]{FF0000} \textcolor{white}{\xmark} & \cellcolor[HTML]{FF0000} \textcolor{white}{\xmark} & \cellcolor[HTML]{FF0000} \textcolor{white}{\xmark}\\
\hline
\end{tabular}
\caption{Comparison of the social choice properties fulfilled by the opinion aggregation functions proposed. Color code:  green means fully satisfied, 
\maite{light green} means satisfied under 
\maite{the assumption of coherent opinion profiles} and red means not satisfied. Two last columns may specify for which values of $\alpha$ the property is fully satisfied (green) or satisfied under some assumptions (
\maite{light green}).\natalia{(Maybe we can put collective coherence as the last row, so that the results in the table are ordered as in the text)} \maite{(However it was on top in all tables...)} \todo{CAREFUL: \maite{Properties are ordered differently than previous tables}}}
\label{tableCompleteComparisonDRF}
\end{table}

\natalia{I would make this explanation into sections that are based on the properties not on the results obtained. That is: Domain Properties, Unanimity properties, Anonymity, Monotonicicty properties, Independence and Collective coherence.}

At first glance, table \ref{tableCompleteComparisonDRF} shows us that Exhaustive Domain, Coherent Domain, Anonymity and Non-dictatorship are fulfilled without question by all the opinion aggregation functions proposed. This is not strange due to the general treatment of the information given to all the functions defined. All the functions don't rely on the type of profiles as input to compute an output, which implies directly that ED will be satisfied. And, it get clear at the definitions of the functions that there is no different treatment of the opinion depending on the agents, which clearly implies that Anonymity will be satisfied.\natalia{(In addition to explaining why we have this result I would also add why this is important or useful: e.g., they accept any kind of opinion profiles and they offer guarantees about treating fairly all treatment to participants.)}\\

Going on to the messy\natalia{(this sounds a bit informal)} part of the table, we will focus on the lower half and will leave the Collective Coherence row for last as its results relate to the second table \ref{tableComparisonCC2}. First of all, remark the Independence property that clearly separates those functions using only the direct opinion, so the Direct function satisfies it, and those functions using any degree of indirect opinion, thus $I,R,B_\alpha$ and $R_\alpha$ don't satisfy it. This property clearly reflects the employment of dependencies used  by the aggregation functions. 

Now look at the Unanimity results, NU, SU, WU, and EU. Regarding 
\maite{Narrow U}nanimity, we can see that this property its completely related to the Monotonicity or Independence properties due to the fact that we are imposing a direct link between the direct opinion and the outcome of the aggregation function. So, it is expected that only the Direct function fulfils NU. 

At the another side, Sided and Weak Unanimity are clearly related one to the other, though WU is more easily fulfilled than SU. We can see that as long as we are using more indirect opinion we are losing Unanimity properties. The Indirect function I relies on the coherence of the profiles in order to fulfil these two properties. And, the Recursive function, which it's even worse due to its dependency to even more far related statements (the opinion on each statement depends on the opinion at the leafs of the graph), it completely fails at all the Unanimity properties.
Regarding the generalised families, $B_\alpha$ satisfies SU, WU, and EU assuming the conditions needed for its parent functions D and I: For SU and WU inherits the restrictions of $I$, and for EU inherits the assumptions of $D$. On the other hand, $R_\alpha$, which inherits from $R$ a complete fail on Unanimity properties has to bounce its weight towards the Direct function in order to gain some Unanimity properties. Although, for SU and EU it's conditions on $\alpha$ are impractical when facing any kind of opinion profile. Relating WU, we see the half of the family more closer to $D$ ($\alpha>\frac 1 2$) satisfies the property.\natalia{I found these explanations a bit difficult to follow. Maybe having the lemmas before in the section in which we introduce the properties may help the reader to understand the relationship between the different results. }

As to Monotonicity and Familiar Monotonicity properties, we will focus on the second one, the fulfilment of Monotonicity is clearly related to the fulfilment of Independence. 

We see that FM is fulfilled by every functions less the ones using $R$. That is because of the type of restriction that the property applies to the indirect opinion: FM needs the indirect opinion regarding the descendants to be still but it doesn't bind the opinion on further statements that relate more indirectly. So, it is predictable that $I$, which only depends on the descendants opinions, fulfil FM, and $R$, which uses more far away opinions than only the ones on the descendants, does not fulfil it.

Finally, let's focus on the row relating the Collective coherence results. Clearly, the Recursive function, which is defined as the expectation of the collective opinion obtained from aggregating satisfies CC, while the other two functions, Direct and Indirect functions, does not. Not even $I$, which uses only indirect opinion, can fulfil CC. 

The family $B_\alpha$, which depends on $D$ and $I$, does not satisfy CC, while the other family $R_\alpha$, relaying on $D$ and $R$, satisfies CC for $\alpha>\frac \epsilon 2$, where $\epsilon$ is set for the coherence of the output. In this family, as more close $\alpha$ is to $0$ more coherent will get the results of $R_\alpha$, as more close to $\frac \epsilon 2$ less similar will be the indirect and direct collective opinion. 

Changing a bit the conditions of the debate, imagining a debate where the relationships are set and valued before starting the valuation process of the statements, we can produce a $DRF$ where the relationships have a defined acceptance degree equal for every agent. Many 
processes regarding the debate construction may lead to this paradigm.
In this new context, there is a big improvement regarding the fulfilment of Collective Coherence, which is reflected in next table \ref{tableComparisonCC2}. The aggregation functions that didn't satisfy CC now, assuming constant acceptance degrees, fulfil CC when restricting the domain to only coherent profiles\footnote{That is, if the coherence of the profiles is more restrictive than the coherence wanted at the output (see proposition \ref{prop:Final}}.

\begin{table}[h]
\centering
\begin{tabular}{ | l | >{\centering}m{1cm} | >{\centering}p{1cm} 
| >{\centering}p{1cm} | >{\centering}p{1cm} | c |}
\hline
\rowcolor{lightgray} Properties $w_i(c)=\lambda_c$ &  $D$ & $I$ \comment{& $B$} & $R$ & $\alpha-I$ & $\alpha-R$\\
\hline
\textbf{Collective coherence}
& \cellcolor[HTML]{DBFB8B} (\checkmark) & \cellcolor[HTML]{DBFB8B} (\checkmark)  & \cellcolor[HTML]{04B45F} \checkmark & \cellcolor[HTML]{DBFB8B} (\checkmark) & \cellcolor[HTML]{DBFB8B} (\checkmark)\\
\hline
\end{tabular}
\caption{New results for constant acceptance degrees. Color code:  green means fully satisfied, \maite{light green }
means satisfied under some assumptions and red means not satisfied. \maite{[]}
}
\label{tableComparisonCC2}
\end{table}
}

\subsection{Summary}
\label{subsec: summary}

From the analysis for each debate scenario above, we can draw the following general observations:
\begin{itemize}
    \item The aggregation functions of the recursive family achieve collective coherence provided that they place little weight on direct opinions (or opinions are coherent and there is consensus on acceptance degrees).
    
    \item Coherence in opinion profiles favours unanimity (specifically, WU and EU), though in different ways. $I$ and $\alpha$-Balanced are fully satisfied, while the family of recursive functions leans on the direct aggregation function to fulfil some unanimity properties with restrictions. 
    As a result, the $\alpha$-Recursive family only satisfy WU and EU under strong conditions on $\alpha$, because the $R$ function never satisfies them.
    
    \item Coherent opinion profiles are not enough for D, I, and $\alpha$-Balanced functions to achieve collective coherence. They also require consensus on acceptance degrees. Recursive functions do not require such consensus (in fact, not even the non-coherent opinion profiles), and hence they are \emph{robust} to the divergence of opinions on the relations between statements in a debate.
    
    \item While the D, I, $B_\alpha$ functions manage to achieve familiar monotonicity in all scenarios, the aggregation functions in the recursive family cannot even when counting on coherent opinion profiles and consensus on acceptance degrees.
    This is because the  aggregated opinion on descendants recursively depends on descendants down the relational model. Thus, changes of opinion on ``grandchildren'' or further down sentences can cause a change of opinion independently of any change of the direct opinion.
\end{itemize}

Based on these general observations above, it is the task of the decision maker to decide the aggregation operator to choose considering: (1) the features of the debate scenario at hand; and (2) the desirable properties to guarantee. As a rule of thumb, since in actual-world debates we cannot assume individual rationality (coherence), we believe that recursive aggregation functions are the best choice to achieve collective rationality, though we would pay the price of losing some other valuable properties, in particular unanimity for values of $\alpha$ that promote a large use of the direct opinion. Otherwise, if we do not value the coherence of the collective output, or we can guarantee somehow that the opinions of  participants are coherent and the participatory system at hand does not allow for divergence on acceptances degrees, the Direct function becomes the aggregation function of choice. Within such constrained settings, the Direct aggregation function fulfils almost every property considered, enve all of them in the debate scenario in section \ref{subsec:sameAcceptanceResults}. We conclude that it seems a good trade-off to consider the Recursive family,  which can behave as similar to the Direct or to the Recursive function as wanted, and set the value of $\alpha$ depending on the features and goals in hand. 

\section{Computational Analysis}
\label{sec:CompAnalysis}

The purpose of this section is twofold. First, given the 
opinion aggregation problem 
in section \ref{subsec:opinionProblem}, we explain the complexity of the different algorithms for computing a collective decision on its target. In particular, we provide an algorithm for computing the recursive aggregation function. Thereafter, in section \ref{subsec:empirical_analysis}, we empirically analyse the use of that algorithm to solve real-world collective decision problems.

\subsection{Computing Aggregation Functions}

\begin{algorithm}
 \footnotesize
\caption{Compute recursive aggregation}\label{alg:recursive}
 \begin{algorithmic}[1]
 \Function{ComputeRecursiveAggregation}{$\langle \sen, \rel, T \rangle,(O_1,\ldots,O_n)$}
 \For{each relationship $r \in \mathcal{R}$} \Comment{Compute averaged acceptances}
 \State aggregated\_acceptance[r]$ \gets $ average\_acceptances($w_1(r),\ldots,w_n(r)$)
 \EndFor
  \State $\mathcal{H}(\langle \sen, \rel, T\rangle) \gets$ DRF\_to\_B-hypergraph($\langle \sen, \rel, T \rangle$) \Comment{Generate B-hypergraph reprentation of DRF}
 \State sorted\_sentences $ \gets$ reverse(topological\_sorting($\mathcal{H}(\langle \sen, \rel,T\rangle$)))\Comment{Compute topological sorting of the DRF B-hypergraph}
 \For{s in sorted\_sentences}\Comment{Compute aggregated valuations}
 \State valuation[s] $\gets$ 0 \Comment{To accumulate aggregated valuations over descendants}
 \State normaliser[s] $\gets$ 0 \Comment{To normalise aggregated valuations over descendants}
 \State compute relationships $R(s)$ to descendants 
 \If{$R(s) \neq \emptyset$}\Comment{if s has descendants} 
 \For{each relationship $r \in R(s)$}
 \State $s_r \gets$ descendant from relationship $r$ 
 \State valuation[s] $\gets$ valuation[s] + aggregated\_valuation[$s_r$] $\cdot$ aggregated\_acceptance[r]
  \State normaliser[s] $\gets$ normaliser[s] + aggregated\_acceptance[r]
 \EndFor
 \State valuation[s] $\gets$ valuation[s] / normaliser[s]
 \Else \Comment{s has no descendants}
 \State valuation[s] $\gets$ average\_valuations($v_1(s),\ldots,v_n(s)$)
 \EndIf
 \State aggregated\_valuation[s] $\gets$ valuation[s]
 \EndFor
 \State \textbf{return} aggregated\_valuation,aggregated\_acceptance 
 \EndFunction
 \end{algorithmic}
 \end{algorithm}
\normalsize

All of the aggregation functions proposed in Section \ref{sec:OpAggFunc} can be calculated by tractable algorithms. For example, the direct function calculates the average for all statements and relationships in a $DRF$ $\langle \sen, \rel, T \rangle$ considering the direct opinions in an opinion profile $P=(O_1,\ldots,O_n)$. Hence, its complexity is given by  $\mathbf{O}((|\mathcal{R}|+|\mathcal{S}|)\times|P|)$, where $|\mathcal{R}|, |\mathcal{S}|$ are the number of relationships and statements, respectively; and $|P|$ is the number of opinions in an opinion profile. 
Computing the indirect and balanced functions can be done by calculating the aggregated acceptance of each relationship as an average and by calculating the aggregated valuation of each statement as the average of the estimation function, which in turn is an average of the indirect opinions for that statement. Hence, their complexity is given by  $\mathbf{O}(|\mathcal{R}|\times|\mathcal{S}|\times|P|)$. The calculation of the recursive function can be done by calculating the aggregated acceptance of each relationship as an average and calculating the aggregated valuation of each statement by starting with statements with no descendants and using these results to calculate the aggregated valuation of the statements directly connected to them. For example, algorithm \ref{alg:recursive} contains the pseudocode for the recursive function. 
In particular, the algorithm starts by computing aggregated acceptances ($w_{R(P)}$) as a weighted average (lines 2-3), which has a complexity of $\mathbf{O}(|\mathcal{R}|\times |P|)$.  
Then, the algorithm computes aggregated valuations ($v_{R(P)}$) starting from the statements with no descendants. In order to do that, we first perform a topological sorting of the DRF. This can be achieved by: (1) transforming the graph associated to the DRF into an acyclic B-hypergraph\footnote{B-hypergraphs are a particular type of hypergraph with efficient algorithms for path finding\cite{gallo1993directed}. Obtaining a B-hypergraph from a DRF is straightforward. In fact, the graph associated with a DRF is a B-hypergraph with the exception of the relationships that connect the very same statements. For instance, consider relationships $r_1$ and $r_6$ in figure \ref{pic:ex_DRF} linking $\tau$ to $s_1$. Since in a hypergraph there cannot be two or more hyperedges over the very same nodes, we will only consider one single hyperedge. In our example, it suffices to consider either $r_1$ or $r_6$. We do not lose anything by doing this simplification because we want to obtain the topological sorting of a DRF, and hence considering one of the relationships connecting the very same statements is enough.}, denoted by $\mathcal{H}(\langle S,\mathcal{R}, T\rangle)$ (line 4); and (2) then performing the topological sorting over the B-hypergraph (line 5).
Starting from the sentences without descendants, the algorithm computes aggregated valuations until reaching the statements in $T$ (lines 5-18). The calculation of the topological sorting for B-hypergraphs has been studied in \cite{gallo1993directed}\footnote{In particular, \cite{gallo1993directed} provides an algorithm to calculate the inverse topological sorting in a F-hypergraph. Any given B-hypergraph can be transformed into a symmetric F-hypergraph by changing the direction of the hyperedges. Note the inverse topological sorting of the symmetric F-hypergraph coincides with the topological sorting in the original B-hypergraph. } and has  a complexity of $\mathbf{O}(|\mathcal{R}|\times|\mathcal{S}|)$. The calculation of the aggregated valuations has a complexity of $\mathbf{O}(|\mathcal{S}|\times max(|\mathcal{R}|,|P|))$. In real debates the number of opinions is usually significantly higher than the number of relationships, hence, the complexity of the calculation of aggregated valuations is given by $\mathbf{O}(|\mathcal{S}|\times|P|)$.  Thus, the total complexity of the recursive functions is $\mathbf{O}(|\mathcal{R}|\times|P|)$.

In \url{https://bitbucket.org/jariiia/workspace/projects/DRF}
we provide a publicly-available implementation of algorithm~\ref{alg:recursive} together with all the aggregation functions defined in this paper. Furthermore, we also provide guidelines to reproduce the experiments reported in section \ref{subsec:empirical_analysis} below.

\subsection{Empirical Analysis}
\label{subsec:empirical_analysis}

In what follows we empirically analyse the time required by our implementation of algorithm \ref{alg:recursive} to compute collective decisions. Our purpose is to investigate whether our  approach can handle collective decision making in practice.

Based on the analysis above, we generated debates, which required to the synthetic generation of DRFs and opinion profiles. On the one hand, we artificially generated DRFs whose sentences are the nodes of a directed acyclic B-hypergraph and whose hyperedges represent the relationships between sentences. We chose the number of sentences in our synthetic DRFs within $\{100,150,200\}$ to represent small, medium and large scenarios. 
Regarding the relationships between sentences, we considered two parameters: 
\begin{itemize}
    \item \emph{Density of relationships.} Given a relationship $r = (\Sigma,s)$, we say that the number of sentences in $\Sigma$ is the density of $r$. Since each relationship is represented as a hyperedge in a B-hypergraph, the density of relationships amounts to the size of the tails of hyperedges in the hypergraph. The (average) density of relationships in our artificial DRFs took values within $\{1,2,3\}$. We set the density value to 1 to generate DRFs for which there is a one-to-one connection between sentences, and so each DRF is in fact a DAG. As to the other two density values (2 and 3), they allow us to generate DRFs where each relationship has two sentences connected to one sentence, and three sentences connected to one sentence respectively.
    \item \emph{Density of number of relationships}, namely the average number of sentences to which each sentence is connected to through relationships. This corresponds to the average out degree of each sentence in the DRF. We chose values for this parameter within  $\{1,2.5,5\}$ to generate DRFs with low, medium, and high density of relationships.
\end{itemize}

On the other hand, to finish generating a debate, we must generate opinion profiles. We generated profiles with number of opinions within $\{10^6,3\cdot10^6,5\cdot10^6\}$, 
to represent the largest known actual-world scenarios\footnote{To the best of our knowledge, the Brexit discussion on UK  \cite{petitionsBrexit} constitutes the largest such discussion: news outlets reported when the number of supporters passed 2 million \cite{BBCnewsBrexit} and the numbers kept growing during the 6 month period that the discussion was open. By the time it closed, there were 6,103,056 participants \cite{petitionsBrexit}. Contrasting numbers of participants can be found for other popular initiatives such as an environmental proposal in Parlement et Citoyens which had 51,493 votes \cite{parlement-et-citoyensEnvironment}
and in the participatory budgeting process in Helsinki \cite{Helsinki19} with 54,246 registered people, which represents 10\% of the city voters. Note that the Parlement et Citoyens and Helsinki debates are probably more representative of real online debates than the Brexit example, where participants were, in effect, just voting on a specific proposal.}. The values for valuations and acceptances were randomly generated within $[-1,1]$ and $[0,1]$ respectively.  

All the computations of collective decisions four our artificially generated debates were performed on an Ubuntu 16.04 box with an Intel(R) Core(TM) i7-8700K CPU @ 3.70GHz, with 31GiB system memory, and 8th Gen Core Processor Host Bridge/DRAM R. 
Furthermore, our experiments only considered the recursive aggregation function (specified in algorithm \ref{alg:recursive}) because it is the more expensive out of those introduced in section \ref{sec:Aggregation}.

We performed three types of analysis:
\begin{itemize}
    \item \emph{Sensitivity to number of participants.} Figure \ref{pic:participant_sensitivity} shows that the time to compute collective decisions increases as the number of participants increases.   The figure shows the results for a medium density of number of relationships and a low density of relationships (which is the most expensive case as we discuss below).
    \item \emph{Sensitivity to density of number of relationships.} Figure \ref{pic:density_sensitivity} shows that the time to compute collective decisions increases as the density of number of relationships grows. Notice though most actual-world scenarios would lie between the low and medium cases, and hence it would take less than one second to solve even the largest debate.
    \item \emph{Sensitivity to density of relationships.} Figure  \ref{pic:relationship_size_sensitivity} shows that the time to compute collective decisions decreases as the density of relationships increases. The figure shows the results for $3 \cdot 10^6$ opinions and a medium density of number of relationships. This tells you that in fact, our algorithm needs more time when relationships between sentences are one-to-one.
\end{itemize}

Overall, notice that computing collective decisions in all the artificially generated debates took less than 1.6 seconds. Therefore, we can conclude that our opinion aggregation functions can be employed to cope with large-scale debates in real time.


\begin{figure}[H]
    \centering
    \includegraphics[width=7cm]{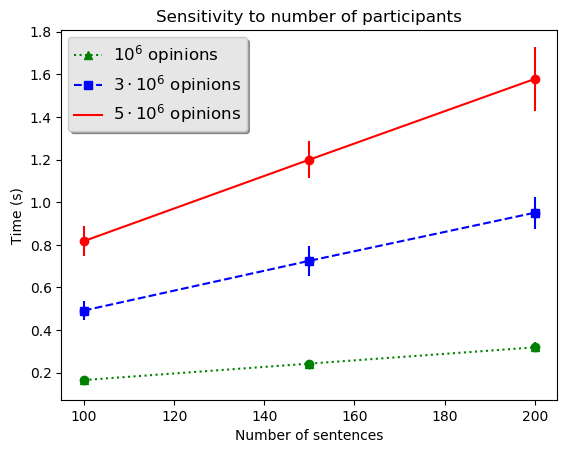}
    \caption{Sensitivity to number of participants in a DRF. The results show the computational time required by the recursive aggregation function as the number of sentences grow.}
    \label{pic:participant_sensitivity}
\end{figure}

\begin{figure}[H]
    \centering
    \includegraphics[width=7cm]{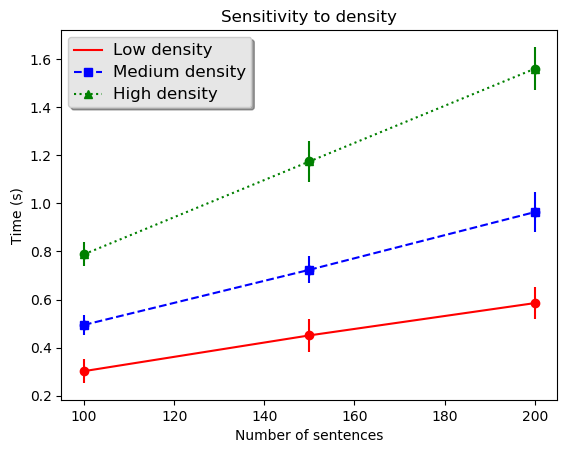}
    \caption{Sensitivity to density of relationships in a DRF. Density is considered in terms of the average out degree of sentences: low (2.0), medium (5.0), high (10.0). The results show the computational time required by the recursive aggregation function for $3 \cdot 10^6$ opinions as the number of sentences grow.}
    \label{pic:density_sensitivity}
\end{figure}

\begin{figure}[H]
    \centering
    \includegraphics[width=7cm]{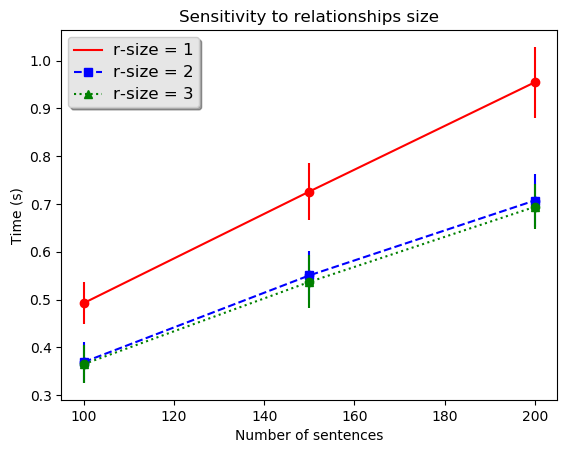}
    \caption{Sensitivity to the number of sentences in the source of a relationship (size of the tail of hypergraph edges). The results show the computational time required by the recursive aggregation function as the number of sentences grow and as the size of hyperedges (tails) grows.}
    \label{pic:relationship_size_sensitivity}
\end{figure}





\section{Related Work}
\label{sec:RelatedWork}

In this section we review the work from the literature that is closest to the work in this paper. That includes work which has focused on support for online discussions, on approaches for computing the outcome of a set of arguments, and on analysing the behaviour of discussions from the standpoint of social choice theory. We discuss work in these three areas separately below, acknowledging that 
some work could validly be considered under more than one heading.

\subsection{Tools for online discussion} 

As mentioned above, we see this work as being inspired by work on online discussion forums such as Decidim Barcelona \cite{decidim}, Better Reykjav\'{\i}k \cite{better-reykjavik}, Decide Madrid \cite{decideMAD}, and Parlement et Citoyens \cite{parlement-et-citoyens}, where participants can carry out a structured discussion around some topic, typically a policy proposal. 
These particular sites allow participants to offer arguments for and against a proposal , and vote(/support) for it in the context of a specific institution. However, some of the citizen participation tools behind some of these sites, such as Decidim \cite{decidimORG} for Barcelona or Consul \cite{consul} for Madrid, have been used in other cities and organisations. Thus, for example, City of Helsinki \cite{Helsinki19} applies an instance of Decidim and the New York City Participatory Budgeting \cite{NewYorkPartBudgeting} is based on Consul. Other participatory tools have also proliferated outside the context of public institutions. There,
we find\maite{:} 
\texttt{consider.it} \cite{consider-it}, 
Appgree \cite{appgree}, and Baoqu \cite{baoqu}  whose  main focus is  scalability -- making the systems fit for use by large numbers of participants \maite{--; } 
Loomio \cite{jackson2016open,loomio}, where participants can both comment on proposals, albeit in an unstructured way, and also vote on them; or Kialo \cite{kialo}, which organises debates in an structured way.
 What distinguishes our work from all of these approaches is that, as discussed in Section~\ref{sec:ModelView}, we provide a much more expressive framework than any existing participation system. First, the DRF makes it possible to express opinions about both statements and relationships between statements, and, similarly to \cite{Dunne2011,Leite2011,baroni2015automatic}, it allows opinions to be expressed using real values, unlike most e-participation systems, which limit to allowing users to express agreement or disagreement with arguments. Second, our framework has been conceived to support relationships involving multiple statements, hence going beyond the limiting pairwise relationships supported by current e-participation systems.

There is also a long-standing line of work which develops tools to map the structure of arguments on some topic.
This line of work draws from a range of sources, nicely summarized in \cite{shum2003roots}, and is exemplified by 
\cite{carr:visual,reed2004araucaria,suthers1995belvedere,vangelder2003enhancing}. 
The focus of this work is on drawing the relationships between arguments as a means of helping people understand the scenarios, and as \cite{benn2011argument,iandoli2014socially,rinner2006argumentation} point out, the resulting maps can be used to support group decision-making.
However, all these approaches deal only with visual \emph{representation} of the arguments --- there is no attempt to compute a summary of the discussion.
In contrast, our focus is on using the results of debate as input to a computational process, rather than providing support for the debate itself. 
In that sense our work could be viewed as a post-processing stage that could be applied in conjunction with any of the tools listed above to support structured discussion.
Indeed, the generality of the relational model means that it is well-fitted to such a task.

There are other approaches that allow for structured argument-based discussions 
and aim to compute the outcome of the discussion.
One notable body of research here is Klein's work on the Deliberatorium \cite{klein2012enabling,klein2015roadmap} which allows for the presentation of arguments and their interactions, and aggregates the opinions.
Unlike our work, there is no analysis of the properties of the aggregation with respect to social choice principles.

In contrast to the work discussed so far, where participants in the debate have the task of structuring their arguments into the correct format, \cite{cabrio2013natural} considers extracting arguments from natural language texts and constructing a formal argumentation representation from them. 
Such a representation can then be summarised as discussed in \cite{rajendran2016assessing}.
\cite{baroni2015automatic} discusses how this kind of approach can be combined with the argumentation-visualization methods described above.

\subsection{Computational argumentation}

Computational argumentation \cite{Rahwan2009} has a lengthy history within artificial intelligence, going back at least as far as \cite{fox:barber:bardhan:mim,mcguire1981opportunistic}. 

At the time of writing, work in this area is split into two broad groups.
First, historically, is work which is concerned with the internal structure of arguments --- what arguments are constructed from, and how this construction takes place.
Early examples of this work include those cited above, along with \cite{fox:krause:elvang:uai93} and \cite{krause:ci}, \cite{loui:ci}, \cite{parsons:ijcqqpr97}, and \cite{prakken:sartor:jancl}.
This line of work has reached its current endpoint with \emph{structured argumentation} systems like logic-based argumentation \cite{besnard:hunter:aij}, assumption-based argumentation \cite{dung:kowalski:toni:aij} and structured argumentation systems such as \textsc{aspic+} \cite{modgil:prakken:ai}, and DeLP \cite{garcia:simari:tplp}.
Second is the line of work on 
\emph{abstract argumentation}, begun by \cite{Dung1995}, which focuses much less on the internal structure of arguments, and instead is mainly concerned with the relationships between arguments. 
This has led to a large body of work expanding on \cite{Dung1995}, for example \cite{baroni:giacomin:semantics,modgil:caminada:proof,vreeswijk:ai}.
In \cite{Dung1995}, the focus is solely on ``attack'' relations, where
arguments are in conflict, and subsequent work has expanded the scope to consider ``support'' relations as well \cite{amgoud2008bipolarity,cayrol:gradual,Cayrol2005}.

There is another way to, broadly, classify work on argumentation into two groups.
One line of work, again exemplified by \cite{Dung1995}, focuses on argumentation as a mechanism for extracting consistent points of view from an inconsistent knowledge base. 
The other line of work deals with how arguments combine, or accrue, in favour of, or against some conclusion. 
This distinction cuts across the structured/abstract distinction with, for example \cite{baroni:giacomin:semantics} being concerned with consistency in abstract argumentation, and \cite{modgil:prakken:ai} dealing with consistency in structured argumentation.
On the other hand,  \cite{besnard:hunter:aij,prakken:icail05,verheij:dgnmr95} discuss accrual in structured argumentation, while \cite{cayrol:gradual} looks at accrual in abstract argumentation.

The work mentioned above uses argumentation as a mechanism for a single entity to come to a conclusion. 
However, as \cite{sycara:decision,walton:krabbe:book} and others have pointed out, argumentation is also a natural mechanism for multiple entities to use to reach consensus on some topic. 
As a result, argumentation has been used \cite{amgoud:maudet:parsons:icmas00,mcburney:parsons:inbook} in multiagent systems as a mechanism for \emph{rational interaction} \cite{mcburney:thesis} for a particular meaning of ``rational''. 
That is ``rational'' in the sense that each stage in the interaction is supported by well-founded reasons.
Here we build upon this prior work in rational interaction. Our approach allows participants in a discussion to put forward opinions about some topic under discussion, opinions that may be either in favour or against the topic.

Our work connects to several of these themes in argumentation.
First, since we are interested in aggregating the opinion from a number of participants, our work is clearly related to the use of argumentation in multiagent interaction. 
Second, the fact that individual steps in the participants' reasoning process are represented in our approach means that our work is connected to work on structured argumentation. (We would argue that it is more abstract, since the relationships that connect statements are not restricted to be rules as in \cite{dung:kowalski:toni:aij,garcia:simari:tplp,modgil:prakken:ai}). Third, our work connects with the idea of argumentation as a means of extracting a coherent view from a number of conflicting opinions.
The fact that this coherent view takes into account the votes of participants also gives our work a fourth connection with the argumentation literature, in its relation to work on accrual.
This voting aspect also places our work in close relation to that of social argumentation \cite{Leite2011,RagoToni2017}, and previous work on collective argumentation\footnote{Note that our work, and \cite{Awad:2015:JAM,caminada2011judgment,Ganzer2018} has little commonality with the ``collective argumentation'' of \cite{bochman2003collective}, which is concerned with argumentation in which relationships exist between sets of arguments.} \cite{Awad:2015:JAM,caminada2011judgment,Ganzer2018} which we will discuss more below.
Finally, one might view our work as being about the combination of different sets of arguments, one for each person who votes on the arguments or the relationships between them.
From that perspective, our work also connects with that of  \cite{Coste-Marquis2007}, which takes as input different sets of arguments and relationships between then, and outputs consistent sets of arguments, thus ``merging'' the input sets.
See \cite{bodanza2017collective} for a survey of work on this topic, and \cite{chen2019preservation} for an excellent short overview of developments to date.

\subsection{Social choice theory}
\label{subsec:socialChoice}

Given a set of alternatives and a set of agents who possess preference relations over the alternatives, social choice theory focuses on how to yield a collective choice that appropriately
reflects the agents’ individual preferences \cite{aziz2017computational}. With this aim, social choice theory has extensively explored many ways of aggregating agents’ individual preferences \cite{gaertner2009primer}.
Since there is a consensus in the literature on the desirable properties that a ``fair'' way of aggregating preferences should satisfy (e.g. no single agent can impose their view on the aggregate; if all agents agree, the aggregate must reflect the agreement; and so on \cite{gaertner2009primer}), aggregation functions can be characterised and compared in terms of the desirable properties they satisfy. Notice though that social choice theory counts on multiple \emph{negative} results, namely impossibility results showing the incompatibility of certain sets of desirable properties (e.g. Arrow’s famous impossibility theorem \cite{Arrow2012}).

Much of the work in social choice theory has placed little emphasis on the structure of the objects over which agents are expressing their preferences.
However, there is a growing body of research that takes the subject of the preferences to be arguments in some form or other. 
Here the foundational work was that of \cite{rahwan2010collective}, later developed in  \cite{Awad:2015:JAM}, which considered --- from one perspective ---  the very same problem that we tackle here: given a topic of discussion and a set of agents expressing their individual opinions about the statements made in the discussion, how can the agents reach a collectively rational decision?. 
The way that this is tackled in \cite{Awad:2015:JAM}, as \cite{bodanza2017collective} explains, is as a version of the 
``merging'' problem mentioned above. 
That is, \cite{Awad:2015:JAM} consider each participant to have a set of arguments, and the relationships between them, and an opinion about which arguments are labelled as being acceptable and which are not. 
The problem they then solve is how to compute a set of labels for the arguments that reflect the opinions of all the participants such that the aggregation of opinions satisfies desirable social choice properties
The same problem was, subsequent to \cite{rahwan2010collective} but before the extended treatment in \cite{Awad:2015:JAM}, considered by \cite{caminada2011judgment}, and more recently by us in \cite{Ganzer2018}.
As \cite{awad2017experimental} points out, \cite{Awad:2015:JAM} and \cite{caminada2011judgment} take different approaches, with \cite{Awad:2015:JAM} considering the opinions as votes, resolved by taking the plurality for individual arguments, while \cite{caminada2011judgment} offers a range of operators that yield a labelling which confirms to the constraints of argumentation semantics\footnote{The resulting labellings are either admissible or complete \cite{Dung1995}.} while also not disagreeing with the opinions of any participant. \cite{awad2017experimental} compares the plurality approach with one of the operators \cite{caminada2011judgment} using human participants. 

The recent work in \cite{chen2019preservation} can be viewed as an extension of the line of work in \cite{Awad:2015:JAM,caminada2011judgment}. Like  \cite{Awad:2015:JAM,caminada2011judgment}, \cite{chen2019preservation} proposes methods for aggregating a collection of individual argumentation frameworks, corresponding to participants in a debate, into a single argumentation framework that appropriately reflects the views of the group as a whole. 
\cite{chen2019preservation} investigates the properties of the aggregation rules  introduced in the paper, and, like ours, the work employs techniques from social choice theory in this analysis.
However, the aim is different.
In \cite{chen2019preservation} the aim is to analyse aggregation rules in terms of their preservation of \emph{semantic properties} of argumentation framework, while our focus in this paper is on \emph{social choice properties} of the aggregation operators.

From the perspective of argumentation, the major difference between \cite{Awad:2015:JAM,caminada2011judgment,chen2019preservation} and our work in \cite{Ganzer2018} and this paper is that we do not start from a set of opinions that are well-formed in an argumentation sense. \cite{Ganzer2018}, which does directly use the notion of arguments and labels (though the work is based on a variation on bipolar argumentation \cite{amgoud2008bipolarity}) does not, unlike \cite{Awad:2015:JAM,caminada2011judgment}, start from a ``legal'' \cite{baroni:caminada:giacomini:ker} labelling. Similarly here we do not require any consistency of  opinions. In both cases this is because we want to represent human opinions that may not be rational in an argumentation-theoretic sense.

Finally, a rather different line of work (and one that seems to have no knowledge of \cite{Awad:2015:JAM,caminada2011judgment,chen2019preservation}) is that of \cite{RagoToni2017}, which extends the QuAD framework first introduced in \cite{baroni2015automatic} to deal with multiple agent debates.
Similarly to our work in \cite{Ganzer2017,Ganzer2018}, the QuAD-V framework in \cite{RagoToni2017} allows pro and con arguments (attackers and defenders in our terminology) and agents' votes  over arguments (labels). 
However, the main focus of the work is not on computing a collective opinion, but on the agents contribute with individually rational opinions, a weaker version of our notion of coherent labelling defined in \cite{Ganzer2018}. 
This contrasts with our focus on the design of aggregation functions that satisfy desirable social choice properties, particularly collective rationality (strict rationality in Rago and Toni's terms), without requiring agents to be individually rational.

Beyond that, the major difference between the work in this paper and what has been done before is the richness of the representation. Here there are three main extensions.
First, \cite{Awad:2015:JAM,caminada2011judgment,chen2019preservation,Ganzer2018} and \cite{RagoToni2017} all deal with abstract arguments. 
Here we deal with structured objects, and, as already mentioned, these are objects that are more general than the usual object studied in structured argumentation since we place no real constraints on the kind of reasoning captured by the relationships that hold between statements.
Second, unlike previous work, we allow opinions to be expressed both about individual statements and the relations between them. 
This combines what has been studied in \cite{Awad:2015:JAM,caminada2011judgment,chen2019preservation,Ganzer2018,RagoToni2017}, where opinions are expressed about individual arguments, but not the relations between them (these are assumed to be fixed), and \cite{Leite2011}, where opinions are expressed about the relationships between arguments but not about the arguments themselves.
Third, we allow opinions to be real-valued. In this we move away from the labellings studied in \cite{Awad:2015:JAM,caminada2011judgment,Ganzer2018}, and their grounding in argumentation semantics, and towards the kind of representation allowed in \cite{Leite2011,RagoToni2017}.

Finally, following the assumptions already established in \cite{Ganzer2018}, we neither assume independence between arguments as a fundamental postulate as is the case in \cite{Awad:2015:JAM,caminada2011judgment,chen2019preservation}, nor do we require the resulting aggregation to agree with an argumentation semantics. We will deal with these differences in turn.

Dropping the independence of arguments should come as no surprise since \cite{Awad:2015:JAM} questions the necessity of assuming independence because of the dependencies between arguments that come already encoded in the form of relationships such as attack. 
Despite the importance of  independence as a fundamental property in the judgement aggregation literature because of its theoretical value in proving strategy-proofness and strategic manipulation\footnote{If the independence criterion is not satisfied, then the function aggregating judgements is not immune to strategic manipulation \cite{dietrich2007strategy}.}, we are not alone in regarding independence as too strong a property. This is because together with mild further conditions, it implies dictatorship \cite{lang2016agenda} and because it is also considered as not very plausible \cite{mongin2008factoring}. This explains why relaxing independence has been subject of much research (e.g. \cite{dietrich2010premiss,lang2016agenda,mongin2008factoring,pigozzi2008independence}). 
This paper goes beyond relaxing independence.
Rather, in this paper we introduce several opinion aggregation functions that use the participants' opinions to compute a collective opinion while considering \emph{dependencies} between statements.
This is in line with our former work in \cite{Ganzer2018}, but here we allow to express dependencies between multiple  statements per relationship, while \cite{Ganzer2018} (and all the work which starts with abstract argumentation frameworks such as \cite{Awad:2015:JAM,caminada2011judgment,chen2019preservation,RagoToni2017}) constrain relationships, and thus dependencies, to exist between pairs of arguments, hence limiting expressiveness. 

Turning to the fact that the result of our aggregations do not match an standard argumentation semantics, we substitute the notion of ``coherence of opinions'' for the form of rationality embodied by those semantics. 
We do this for reasons that we have already touched on above with respect to the input opinions --- we feel that insisting on an output that is rational in an argumentation theoretic sense is not necessarily realistic given that we start from opinions that are put forward by human participants who may not be not consistent in their views. Instead of forcing the output of aggregation to be rational in an argumentation theoretic sense, we instead compute a measure which assesses how much concordance there is between related opinions, and assess our novel aggregation functions by the degree to which they can assure that their output ensures collective coherence.
 

Finally, from a pure social choice (not a combined argumentation and social choice) perspective notice, that it is common in the literature on judgement aggregation and preference aggregation to impose properties on the objects under aggregation in order that aggregation operators can guarantee desirable properties. 
For instance, in the case of distance-based aggregators, the \emph{Kemeny} rule \cite{endriss_moulin_2016} only considers consistent judgement sets, and hence disregards those which are not, whereas premise-based aggregators \cite{endriss_moulin_2016} typically make assumptions on the agenda to guarantee consistency and completeness. 
Our work does not rely on this structuring of the target objects of the aggregation operators. 
Instead, we have introduced aggregation operators capable of guaranteeing collective coherence when opinions are unconstrained.
This is motivated by the need for disregarding rationality when humans are involved in debates, since their opinions may eventually contain  contradictions and inconsistencies.



\section{Conclusions and future work}
\label{sec:conclusions}

There is an emerging interest in the use of information and communication systems to allow citizens to  participate in the governance process.  
These e-participation systems aim to strengthen the collaboration between governments and citizens, and to empower citizens in participating in policy decision-making. 
The majority of such systems allow citizens to express their views on particular issues, but they either fail to aggregate the different views in a meaningful way, or they aggregate in a way that limits users to just voting in favour of or against an issue. 
Current research in the areas of social choice and argumentation could help to formally represent richer human debates in e-participation systems and to define computational processes for aggregating users views. 
However, approaches belonging to these areas have two main drawbacks that significantly decrease their potential as a solution for real e-participation systems. 
Firstly, they assume that participants agree on the structure of a debate, that is the pieces of information that are relevant to particular issues and  the relationships among them. 
Secondly, they assume that users opinions are rational, which is too restrictive for human participants.

To address these limitations, we have proposed a model to represent and reason about human debates. 
Our model allows users not only to introduce new pieces of information to the discussion and relate them to existing pieces but also to express disagreement with the relationships put forward by other users. 
In addition, our model does not assume that users opinions are rational and we define a weaker notion of rationality  to characterise \textit{coherent} user opinions. 
Considering the degree of coherence of individual users opinions and the level of consensus that users have about the debate structure among users, we analyse the outcomes of different opinion aggregation functions in terms of social choice properties. 
Our analysis demonstrates that the recursive aggregation method could compute a coherent collective opinion even if individual opinions are incoherent and there is a lack of consensus on the debate structure. 
As we impose more restrictions on the coherence of individual opinions and consensus among users on the debate structure, more aggregation methods also compute coherent collective opinions. 
We conclude our analysis with a computational evaluation in which we study the computational cost of aggregating collective opinions and experimentally demonstrate that collective opinions can be computed efficiently for real-sized debates. 

There are several lines of future work that we have in mind. On the theoretical side, we are investigating methods to analyse the quality of a debate represented in the relational model. 
For example, we are studying the use of systematic incoherence in participant opinions as a way to identify structural problems in a debate. 
On the practical side, we are planning to evaluate the practical impact of using the relational model in real online debates.
We have already evaluated \cite{serramia:ccia19} an earlier debate model \cite{Ganzer2018} using data from Decidim Barcelona, and intend to do the same for the relational model.

\section*{Acknowledgements}
This work was supported by the UK Engineering \& Physical Sciences Research Council (EPSRC) under grant \#EP/P010105/1 and by projects LOGISTAR (H2020-769142), AI4EU (H2020-825619), Crowd4SDG (H2020-872944), CI-SUSTAIN (PID2019-104156GB-I00), PGC2018-096212-B-C33,  2017 SGR 172, and 2017 SGR 00341. 
Jordi Ganzer was supported by a PhD studentship from King's College London.

\newpage

\appendix

\section{Formal proofs and results}\label{sec:Proofs}
\comment{}

\renewcommand\thesection{\Alph{section}}

In the following, we prove all the formal results presented in section \ref{sec:Analysis} regarding the satisfaction of social choice properties by the opinion aggregation functions introduced in section \ref{sec:OpAggFunc}.The section is divided into four parts, one per debate scenario as analysed in section \ref{sec:Analysis}. 
\begin{enumerate}
    \item Unconstrained opinion profiles;
    \item Constrained opinion profiles: assuming consensus on acceptance degrees;
    \item Constrained opinion profiles: assuming coherent profiles; and
   \item Constrained opinion profile: assuming consensus on acceptance degrees and coherent profiles.
\end{enumerate}
Furthermore, for each scenario, our results will be grouped by aggregation function according to the following order: Direct aggregation, Indirect aggregation, Recursive aggregation, Balanced family aggregation and Recursive family aggregation



\subsection{Unconstrained opinion profiles}
\label{Proofs:general}

In this section we analyse the social choice properties fulfilled by the aggregation functions introduced in Section \ref{sec:OpAggFunc}: assuming unconstrained opinions profiles (\emph{any} opinion profile is deemed to be possible input for the aggregation functions).
The results of this section are summarised in table \ref{tableGeneralResults} in section \ref{subsec:generalResults}.


\begin{prop}\label{prop:D}
The aggregation function $D$ satisfies the following properties:
\begin{enumerate}[(i)]
    \item Exhaustive Domain and Coherent Domain;
    \item Anonymity and Non-Dictatorship;
    \item Monotonicity and Familiar Monotonicity;
    \item Narrow Unanimity, Sided Unanimity and Weak Unanimity; and
    \item Independence.
\end{enumerate}
And does not satisfy:
\begin{enumerate}[(i)]
\setcounter{enumi}{5}
    \item Collective coherence; and
    \item Endorsed Unanimity.
\end{enumerate}
\end{prop}

\begin{proof}(of proposition \ref{prop:D})
\begin{enumerate}[(i)]
    \item Exhaustive Domain is straightforward and Collective Domain follows directly.
    
    \item Anonymity and Non-Dictatorship. Let $P=(O_1,\dotsc,O_n)$ be an opinion profile over a $DRF$ and $\sigma$ a permutation over a set of agents $Ag=\{1,\dotsc,n\}$. We must show that $D$ maintains the same collective opinion over the permuted opinion profile $P'=(O_{\sigma(1)},\dotsc,O_{\sigma(n)})$, i.e. that $D(P)=D(P')$. This is the case because the next two equalities hold:
$$v_{D(P)}(s)=\frac1n\sum_{i=1}^n v_i(s)=\frac1n\sum_{i=1}^n v_{\sigma(i)}(s)=v_{D(P')}(a);$$
$$w_{D(P)}(r)=\frac1n\sum_{i=1}^n w_i(r)=\frac1n\sum_{i=1}^n w_{\sigma(i)}(r)=w_{D(P')}(r).$$
Therefore, Anonymity holds and Non-Dictatorship follows from it as we discussed in section \ref{subsec:socialChoiceProperties}.

    \item Monotonicity and Familiar Monotonicity. Let $s$ be a statement and $P$ and $P'$ two opinion profiles satisfying the Monotonicity assumptions in the definition of the property in section \ref{subsec:socialChoice}, i.e. $P=(O_1,\dotsc,O_n)$ and $P=(O'_1,\dotsc,O'_n)$ are such that $v_i(s)\leq v'_i(s)$ for every agent $i\in \{1,\ldots,n\}$. Then, from the definition of $D$, we obtain the aggregated valuation on $s$ is: 
$$v_{D(P)}(s)=\frac1n\sum_{i=1}^n v_i(s)\leq \frac1n\sum_{i=1}^n v'_i(s)= v_{D(P')}(s)$$
Therefore, $D$ satisfies Monotonicity. Hence, from this and lemma \ref{lem:monotonicity}, Familiar Monotonicity also holds. 

    \item Narrow Unanimity, Sided Unanimity and Weak Unanimity. Let $P=(O_1,\dotsc,O_n)$ be an opinion profile over a $DRF$ and a statement $s\in\sen$ such that $v_i(s)=\lambda$ for every agent in $Ag = \{1,\dotsc,n\}$. The aggregated opinion on $s$ is:
$$v_{D(P)}(s)=\frac1n\sum_{i=1}^n v_i(s)=\frac1n\sum_{i=1}^n \lambda=\lambda$$

   Hence, Narrow unanimity is fulfilled by $D$. As discussed in section \ref{subsec:socialChoice}, Weak Unanimity follows from Narrow Unanimity. Furthermore, according to proposition \ref{prop:NU-SU} Sided Unanimity follows from Narrow Unanimity and Monotonicity.

    
\item Independence follows directly from the fact that $D$ satisfies Monotonicity and from proposition \ref{prop:MtoI}.

\item Collective Coherence. To prove that it does not hold, it suffices to find a $DRF$ and an opinion profile for which there is no collective coherence. Thus, consider the example depicted below in figure \ref{fig:DCountCC}.

\begin{figure}[H]
\centering
\begin{tikzpicture}[node distance=0.5cm and 0.5cm,>=latex,auto, every place/.style={draw}]

     \node [place,thick] (s) {$s$};
      \node [place,thick] (a) [right= 5cm of s]{$a$};
	
    \node[above =0.4cm of a] at (a) {$v(a)=-1$};
    \node[above=0.4cm of s] at (s) {$v(s)=1$};
    \path[->] (s) edge node {$w(r)=1$} (a); 

\end{tikzpicture}
\caption{}
\label{fig:DCountCC}
\end{figure}

If we check coherence for statement $s$, we obtain that:
$$|v_{D(P)}(s)-e_{D(P)}(s)|=v(s)-v(a)=2>\epsilon.$$
for any $\epsilon\in(0,1)$, and hence collective coherence does not hold for this profile.




\item Endorsed Unanimity. Using the opinion profile depicted in figure \ref{fig:DCountCC}, we observe that even with full negative support on $s$ (i.e. $v(a)=-1$), the result of the aggregation is the opposite ($v_{D(P)}(s)=1$). Therefore, this opinion profile also serves as a counterexample to prove that $D$ does not satisfy Endorsed Unanimity.
\end{enumerate}
\end{proof}

\begin{prop}\label{prop:I}
The aggregation function $I$ satisfies the following properties:
\begin{enumerate}[(i)]
    \item Exhaustive Domain and Coherent Domain;
    \item Anonymity and Non-Dictatorship;
    \item Endorsed Unanimity; and
    \item Familiar Monotonicity.
\end{enumerate}
And does not satisfy:
\begin{enumerate}[(i)]
\setcounter{enumi}{4}
    \item Collective coherence;
    \item Narrow Unanimity, Sided Unanimity and Weak Unanimity;
    \item Monotonicity; and
    \item Independence.
\end{enumerate}
\end{prop}

\begin{proof}(of proposition \ref{prop:I})
\begin{enumerate}[(i)]

\item Exhaustive and Coherent domain are straighforward.

\item Anonymity and Non-Dictadorship. Let $P=(O_1,\dotsc,O_n)$ be an opinion profile over a $DRF$ and $\sigma$ a permutation over the agent in $Ag=\{1,\dotsc,n\}$. 
We must show that $I$ maintains the same collective opinion over the permuted opinion profile $P'=(O_{\sigma(1)},\dotsc,O_{\sigma(n)})$, i.e. that  $I(P)=I(P')$. 

For any $i\in\{1,\dotsc,n\}$ there is only one $j\in\{1,\dotsc,n\}$ such that $\sigma(j)=i$, and hence in terms of expectation functions we know that $e_i=e_{\sigma(j)}$. Using that, we can show that $I(P)=I(P')$ as follows:

$$v_{I(P)}(s)=\frac1n\sum_{i=1}^n e_i(s)=\frac1n\sum_{i=1}^n e_{\sigma(i)}(s)=v_{I(P')}(a);$$
$$w_{I(P)}(r)=\frac1n\sum_{i=1}^n w_i(r)=\frac1n\sum_{i=1}^n w_{\sigma(i)}(r)=w_{I(P')}(r)$$.

\item Endorsed Unanimity. Let $s$ be a sentence and $P=(O_1,\dotsc,O_n)$ an opinion profile satisfying that $v_i(s_r)=1$ for any agent $i$ and descendant $s_r\in D(s)$ of sentence $s$. Since the expectation over $s$ is:
$$e_i(s)=\frac 1 {\sum_{r\in R^+(s)}w_i(r)}\sum_{r\in R^+(s)}w_i(r) v_i(s_r)=\frac 1 {\sum_{r\in R^+(s)}w_i(r)}\sum_{r\in R^+(s)} w_i(r)=1,$$
then the aggregated value for $s$ is:
$$v_{I(P)}(s)=\frac 1 n\sum_{i} e_i(s)=1.$$
Analogously, if we assume that $v_i(s_r)=-1$ for any agent $i$ and descendant $s_r\in D(s)$ of sentence $s$, we would obtain that $v_{I(P)}(s)= -1$.
Since $v_{I(P)}(s) > 0$ when there is full positive support (and $v_{I(P)}(s) < 0$ for negative support), Endorsed Unanimity holds.

\item Familiar Monotonicity. It is straightforward to see that the output of the Indirect aggregation function, which uses an expectation function, depends only on the values on descendants and their relationships. So, it is clear that a different opinion profile maintaining the same values for descendants and their relationships will not change the output of the function.

\item Collective Coherence.  To prove that it does not hold, it suffices to find a $DRF$ and an opinion profile for which there is no collective coherence. Thus, consider the example depicted below in figure \ref{fig:ICountCC}.
Here $v_{I(P)}(s)=1$ and $v_{I(P)}(a)= -1 =v_{I(P)}(b)$. Now, if we check coherence for statement $s$, we obtain that $|v_{I(P)}(s)-e_{I(P)}(s)|=2>\epsilon$ for any $\epsilon\in(0,1)$, and hence collective coherence does not hold for this profile.

\begin{figure}[H]
\centering
\begin{tikzpicture}[node distance=0.5cm and 0.5cm,>=latex,auto, every place/.style={draw}]

     \node [place,thick] (s) {$s$};
      \node [place,thick] (a) [right= 2cm of s]{$a$};
      \node [place,thick] (b) [right= 2cm of a]{$b$};
	
    \node[above =0.4cm of a] at (a) {$v(a)=1$};
    \node[above=0.4cm of s] at (s) {$v(s)=1$};
   \node[above=0.4cm of b] at (b) {$v(b)=-1$};

    \path[->] (s) edge node {$w(r_1)=1$} (a); 
        \path[->] (a) edge node {$w(r_2)=1$} (b); 

\end{tikzpicture}
\caption{}
\label{fig:ICountCC}
\end{figure}


\item Narrow Unanimity, Sided Unanimity and Weak Unanimity. Next we build a $DRF$ and an opinion profile for which Weak Unanimity does not hold despite satisfying the assumptions. Consider the example in figure \ref{fig:ICountWU} with opinion profile $P=(O=(v,w))$. Although $v(s)=1$, $v_{I(P)}(s)=-1$ instead of greater than 0, and hence $I$ does not satisfy Weak unanimity. As discussed in section \ref{subsec:socialChoiceProperties}, an aggregation function satisfying either Narrow Unanimity or Sided Unanimity also satisfies Weak Unanimity. Thus, since Weak Unanimity does not hold, neither do Narrow Unanimity and Sided Unanimity.


\begin{figure}[H]
\centering
\begin{tikzpicture}[node distance=0.5cm and 0.5cm,>=latex,auto, every place/.style={draw}]

     \node [place,thick] (s) {$s$};
      \node [place,thick] (a) [right= 2cm of s]{$a$};
	
    \node[right=0.4cm of a] at (a) {$v(a)=-1$};
    \node[left=0.4cm of s] at (s) {$v(s)=1$};
    \path[->] (s) edge node {$w(r)=1$} (a); 
\end{tikzpicture}
\caption{}
\label{fig:ICountWU}
\end{figure}


\item Monotonicity. Next we build a $DRF$ and an opinion profile for which Monotonicity does not hold despite satisfying the assumptions. Consider the opinion profile in figures \ref{fig:ICountM1} and \ref{fig:ICountM2} for the same $DRF$. The two profiles $P=(O=(v,w))$ and $P'=(O'=(v',w'))$ only differ on the valuation of $a$: $v(a)=1$ and $v'(a)=-1$.

Clearly, $x=v(s)\leq v'(s)=x$, thus satisfying the assumptions of monotonicity. However, since the aggregated valuations on $s$ are: $v_{I(P)}(s)=1$ and $v_{I(P')}(s)=-1$, it does not satisfy that $v_{I(P)}\leq v_{I(P')}$, and hence Monotonicity does not hold.

\begin{figure}[H]
\centering
\begin{tikzpicture}[node distance=0.5cm and 0.5cm,>=latex,auto, every place/.style={draw}]

     \node [place,thick] (s) {$s$};
      \node [place,thick] (a) [right= 2cm of s]{$a$};
	
    \node[right=0.4cm of a] at (a) {$v(a)=1$};
    \node[left=0.4cm of s] at (s) {$v(s)=x$};
    \path[->] (s) edge node {$w(r)=1$} (a); 
\end{tikzpicture}

\caption{}
\label{fig:ICountM1}
\end{figure}

\begin{figure}[H]
\centering
\begin{tikzpicture}[node distance=0.5cm and 0.5cm,>=latex,auto, every place/.style={draw}]

     \node [place,thick] (s) {$s$};
      \node [place,thick] (a) [right= 2cm of s]{$a$};
	
    \node[right=0.4cm of a] at (a) {$v'(a)=-1$};
    \node[left=0.4cm of s] at (s) {$v'(s)=x$};
    \path[->] (s) edge node {$w'(r)=1$} (a); 
\end{tikzpicture}

\caption{}
\label{fig:ICountM2}
\end{figure}

\item Independence. Next we build a $DRF$ and an opinion profile for which Independence does not hold despite satisfying the assumptions. Consider the opinion profiles $P$ and $P'$ in figures \ref{fig:ICountI1} and  \ref{fig:ICountI2}. Although $v(s)=v'(s)$ for those profiles, the aggregated valuations on $s$ do not match: $1=v_{I(P)}(s)\neq v_{I(P')}(s)=0$. Therefore, Independence does not hold.

\begin{figure}[H]
\centering
\begin{tikzpicture}[node distance=0.5cm and 0.5cm,>=latex,auto, every place/.style={draw}]

     \node [place,thick] (s) {$s$};
      \node [place,thick] (a) [right= 2cm of s]{$a$};
	
    \node[right=0.4cm of a] at (a) {$v(a)=1$};
    \node[left=0.4cm of s] at (s) {$v(s)=1$};
    \path[->] (s) edge node {$w(r)=1$} (a); 
\end{tikzpicture}
\caption{}
\label{fig:ICountI1}
\end{figure}
\begin{figure}[H]
\centering
\begin{tikzpicture}[node distance=0.5cm and 0.5cm,>=latex,auto, every place/.style={draw}]

     \node [place,thick] (s) {$s$};
      \node [place,thick] (a) [right= 2cm of s]{$a$};
	
    \node[right=0.4cm of a] at (a) {$v'(a)=0$};
    \node[left=0.4cm of s] at (s) {$v'(s)=1$};
    \path[->] (s) edge node {$w(r)=1$} (a); 
\end{tikzpicture}
\caption{}
\label{fig:ICountI2}
\end{figure}

\end{enumerate}
\end{proof}

\begin{prop}\label{prop:R}
The aggregation function $R$ satisfies the following properties:
\begin{enumerate}[(i)]
    \item Collective Coherence;
    \item Exhaustive Domain and Coherent Domain; 
    \item Anonymity and Non-Dictatorship;
\end{enumerate}
And does not satisfy:
\begin{enumerate}[(i)]
\setcounter{enumi}{3}
    \item Narrow Unanimity, Sided Unanimity and Weak Unanimity;
    \item Endorsed Unanimity;
    \item Familiar Monotonicity, so neither Monotonicity;
    \item Independence.
\end{enumerate}
\end{prop}

\begin{proof}
\begin{enumerate}[(i)]
    \item Collective Coherence. Since $v_{R(P)}=e_{R(P)}$, the collective opinion for $R$ is exactly the result of applying the estimation function, and hence collective coherence follows because $|v_{R(P)}(s)-e_{R(P)}(s)| = 0 < \epsilon$ for any $\epsilon\in(0,1)$ and any sentence $s\in\sen$.
    
    \item Exhaustive Domain and Coherent Domain.
    Straightforward.
    
    \item Anonymity and Non-Dictatorship.
    Let $P=(O_1,\dotsc,O_n)$ be an opinion profile over a $DRF$ and $\sigma$ a permutation over the agents in $Ag=\{1,\dotsc,n\}$. 
    We must show that $R$ maintains the same collective opinion over the permuted opinion profile $P'=(O_{\sigma(1)},\dotsc,O_{\sigma(n)})$, i.e. that  $R(P)=R(P')$.

We consider first the sentences $s\in\sen$ with no descendants such that $R^+(s)=\emptyset$. Since these have no descendants, $R$ computes the collective opinion on them using $D$. As shown by proposition \ref{prop:D}, since $D$ satisfies anonymity, it will also hold for $R$ when considering sentences with no descendants. Thus, since these sentences, which are used at the beginning of the recursive process run by $R$, will not change through permutations, the collective opinion over any sentence will be the same after permutations. Therefore, anonymity holds for $R$, and from this Non-Dictatorship.

\item Weak, Narrow and Sided Unanimity. The example of $DRF$ depicted in figure \ref{fig:RCountWU} with opinion profile $P=(O=(v,w))$ will be enough to show that $R$ does not satisfy Weak unanimity. 
Although $v(s)=1$, and hence the assumptions for Weak Unanimity hold, $v_{R(P)}(s)=-1$ influenced by the valuation on $b$. Since $v_{R(P)}(s)$ is not positive, Weak unanimity does not hold for $R$, and consequently neither Side unanimity nor Narrow unanimity.

\begin{figure}[H]
\centering
\begin{tikzpicture}[node distance=0.5cm and 0.5cm,>=latex,auto, every place/.style={draw}]

     \node [place,thick] (s) {$s$};
      \node [place,thick] (a) [right= 2cm of s]{$a$};
      \node [place,thick] (b) [right= 2cm of a]{$b$};
	
    \node[above=0.4cm of a] at (a) {$v(a)=1$};
    \node[left=0.4cm of s] at (s) {$v(s)=1$};
    \node[right=0.4cm of b] at (b) {$v(b)=-1$};
    \path[->] (s) edge node {$w(r_1)=1$} (a); 
    \path[->] (a) edge node {$w(r_2)=1$} (b);    
\end{tikzpicture}
\caption{}
\label{fig:RCountWU}
\end{figure}

\item Endorsed Unanimity. Consider again the opinion profile depicted in  figure \ref{fig:RCountWU}. Clearly, since $v(a)=1$, $s$ has full positive support, but $v_{R(P)}(s)=-1$. Since $v_{R(P)}(s)$ is not positive, Endorsed Unanimity does not hold.

\item Familiar Monotonicity and Monotonicity. We build a $DRF$ and two opinion profiles for which Familiar Monotonicity does not hold despite satisfying the assumptions. Consider the two opinion profiles $P=(O=(v,w))$ and $P'=(O'=(v',w'))$ depicted in figures \ref{fig:RCountM1} and \ref{fig:RCountM2} respectively.
Considering $s$, these two profiles satisfy the assumptions of Familiar Monotonicity: $v(s)\leq v'(s)$ and the values on the indirect opinion are the same. However, $P$ and $P'$ differ on the value on $b$: $v(b)=1$ and $v'(b)=-1$. This leads to a change of value on the aggregated value on $s$. Thus, $v_{R(P)}(s)\not\leq v_{R(P')}(s)$, and $R$ fails at satisfying Familiar Monotonicity. By lemma \ref{lem:monotonicity}, Monotonicity does not hold either.

\begin{figure}[H]
\centering
\begin{tikzpicture}[node distance=0.5cm and 0.5cm,>=latex,auto, every place/.style={draw}]

     \node [place,thick] (s) {$s$};
      \node [place,thick] (a) [right= 2cm of s]{$a$};
      \node [place,thick] (b) [right= 2cm of a]{$b$};
	
	\node[above=0.4cm of a] at (a) {$v(a)=1$};
    \node[right=0.4cm of b] at (b) {$v(b)=1$};
    \node[left=0.4cm of s] at (s) {$v(s)=x$};
    \path[->] (s) edge node {$w(r_1)=1$} (a); 
    \path[->] (a) edge node {$w(r_2)=1$} (b); 
\end{tikzpicture}

\caption{}
\label{fig:RCountM1}
\end{figure}

\begin{figure}[H]
\centering
\begin{tikzpicture}[node distance=0.5cm and 0.5cm,>=latex,auto, every place/.style={draw}]

     \node [place,thick] (s) {$s$};
      \node [place,thick] (a) [right= 2cm of s]{$a$};
      \node [place,thick] (b) [right= 2cm of a]{$b$};
	
	\node[above=0.4cm of a] at (a) {$v'(a)=1$};
    \node[right=0.4cm of b] at (b) {$v'(b)=-1$};
    \node[left=0.4cm of s] at (s) {$v'(s)=x$};
    \path[->] (s) edge node {$w'(r_1)=1$} (a); 
    \path[->] (a) edge node {$w'(r_2)=1$} (b); 
\end{tikzpicture}

\caption{}
\label{fig:RCountM2}
\end{figure}


\item Independence. Straightforward from the example employed in proposition \ref{prop:I} to prove lack of independence.
    
\end{enumerate}
\end{proof}

Next, we provide the proofs for the analysis of the families of $\alpha$-balanced aggregation functions $\{B_\alpha\}_{\alpha\in(0,1)}$ and $\alpha$-recursive aggregation functions $\{R_\alpha\}_{\alpha\in(0,1)}$. Before that, we first introduce some general lemmas that will be useful to build the proofs of the propositions for both families.
To ease notation, these general lemmas that follow consider two generic aggregation functions $F$ and $G$, as well as a generic aggregation function $H=\alpha F + (1-\alpha) G$ instead of $v_H(P)=\alpha v_F(P)+(1-\alpha)v_G(P)$. Hereafter, the following lemmas establish the social properties fulfilled by $H$.



\begin{lemma}\label{prop:FGExh}
Let $F$ and $G$ be two opinion aggregation functions satisfying Exhaustive domain. For any $\alpha\in(0,1)$, aggregation function $H=\alpha F +(1-\alpha)G$ also satisfies Exhaustive domain. 
\end{lemma}
\begin{proof}
Straightforward from the fact that both $F$ and $G$ satisfy Exhaustive domain.
\end{proof}

\begin{lemma}\label{prop:FGAn}
Let $F$ and $G$ two opinion aggregation functions satisfying Anonymity over domain $\mathcal{D}$ . For any $\alpha\in(0,1)$, aggregation function $H=\alpha F +(1-\alpha)G$ also satisfies Anonymity over domain $\mathcal{D}$.  
\end{lemma}
\begin{proof}
For any given opinion profile $P$ and its permuted profile $P'$, if $F(P)=F(P')$ and $G(P)=G(P')$, then it follows that $H(P)=H(P')$.
\end{proof}

\begin{lemma}\label{prop:FGFamMon}
Let $F$ and $G$ two opinion aggregation functions satisfying Familiar Monotonicity over domain $\mathcal{D}$. For any $\alpha\in(0,1)$, aggregation function $H=\alpha F +(1-\alpha)G$ also satisfies Familiar Monotonicity on domain $\mathcal{D}$.  
\end{lemma}
\begin{proof}
Let $P=(O_1=(v_1,w_1),\dotsc,O_n=(v_n,w_n))$ and $P'=(O'_1=(v'_1,w'_1),\break\dotsc,O'_n=(v'_n,w'_n))$ be a profile satisfying the assumptions of familiar montonicity for a statement $s$, i.e. $v_i(s)\leq v'_i(s)$ for any $i$ and for any $r\in R^+(s)$ then $w_i(r)=w'_i(r)$ and $v_i(s_r)=v'_i(s_r)$. Since $F$ and $G$ satisfy familiar monotonicity, then $v_{F(P)}(s)\leq v_{F(P')}(s)$ and $v_{G(P)}(s)\leq v_{G(P')}(s)$. Thus, since $H=\alpha F +(1-\alpha)G$, it follows directly that $v_{H(P)}(s)\leq v_{H(P')}(s)$, so familiar montonicity holds for $H$. \end{proof}

\begin{lemma}\label{prop:FGSided}
Let F and G two opinion aggregation functions satisfying Sided unanimity on domain $\mathcal{D}$. For any $\alpha\in(0,1)$, aggregation function $H=\alpha F+ (1-\alpha)G$ also satisfies Sided unanimity on $\mathcal{D}$.
\end{lemma}
\begin{proof}
Since Sided unanimity holds for $F$ and $G$, we know that given any opinion profile $P$ of agents $\{1,\dotsc,n\}$, i.e. if for any $i \in \{1,\dotsc,n\}$  $v_i(s)>0$ then $v_F(s)>0$ and $v_G(s)>0$, and since $v_H=\alpha v_F+ (1-\alpha)v_G$, it also follows that $v_H(s)>0$. Likewise for the negative case, so Sided Unanimity holds for $H$.
\end{proof}

\begin{lemma}\label{prop:FGWeak}
Let F and G two opinion aggregation functions satisfying Weak unanimity on the domain $\mathcal{D}$. For any $\alpha\in(0,1)$, aggregation function $H=\alpha F+ (1-\alpha)G$ also satisfies Weak unanimity over domain $\mathcal{D}$.
\end{lemma}
\begin{proof}
Since Sided unanimity holds for $F$ and $G$, we know that given any opinion profile $P$ of agents $\{1,\dotsc,n\}$, for any $i \in \{1,\dotsc,n\}$, if $v_i(s) = 1$, then $v_F(s)>0$ and $v_G(s)>0$. Since $v_H=\alpha v_F+ (1-\alpha)v_G$, it also follows that $v_H(s)>0$, and hence Weak Unanimity holds for $H$. Analogously for the negative case.
\end{proof}

\begin{lemma}\label{prop:FGEnd}
Let F and G two opinion aggregation functions satisfying Endorsed unanimity on domain $\mathcal{D}$. For any $\alpha\in(0,1)$, aggregation function $H=\alpha F+ (1-\alpha)G$ also satisfies Endorsed unanimity on $\mathcal{D}$.
\end{lemma}
\begin{proof}
Since Endorsed unanimity holds for $F$ and $G$, we know that given any opinion profile $P$ of agents $\{1,\dotsc,n\}$, for any $i \in \{1,\dotsc,n\}$ and descendant $s_r\in D(s)$ of sentence $s$, if  $v_i(s_r) > 1$, then $v_F(s)>0$ and $v_G(s)>0$. Since $v_H=\alpha v_F+ (1-\alpha)v_G$, it also follows that $v_H(s)>0$, and hence Endorsed Unanimity holds for $H$. Analogously for the full negative support case.
\end{proof}

We are now ready to prove the results for $\alpha$-balanced aggregation functions in $\{B_\alpha\}_{\alpha\in(0,1)}$.

\begin{prop}\label{prop:Balpha}
The family of $\alpha$-balanced aggregation functions $\{B_\alpha\}_{\alpha\in(0,1)}$ satisfies the following properties:
\begin{enumerate}[(i)]
    \item Exhaustive Domain and Coherent Domain;
    \item Anonymity and Non-Dictatorship;
    \item Weak Unanimity for $\alpha> \frac 1 2$;
    \item Endorsed Unanimity for $\alpha< \frac 1 2$;
    \item Familiar Monotonicity;
\end{enumerate}
and does not satisfy:
\begin{enumerate}[(i)]
\setcounter{enumi}{5}
    \item Collective coherence;
    \item Narrow Unanimity, nor Sided Unanimity;
    \item Monotonicity;
    \item Independence.
\end{enumerate}
\end{prop}

\begin{proof}
\begin{enumerate}[(i)]
    \item Exhaustive Domain and Coherent Domain follow from propositions \ref{prop:D} and \ref{prop:I}, and from lemma \ref{prop:FGExh}.
    
    \item Anonymity and Non-Dictatorship follow from propositions \ref{prop:D} and \ref{prop:I}, and from lemma \ref{prop:FGAn}.

    \item Weak Unanimity. Let $P=(O_1,\dotsc,O_n)$ be an opinion profile over a $DRF$ for the agents in $Ag=\{1,\dotsc,n\}$, and $s\in\sen$ a sentence such that $v_i(s)=1$ for any $i$. By proposition \ref{prop:D}, we know that Weak unanimity holds for the Direct aggregation function, and hence  
    $v_{D(P)}(s)=\frac 1 n \sum_{i\in Ag} v_i(s)=1.$
    Now we turn our attention to $I$, the indirect function. The worst scenario occurs when $v_{I(P)}(s)=-1$ because aggregating this value to $v_{D(P)}(s)$ might prevent that $v_{B_\alpha(P)}(s)>0$, and thus that Weak unanimity holds. The DRF and an opinion profile depicted in figure \ref{fig:BalphaW} exemplifies this case. 
    
    
    \begin{figure}[H]
\centering
\begin{tikzpicture}[node distance=0.5cm and 0.5cm,>=latex,auto, every place/.style={draw}]

     \node [place,thick] (s) {$s$};
      \node [place,thick] (a) [right= 2cm of s]{$a$};
	
    \node[right=0.4cm of a] at (a) {$v(a)=-1$};
    \node[left=0.4cm of s] at (s) {$v(s)=1$};
    \path[->] (s) edge node {$w(r)=1$} (a); 
\end{tikzpicture}
\caption{}
\label{fig:BalphaW}
\end{figure}

Since $v_{D(P)}(s)=1$ and $v_{I(P)}(s)=-1$,  $v_{B_\alpha(P)}(s)=\alpha-(1-\alpha)=2\alpha - 1.$
Thus, if we set $\alpha$ so that $\alpha>\frac 1 2$, then we ensure that $v_{B_\alpha(P)}(s)>0$, and Weak unanimity holds. The proof is analogous for the negative case of Weak unanimity.


    
    \item Endorsed Unanimity. Let $s$ be a sentence and $P=(O_1,\dotsc,O_n)$ an opinion profile satisfying that $v_i(s_r)=-1$ for any agent $i$ and descendant $s_r\in D(s)$ of sentence $s$. In other words, $s$ has full negative support. It follows that $v_{I(P)}(s)=-1$. Likewise for our proof for Weak unanimity above, we consider the worst case, which would occur when $v_{D(P)}(s)=1$. Figure  \ref{fig:BalphaW} depicts a DRF and single-opinion profile illustrating this case. Since $v_{D(P)}(s)=1$ and $v_{I(P)}(s)=-1$,  $v_{B_\alpha(P)}(s) = \alpha - (1-\alpha)=2\alpha-1$. Thus, if we set $\alpha$ so that $\alpha < \frac 1 2$, then we ensure that $v_{B_\alpha(P)}(s)<0$, and Endorsed unanimity holds. The proof is analogous for the positive case (full positive support) of Endorsed unanimity.

    
    
    
    \item Familiar Monotonicity follows from propositions \ref{prop:D} and \ref{prop:I}, and from lemma \ref{prop:FGFamMon}.
    
    \item Collective coherence.To prove that it does not hold, it suffices to find a $DRF$ and an opinion profile for which there is no collective coherence. Thus, consider the DRF with one-opinion profile depicted below in figure \ref{fig:BalphaCountCC}. Computing the aggregations for the Direct and Indirect functions, we have that $v_{D(P)}(s)=1$, $v_{D(P)}(a)=0$, and, $v_{I(P)}(s)=0$ and $v_{I(P)}(a)=-1$. Therefore, $v_{B_\alpha(P)}(s)=\alpha$ and $v_{B_\alpha(P)}(a)=(-1)(1-\alpha)$. And hence, the coherence at sentence $s$ is:    $|v_{B_\alpha(P)}(s)-e_{B_\alpha(P)}(s)|=|v_{B_\alpha(P)}(s)- v_{B_\alpha(P)}(a)|=
    1.$     Thus, we conclude that, for any $\epsilon\in(0,1)$,  $\epsilon$-coherence cannot be satisfied regardless of the value of $\alpha$. Therefore, $B_\alpha$ does not satisfy $\epsilon$-coherence.

        \begin{figure}[H]
\centering
\begin{tikzpicture}[node distance=0.5cm and 0.5cm,>=latex,auto, every place/.style={draw}]

     \node [place,thick] (s) {$s$};
      \node [place,thick] (a) [right= 2cm of s]{$a$};
      \node [place,thick] (b) [right= 2cm of a]{$b$};
	
    \node[below=0.4cm of a] at (a) {$v(a)=0$};
    \node[left=0.4cm of s] at (s) {$v(s)=1$};
        \node[right=0.4cm of b] at (b) {$v(b)=-1$};

    \path[->] (s) edge node {$w(r)=1$} (a); 
    \path[->] (a) edge node {$w(r)=1$} (b); 
\end{tikzpicture}
\caption{}
\label{fig:BalphaCountCC}
\end{figure}

    
    
    
    \item Sided Unanimity, Narrow Unanimity. To prove that neither of these properties holds, it suffices to find a $DRF$ and an opinion profile for which there is no Sided unanimity. 
    In particular, we will show that for any $\alpha\in(0,1)$ we can find a DRF and an opinion profile for which Sided unanimity and Narrow Unanimity do not hold. 
    Consider then the DRF with single-opinion profile in figure \ref{fig:BalphaCountSU}, where $x\in(0,1)$ is such that $0<x<\frac {1-\alpha} {\alpha}$. Since $v(s) = x > 0$, the assumptions for Sided unanimity hold at sentence $s$. Now, since $v_{D(P)}(s)=x$ and $v_{I(P)}(s)=-1$, it follows that $v_{B_\alpha(P)}(s)=\alpha x +(1-\alpha)(-1) = \alpha x +\alpha - 1 < \alpha \frac{1-\alpha}{\alpha} +\alpha - 1 = 0$. Since $v_{B_\alpha(P)}\not>0$, Sided unanimity fails at $s$, and as it is single-opinion profile Narrow unanimity fails too. The proof goes analogously for the negative case of Sided unanimity.

    
    \begin{figure}[H]
\centering
\begin{tikzpicture}[node distance=0.5cm and 0.5cm,>=latex,auto, every place/.style={draw}]

     \node [place,thick] (s) {$s$};
      \node [place,thick] (a) [right= 2cm of s]{$a$};
	
    \node[right=0.4cm of a] at (a) {$v(a)=-1$};
    \node[left=0.4cm of s] at (s) {$v(s)=x$};
    \path[->] (s) edge node {$w(r)=1$} (a); 
\end{tikzpicture}
\caption{}
\label{fig:BalphaCountSU}
\end{figure}



    \item Monotonicity. It suffices to find a $DRF$ and an opinion profile for which there is no Monotonicity. Consider the two single-opinion profiles $P$ and $P'$ over the very same DRF in figures \ref{fig:BalphaCountM1}  and \ref{fig:BalphaCountM2}. 
    We will check Monotonicity at sentence $s$, where the conditions for unanimity hold because $v(s) \leq v'(s)$.
    Computing $B_\alpha$ we obtain that $v_{B_\alpha(P)}(s)=1$ and $v_{B_\alpha(P')}(s)=2\alpha - 1$. To fulfil Monotonicity both expressions must satisfy that $v_{B_\alpha(P)}(s)\leq v_{B_\alpha(P')}(s)$, namely that $1\leq 2\alpha-1$. This is only possible for $\alpha\geq 1$. Therefore, for any $\alpha\in(0,1)$ Monotonicity does not hold.
    

    \begin{figure}[H]
    \centering
\begin{tikzpicture}[node distance=0.5cm and 0.5cm,>=latex,auto, every place/.style={draw}]

     \node [place,thick] (s) {$s$};
      \node [place,thick] (a) [right= 2cm of s]{$a$};
	
    \node[right=0.4cm of a] at (a) {$v(a)=1$};
    \node[left=0.4cm of s] at (s) {$v(s)=1$};
    \path[->] (s) edge node {$w(r)=1$} (a); 
\end{tikzpicture}

\caption{}
\label{fig:BalphaCountM1}
\end{figure}

\begin{figure}[H]
    \centering
\begin{tikzpicture}[node distance=0.5cm and 0.5cm,>=latex,auto, every place/.style={draw}]

     \node [place,thick] (s) {$s$};
      \node [place,thick] (a) [right= 2cm of s]{$a$};
	
    \node[right=0.4cm of a] at (a) {$v'(a)=-1$};
    \node[left=0.4cm of s] at (s) {$v'(s)=1$};
    \path[->] (s) edge node {$w'(r)=1$} (a); 
\end{tikzpicture}
\caption{}
\label{fig:BalphaCountM2}
\end{figure}

$$1\leq 2\alpha-1$$ 

    \item Independence. For any $\alpha\neq 1$, $B_\alpha$ does not fulfil Independence due to its dependence on $I$.
    
\end{enumerate}
\end{proof}

\begin{prop}
The family of $\alpha$-recursive aggregation functions  $\{R_\alpha\}_{\alpha\in(0,1)}$ satisfies the following properties:
\begin{enumerate}[(i)]
    \item Collective Coherence for $\alpha< \frac \epsilon 2$;
    \item Exhaustive Domain and Coherent Domain;
    \item Anonymity and Non-Dictatorship;
    \item Weak Unanimity for $\alpha>\frac 1 2$;
\end{enumerate}
and does not satisfy:
\begin{enumerate}[(i)]
\setcounter{enumi}{4}
    \item Sided unanimity, so neither Narrow Unanimity;
    \item Endorsed Unanimity;
    \item Familiar Monotonicity, so neither Monotonicity;
    \item Independence.
\end{enumerate}
\end{prop}

\begin{proof}
\begin{enumerate}[(i)]
    \item Collective Coherence. Given $\epsilon>0$ and a DRF, we must prove that $|v_{R_\alpha(P)}(s)-e_{R_\alpha(P)}(s)|<\epsilon$ for any sentence $s\in\sen$. First, we develop the difference between valuation and estimation for the collective opinion:
    
    
    
    \begin{align*}
    &v_{R_\alpha(P)}(s)-e_{R_\alpha(P)}(s)=v_{R_\alpha(P)}(s)-\frac {\sum_{r\in R^+(s)} w_{R_\alpha(P)}(r)v_{R_\alpha(P)}(s_r)} {\sum_{r\in R^+(s)} w_{R_\alpha(P)}(r)}\\
    &=[\alpha v_{D(P)}(s) + (1-\alpha)v_{R(P)}(s)]-
     \frac{\sum_{r\in R^+(s)} w_{D(P)}(r)\big[\alpha v_{D(P)}(s_r) + (1-\alpha)v_{R(P)}(s_r)\big]}{\sum_{r\in R^+(s)}w_{R_\alpha(P)}(r)}\\
    &=\alpha\big[ v_{D(P)}(s) - \frac{\sum_{r\in R^+(s)} w_{R_\alpha(P)}(r) v_{D(P)}(s_r)}{\sum_{r\in R^+(s)}w_{D(P)}(r)}\big]+\\
    &+ (1-\alpha)\big[v_{R(P)}(s)-\frac{\sum_{r\in R^+(s)} w_{R(P)}(r) v_{R(P)}(s_r)}{\sum_{r\in R^+(s)}w_{R(P)}(r)}\big]\\
    &=\alpha\big[v_{D(P)}(s)-e_{D(P)}(s)\big]+(1-\alpha)\big[v_{R(P)}(s)-e_{R(P)}(s)\big]\\
    &=\alpha(v_{D(P)}(s)-e_{D(P)}(s)).\\
    \end{align*}
    
    Notice that we get rid of $v_{R(P)}(s)-e_{R(P)}(s)$ because is zero. Now the coherence of  $R_\alpha$ directly depends on the coherence of the direct aggregation function $D$ and $\alpha$. Figure \ref{fig:D_WorstExample} depicts a DRF with a single-opinion profile representing a worst-case scenario for $D$ because $v_{D(P)}(s)-e_{D(P)}(s)=2$. Considering the coherence of $R_\alpha$, we have that $|v_{R_\alpha(P)}(s)-e_{R_\alpha(P)}(s)|=\alpha|v_{D(P)}(s)-e_{D(P)}(s)|\leq 2\alpha$ for any profile $P$. Therefore, we must ensure that $\alpha<\frac \epsilon 2$  so that $|v_{R_\alpha(P)}(s)-e_{R_\alpha(P)}(s)|<\epsilon$ holds for any profile of the domain, and hence Collective coherence holds for $R_\alpha$.
    

\begin{figure}[H]
\centering
\begin{tikzpicture}[node distance=0.5cm and 0.5cm,>=latex,auto, every place/.style={draw}]

     \node [place,thick] (s) {$s$};
     \node [place,thick] (a) [right= 2cm of s]{$a$};
	
    \node[left=0.4cm of s] at (s) {$v(s)=1$};
    \node[right=0.4cm of a] at (a) {$v(a)=-1$};
    \path[->] (s) edge node {$w(r)=1$} (a); 

\end{tikzpicture}
\caption{}
\label{fig:D_WorstExample}
\end{figure}

    
    
    \item Exhaustive Domain and Coherent Domain directly follow from propositions \ref{prop:D} and \ref{prop:R} and lemma \ref{prop:FGExh}.
    \item Anonymity and Non-Dictatorship follow directly from propositions \ref{prop:D} and \ref{prop:R} and lemma \ref{prop:FGAn}.
    
    \item Weak unanimity. To prove Weak unanimity we can resort to the proof built to prove Weak unanimity for $B_\alpha$ in proposition \ref{prop:Balpha}. We simply have to substitute $B_\alpha$ for $R_\alpha$.
    
    \item Sided Unanimity and Narrow Unanimity. To prove that neither of these properties holds, it suffices to find a $DRF$ and an opinion profile for which there is no Sided unanimity (nor Narrow unanimity). 
    Consider the DRF and the single-opinion profile depicted in figure \ref{fig:RalphaCountSU}, where $x\in(0,1)$ is such that $0<x<\frac {1-\alpha} {\alpha}$.
    Since $x>0$, the assumption for Sided Unanimity holds at $s$. However, $v_{R_\alpha(P)}(s)$ is not positive, since 
 $v_{R_\alpha(P)}(s)=x\alpha-1+\alpha<\frac {1-\alpha} {\alpha}\alpha-1+\alpha=0$, and hence Sided (and Narrow) unanimity does not hold. 
    
    \begin{figure}[H]
\centering
\begin{tikzpicture}[node distance=0.5cm and 0.5cm,>=latex,auto, every place/.style={draw}]

     \node [place,thick] (s) {$s$};
     \node [place,thick] (a) [right= 2cm of s]{$a$};
	
    \node[left=0.4cm of s] at (s) {$v(s)=x$};
    \node[right=0.4cm of a] at (a) {$v(a)=-1$};
    \path[->] (s) edge node {$w(r)=1$} (a); 

\end{tikzpicture}
\caption{}
\label{fig:RalphaCountSU}
\end{figure}

    \item Endorsed Unanimity. Next we build a DRF and an opinion profile for which Endorsed Unanimity does not hold. 
    Consider the DRF and the opinion profile $P$ depicted in figure \ref{fig:CountFMRalpha1}. The assumptions for Endorsed unanimity hold at $s$ because $s$ has full negative support.
    However, $v_{R_\alpha(P)}(s)$ is not negative: since $v_{D(P)}(s)=1$ and $v_{R(P)}(s)=1$, we obtain that $v_{R_\alpha(P)}(s)=1$ for any $\alpha\in(0,1)$. Therefore, Endorsed Unanimity does not hold.
    
    \begin{figure}[H]
\centering
\begin{tikzpicture}[node distance=0.5cm and 0.5cm,>=latex,auto, every place/.style={draw}]

     \node [place,thick] (s) {$s$};
      \node [place,thick] (a) [right= 2cm of s]{$a$};
      \node [place,thick] (b) [right= 2cm of a]{$b$};
	
    \node[below=0.4cm of a] at (a) {$v(a)=-1$};
    \node[left=0.4cm of s] at (s) {$v(s)=1$};
        \node[right=0.4cm of b] at (b) {$v(b)=1$};

    \path[->] (s) edge node {$w(r_1)=1$} (a); 
    \path[->] (a) edge node {$w(r_2)=1$} (b); 
\end{tikzpicture}
\caption{}
\label{fig:CountFMRalpha1}
\end{figure}


    \item Familiar Monotonicity and Monotonicity. We build a $DRF$ and an opinion profile for which Familiar Monotonicity does not hold despite satisfying the assumptions.
    Consider the DRF and single-opinion profile $P$ in figure \ref{fig:CountFMRalpha1} together with another single-opinion profile $P'$ in figure \ref{fig:CountFMRalpha2}. 
    Clearly, the assumptions of Familiar Monotonicity are fulfilled at $s$ because $v_i(s)\leq v'_i(s)$ and the descendant of $s$ has the same value. However, we will show that $v_{R_\alpha(P)}(s) \leq v_{R_\alpha(P')}(s)$ is not true.
    For both profiles we have that $v_{D(P)}(s)=v_{D(P')}(s)=1$, and, $v_{R(P)}(s)=1$ and $v_{R(P')}(s)=1-x$. Thus, for any $\alpha\in(0,1)$: $v_{R_\alpha(P)}(s)=\alpha+(1-\alpha)=1$ and $v_{R_\alpha(P')}(s)=\alpha+(1-\alpha)(1-x)=1-x(1-\alpha) <1$ for any $x\in(0,1)$. Therefore, $v_{R_\alpha(P)}(s) > v_{R_\alpha(P')}(s)$ and Familiar Monotonicity is not satisfied. By lemma \ref{lem:monotonicity}, Monotonicity does not hold either.

 \begin{figure}[H]
\centering
\begin{tikzpicture}[node distance=0.5cm and 0.5cm,>=latex,auto, every place/.style={draw}]

     \node [place,thick] (s) {$s$};
      \node [place,thick] (a) [right= 2cm of s]{$a$};
      \node [place,thick] (b) [right= 2cm of a]{$b$};
	
    \node[below=0.4cm of a] at (a) {$v'(a)=-1$};
    \node[left=0.4cm of s] at (s) {$v'(s)=1$};
        \node[right=0.4cm of b] at (b) {$v'(b)=1-x<1$};

    \path[->] (s) edge node {$w'(r_1)=1$} (a); 
    \path[->] (a) edge node {$w'(r_2)=1$} (b); 
\end{tikzpicture}
\caption{}
\label{fig:CountFMRalpha2}
\end{figure}



    \item Independence. Clearly, $R_\alpha$ will not fulfil Independence for any $\alpha\neq 1$ because of its dependence on $R$.
\end{enumerate}
\end{proof}

\subsection{Constrained opinion profiles: assuming consensus on acceptance degrees}\label{Proofs:Uniform}

This section relates to section \ref{subsec:consensusResults}, where we assume that opinion profiles share consensus on their acceptance degrees on relationships, i.e. for each relationship $r\in \rel$ of a DRF all the agents agree on their acceptance degrees: $w_i(r)=w_j(r)$ $\forall i,j\in Ag$.

In previous section \ref{Proofs:general}, each proof and counterexample used to demonstrate that an aggregation function does or does not satisfy a property uses opinion profiles composed by one single agent. Thus, those proofs serve as well in this section when assuming consensus on acceptance degrees. For this reason, adding this assumption does not change any of the properties fulfilled by the aggregation functions in the general case (table \ref{tableGeneralResults}), and therefore there are no further desirable properties gained in this scenario with respect to the more general scenario thoroughly analysed in section \ref{Proofs:general}.

\subsection{Constrained opinion profiles: assuming coherent profiles}\label{Proofs:Coherence}

This section corresponds to the results displayed in table \ref{tableCoherentResults} in section \ref{sec:coherentResults}. We prove the results regarding the social choice properties satisfied by the aggregation functions introduced in Section \ref{sec:OpAggFunc} when assuming the domain of the aggregation functions to be $\epsilon$-coherent for some $\epsilon\in(0,1)$. This means that we consider that our aggregation functions take in coherent opinion profiles.


Since in the previous section many properties have been proven for the general case, we will not need to prove them again for this more restrictive scenario. For each opinion aggregation function, we will prove only those results regarding social choice properties that change by the addition of the coherence assumption and disprove again, this time for coherent domains, those properties which are yet not satisfied.

\begin{prop}\label{prop:DCoh} 
For any $\epsilon\in(0,1)$, $D$ over an $\epsilon$-coherent domain satisfies Endorsed Unanimity.
\end{prop}

\begin{proof}
Let $s$ a statement in a DRF and $R^+(s)$ the set of relationships $r$ from $s$ to its descendants $s_r$. Let $P$ be an $\epsilon$-coherent profile for $\epsilon\in(0,1)$ with full positive support on $s$, i.e. $v_i(s_r)=1$ for any $i$ and descendant $s_r\in D(s)$. Then:
$$e_i(s)=\frac 1 {\sum_{r\in R^+(s)} w_i(r)}\sum_{r\in R^+(s)} v_i(s_r)w_i(r)=\frac 1 {\sum_{r\in R^+(s)} w_i(r)}\sum_{r\in R^+(s)} w_i(r)=1$$
By the $\epsilon$-coherence of $P$ we have that:
$$|v_i(s)-e_i(s)|<\epsilon\Longrightarrow v_i(s)> e_i(s)-\epsilon=1-\epsilon.$$
Therefore, for any $\epsilon\in(0,1)$ we can ensure that $v_i(s)>0$ for any $i$ and the conditions for Sided Unanimity hold. Now, since $D$ satisfies Sided Unanimity (by proposition \ref{prop:D}), we obtain that $v_D(s)>0$, and hence $D$ fulfils Endorsed Unanimity. 
\end{proof}

\begin{prop}\label{prop:DCohCountCC}
$D$ over a $\delta$-coherent domain, where $\delta\in(0,1)$, still does not satisfy $\epsilon$-Collective coherence for any $\epsilon\in(0,1)$.
\end{prop}

\begin{proof} 
Consider the DRF and $\delta$-coherent opinion profile $P$ depicted in figure \ref{fig:DCohCountCC} and any $\delta\in(0,1)$. We will show that the collective opinion yield by the direct function for this example is never $\epsilon$-coherent for any $\epsilon\in(0,1)$.

Clearly, this profile is $\delta$-coherent for any $\delta>0$. Computing the direct function at $s$ we obtain that:  $v_{D(P)}(s)=-1$, $v_{D(P)}(a)=0$ and $w_{D(P)}(r)=\frac 1 2$. Now, if we check collective coherence at $s$, we see that: $|v_{D(P)}(s)-e_{D(P)}(s)|=|-1-0|=1> \epsilon$. Thus, since $1$ is larger than any $\epsilon$ value that we take in $(0,1)$, $D$ does not satisfy $\epsilon$-Collective coherence.


\begin{figure}[H]
    \centering
    \includegraphics[scale=1.2]{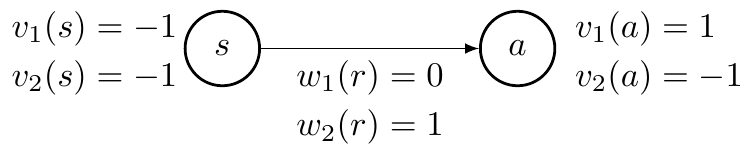}
    \caption{}
    \label{fig:DCohCountCC}
\end{figure}




\comment{We impose the failure of the coherence condition:

$$|v_{D(P)}(s)-e_{D(P)}(s)|\geq \epsilon \Longleftrightarrow \frac{1+x} 2 \geq\epsilon \Longleftrightarrow  x\geq 2\epsilon-1$$

and the initial condition  $-\delta<x<\delta$ to ensure the $\delta$-coherence of $P$. Let's check that conditions for the following set 
$\{x\ |\ -\delta<x<\delta,\ x\geq 2\epsilon-1 \}$ to be not empty, and therefore to exist counterexamples for the $\epsilon$-Collective coherence choosing $x$ from this set. 

If $\delta\leq 2\epsilon -1$ then there is no $x$ satisfying both conditions. That is not possible if $\epsilon\leq \frac 1 2$, so $2\epsilon-1\leq 0$ not allowing any $\delta<2\epsilon-1$. Then, for any $\delta>0$, if we choose  $x$ such that $2\epsilon-1\leq x<\delta$ we produce a counterexample of CC for any $0<\epsilon\leq\frac 1 2$ and any $\delta>0$. 
Now, we check the case where $\epsilon > \frac 1 2$. }

\end{proof}

\begin{prop}\label{prop:ICoh}
For $\epsilon\in(0,1)$, $I$ over an $\epsilon$-coherent domain satisfies Weak Unanimity.
\end{prop}

\begin{proof}
Consider a DRF with a statement $s\in\sen$ and $P=(O_1=(v_1,w_1),\dotsc,\break O_n=(v_n,w_n))$ an opinion profile such that $v_i(s)=1$ for every $i$. Hence, the conditions for Weak unanimity hold. If the profile $P$ is $\epsilon$-coherent, where $\epsilon\in(0,1)$, then we can conclude that for any $i$: $1-\epsilon<e_i(s)<1+\epsilon$, being $1-\epsilon>0$ for any $\epsilon\in(0,1)$. Now, computing $v_I$ at $s$ we get:
$$
    v_{I(P)}(s)=\frac 1 n \sum_i e_i(s)>\frac 1 n \sum_i 1-\epsilon> 0
$$
Since $v_{I(P)}(s) > 0$, Weak unanimity holds. The proof for the negative case of Weak unanimity is analogous.




\end{proof}

\begin{prop}\label{prop:ICohCountCC}
For any $\delta\in(0,1)$, $I$ over an $\delta$-coherent domain still does not satisfy $\epsilon$-Collective coherence for any $\epsilon\in(0,1)$.
\end{prop}
\begin{proof}
To prove that this property does not hold, it suffices to find a $DRF$ and an opinion profile for which there is no $\epsilon$-Collective coherence. 
Consider the DRF and opinion profile $P$ in figure \ref{fig:ICohCountCC}. Clearly, opinions $O_1$ and $O_2$ of $P$ are $\delta$-coherent for any $\delta>0$. Now, we compute the indirect function for all the statement: $v_{I(P)}(s)=\frac {-1} 2$, $v_{I(P)}(a)=\frac 1 2,v_{I(P)}(b)=0$, and, $w_{I(P)}(r_1)=\frac 1 2 =w_{I(P)}(r_2)$. If we check coherence at $s$ we see that:
$$|v_{I(P)}(s)-e_{I(P)}(s)|=|v_{I(P)}(s)-v_{I(P)}(a)|=|\frac {-1} 2- \frac 1 2|=1> \epsilon.$$
Thus, since $1$ is larger that any $\epsilon$ value that we take in $(0,1)$, $I$ does not satisfy $\epsilon$-Collective coherence.


    \begin{figure}[H]
    \centering
    \includegraphics[scale=1.2]{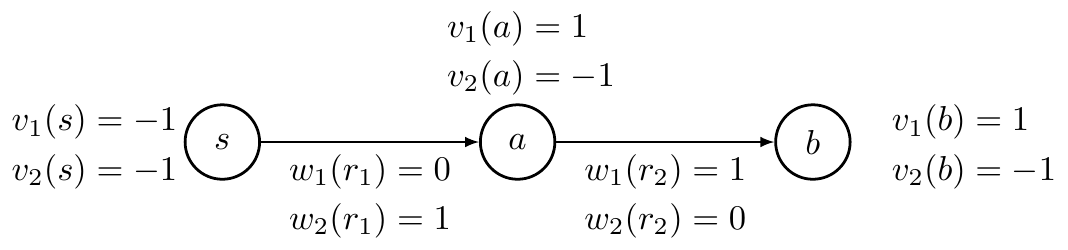}
    \caption{}
    \label{fig:ICohCountCC}
\end{figure}



\end{proof}

\begin{prop}\label{prop:ICohCount}
For any $\epsilon\in(0,1)$, $I$ over an $\epsilon$-coherent domain still does not satisfy:
\begin{enumerate}[(i)]
    \item Sided Unanimity, and therefore Narrow Unanimity;
    \item Monotonicity.
\end{enumerate}
\end{prop}

\begin{proof}

\begin{enumerate}[(i)]
     
    \item Sided unanimity and narrow unanimity. To prove that these properties not hold, it suffices to find a $DRF$ and an opinion profile for which there is no Sided unanimity. 
    Consider the DRF and one-opinion profile depicted in figure \ref{fig:ICohCountSU} such that $\epsilon\in(0,1)$ and $x,y$ such that $0<x<y<\epsilon$. The assumptions of Sided unanimity are fulfilled at $s$. We check that the opinion in the profile is $\epsilon$-coherent because 
$|v(s)-e(s)|=|x-y+\epsilon|<\epsilon$. However, $v_{I(P)}(s)=y-\epsilon<0$, instead of positive, and hence Sided Unanimity is not satisfied. As in previous proofs, as the counterexample is a single-opinion, Narrow unanimity does not hold either.
    
    
    
    \begin{figure}[H]
\centering
\begin{tikzpicture}[node distance=0.5cm and 0.5cm,>=latex,auto, every place/.style={draw}]

     \node [place,thick] (s) {$s$};
      \node [place,thick] (a) [right= 2cm of s]{$a$};
	
    \node[right=0.4cm of a] at (a) {$v(a)=y-\epsilon$};
    \node[left=0.4cm of s] at (s) {$v(s)=x$};
    \path[->] (s) edge node {$w(r)=1$} (a); 
\end{tikzpicture}
\caption{}
\label{fig:ICohCountSU}
\end{figure}


\item Monotonicity.
To prove that it does not hold, it suffices to find a $DRF$ with opinion profiles satisfying the Monotonicity assumptions are satisfied, and yet Monotonicity is not. In fact, we will create a generic counterexample for any $\epsilon$-coherent domain.
Consider the $DRF=\langle \sen,\rel,\tau\rangle$ depicted in figure \ref{fig:ICohCountM1} with $\sen=\{s,a\}$ and $\rel =\{r\}$. Also in the figure,
let $P=((v,w)))$ be an opinion profile of one single agent such that $v(s)=x$, $w(r)=1$, and $v(a)=y>$, where $-1<y<x<1$ and $0<x-y<\epsilon$.



\begin{figure}[H]
\centering
\begin{tikzpicture}[node distance=0.5cm and 0.5cm,>=latex,auto, every place/.style={draw}]

     \node [place,thick] (s) {$s$};
      \node [place,thick] (a) [right= 2cm of s]{$a$};
	
    \node[right=0.4cm of a] at (a) {$v(a)=y$};
    \node[left=0.4cm of s] at (s) {$v(s)=x$};
    \path[->] (s) edge node {$w(r)=1$} (a); 
\end{tikzpicture}

\caption{}
\label{fig:ICohCountM1}
\end{figure}

Since $v(s)-e_O(s)=v(s)-v(a)$, then the profile $P$ is clearly $\epsilon$-coherent, and hence $P\in\mathbb{C}_\epsilon(DRF)$. We compute the collective opinion using $I$ at $s$ as: $v_{I(P)}(s)=e_O(s)=v(a)=y$.

Now, consider another profile $P'=(O=(v',w))$ over the same DRF, shown in figure \ref{fig:ICohCountM2}, such that $v'(s)=x+\frac \omega 3$ and $v'(s)=y-\frac \omega 3$, where   $\omega>0$, such that $x-y+\omega\leq\epsilon$, $x+\frac \omega 3 \leq 1$ and $y-\frac \omega 3 \geq -1$.Clearly $v'(s)>v(s)$ and $P$ is also $\epsilon$-coherent, i.e.:
$$v(s)-e_{O'}(s)=(x+\frac \omega 3)-(y-\frac \omega 3)<\epsilon.$$
Nonetheless, $v_{I(P')}(s)=e_{O'}(s)= y-\frac \omega 3<y$, which means that $v_{I(P)}(s)\not\leq v_{I(P')}(s)$, and hence this example cannot satisfy Monotonicity for any  $\epsilon$.

\begin{figure}[H]
\centering
\begin{tikzpicture}[node distance=0.5cm and 0.5cm,>=latex,auto, every place/.style={draw}]

     \node [place,thick] (s) {$s$};
      \node [place,thick] (a) [right= 2cm of s]{$a$};
	
    \node[right=0.4cm of a] at (a) {$v'(a)=y-\frac \omega 3$};
    \node[left=0.4cm of s] at (s) {$v'(s)=x+\frac \omega 3$};
    \path[->] (s) edge node {$w(r)=1$} (a); 
\end{tikzpicture}
\caption{}
\label{fig:ICohCountM2}
\end{figure}


\end{enumerate}
\end{proof}

\begin{prop}\label{prop:RCoh}
For any $\epsilon\in(0,1)$ and considering the domain to be $\epsilon$-coherent, $R$ does not fulfil the following properties:
\begin{enumerate}[(i)]
    \item Weak Unanimity, neither Sided nor Narrow Unanimity;
    \item Endorsed Unanimity;
    \item Familiar Monotonicity, and therefore Monotonicity.
\end{enumerate}
\end{prop}
\begin{proof}
\begin{enumerate}[(i)]
 \item Weak Unanimity, Sided unanimity, Narrow unanimity. To prove that neither of these properties hold, it suffices to build a DRF and opinion profile to show that Weak Unanimity does not hold. This is sufficient because Weak unanimity is a weaker case of Sided and Narrow unanimity, proposition \ref{prop:unanimity} tells us that Sided and Narrow unanimity will not hold if Weak Unanimity does not. 
 Consider the DRF and opinion profile $P=((v,w))$ in picture \ref{fig:RCohCountWU} such that $w(r)=1$ for any relationship $r\in \rel$, $\epsilon\in(0,1)$ and $\delta\in (0,\epsilon)$ and $m\in\N$ so that $m\delta\geq 1>(m-1)\delta$.


\begin{figure}[H]
\centering
\begin{tikzpicture}[node distance=0.5cm and 0.5cm,>=latex,auto, every place/.style={draw}]

     \node [place,thick] (s) {$s$};
      \node [place,thick] (a1) [right= 2cm of s]{$a_1$};
      \node [thick] (dots) [right= 1cm of a1]{$\cdots$};
      \node [place,thick] (am1) [right= 1cm of dots]{$a_{m-1}$};
      \node [place,thick] (am) [right= 2cm of am1]{$a_m$};
	
    \node[below=0.4cm of a1] at (a1) {$v(a_1)=1-\delta$};
    \node[above=0.4cm of am1] at (am1) {$v(a_{m-1})=1-(m-1)\delta$};
     \node[below=0.4cm of am] at (am) {$v(a_m)=1-m\delta$};
    \node[above=0.4cm of s] at (s) {$v(s)=1$};
    \path[->] (s) edge node {} (a1); 
    \path[->] (a1) edge node {} (dots); 
    \path[->] (dots) edge node {} (am1); 
    \path[->] (am1) edge node {} (am); 

\end{tikzpicture}
\caption{}
\label{fig:RCohCountWU}
\end{figure}

Clearly, the outcome of the recursive function at each sentence is obtained from the value of the recursive function at the previous sentence, i.e.:
$$v_{R(P)}(a_m)=v_{R(P)}(a_{m-1})=...=v_{R(P)}(a_1)=v_{R(P)}(s),$$ which actually is the value $v(a_m)=1-m\delta\leq0$. So, this is an opinion profile $\epsilon$-coherent fulfilling the assumptions of Weak unanimity at sentence $s$ because $v(s)=1$. However, the value of the recursive function at $s$ is negative. Therefore, $R$ does not fulfil Weak unanimity.

\item Endorsed Unanimity. We build a DRF and opinion profile to show that Endorsed unanimity does not hold from the example in the previous proof. Figure \ref{fig:RCohCountEU} shows our example, which extends the one in figure \ref{fig:RCohCountWU} with an additional sentence $a$. Since $v(a_i)-v(a_{i-1}=\delta$, likewise in the proof above, we have an $\epsilon$-coherent opinion profile. Since $v(s)=1$ the assumption for Endorsed unanimity at $a$ is satisfied, but since $v(a)=1-m\delta\leq0$, Endorsed unanimity does not hold.
    

\begin{figure}[H]
\centering
\begin{tikzpicture}[node distance=0.5cm and 0.5cm,>=latex,auto, every place/.style={draw}]

     \node [place,thick] (s) {$s$};
     \node[place,thick] (a) [left=2cm of s] {$a$};
      \node [place,thick] (a1) [right= 2cm of s]{$a_1$};
      \node [thick] (dots) [right= 1cm of a1]{$\cdots$};
      \node [place,thick] (am1) [right= 1cm of dots]{$a_{m-1}$};
      \node [place,thick] (am) [right= 2cm of am1]{$a_m$};
	
    \node[below=0.4cm of a1] at (a1) {$v(a_1)=1-\delta$};
    \node[above=0.4cm of am1] at (am1) {$v(a_{m-1})=1-(m-1)\delta$};
     \node[below=0.4cm of am] at (am) {$v(a_m)=1-m\delta$};
    \node[above=0.4cm of s] at (s) {$v(s)=1$};
     \node[below=0.4cm of s] at (a) {$v(a)=x$};
    \path[->] (s) edge node {} (a1); 
    \path[->] (a) edge node {} (s); 
    \path[->] (a1) edge node {} (dots); 
    \path[->] (dots) edge node {} (am1); 
    \path[->] (am1) edge node {} (am); 
\end{tikzpicture}
\caption{}
\label{fig:RCohCountEU}
\end{figure}


\item Familiar Monotonicity and Monotonicity. Consider the opinion profiles $P$ and $P'$ over the same DRF depicted in figures \ref{fig:RCohCountFM1} and \ref{fig:RCohCountFM2} respectively.
Since $v(s)=v(a)=v(b) = 1$, $P$ is $\epsilon$-coherent. By setting $0<x<\epsilon$, we also obtain that $P'$ is $\epsilon$-coherent. Therefore, both $P$
and $P'$ are $\epsilon$-coherent and the assumptions for familiar monotonicity hold at $s$. However, since $1=v_{R(P)}(s)>v_{R(P')}(s)=1-x$, Familiar monotonicity cannot hold. By lemma \ref{lem:monotonicity} Monotonicity does not hold either.

\begin{figure}[H]
\centering
\begin{tikzpicture}[node distance=0.5cm and 0.5cm,>=latex,auto, every place/.style={draw}]

     \node [place,thick] (s) {$s$};
      \node [place,thick] (a) [right= 2cm of s]{$a$};
       \node [place,thick] (b) [right= 2cm of a]{$b$};
	
    \node[above=0.4cm of a] at (a) {$v(a)=1$};
    \node[left=0.4cm of s] at (s) {$v(s)=1$};
        \node[right=0.4cm of b] at (b) {$v(b)=1$};
    \path[->] (s) edge node {$w(r_1)=1$} (a); 
    \path[->] (a) edge node {$w(r_2)=1$} (b); 
\end{tikzpicture}
\caption{}
\label{fig:RCohCountFM1}
\end{figure}
\begin{figure}[H]
\centering
\begin{tikzpicture}[node distance=0.5cm and 0.5cm,>=latex,auto, every place/.style={draw}]

     \node [place,thick] (s) {$s$};
      \node [place,thick] (a) [right= 2cm of s]{$a$};
       \node [place,thick] (b) [right= 2cm of a]{$b$};
	
    \node[above=0.4cm of a] at (a) {$v'(a)=1$};
    \node[left=0.4cm of s] at (s) {$v'(s)=1$};
        \node[right=0.4cm of b] at (b) {$v'(b)=1-x$};
    \path[->] (s) edge node {$w'(r_1)=1$} (a); 
    \path[->] (a) edge node {$w'(r_2)=1$} (b); 
\end{tikzpicture}
\caption{}
\label{fig:RCohCountFM2}
\end{figure}


   \comment{ Let's consider the opinion profiles $P$ and $P'$, depicted in figures \ref{fig:RCountFM1} and \ref{fig:RCountFM2} respectively, taking $0<\delta,\gamma<1$, $x<1$ and $z>-1$.
    
    As can be seen, these two profiles meet the assumptions for familiar monotonicity, i.e., $v'(s)\geq v(s)$ and for the descendant and its relationships the values hold $v(a)=v'(a)$ and $w(r_1)=w'(r_1)$.
    
    \begin{figure}[H]
\centering
\begin{tikzpicture}[node distance=0.5cm and 0.5cm,>=latex,auto, every place/.style={draw}]

     \node [place,thick] (s) {$s$};
      \node [place,thick] (a) [right= 2cm of s]{$a$};
      \node [place,thick] (b) [right= 2cm of a]{$b$};
	
    \node[right=0.4cm of a] at (b) {$v(b)=z$};
    \node[left=0.4cm of s] at (s) {$v(s)=x$};
    \node[below=0.4cm of s] at (a) {$v(a)=y$};
    \path[->] (s) edge node {$w(r_1)=1$} (a); 
    \path[->] (a) edge node {$w(r_2)=1$} (b); 
\end{tikzpicture}
\caption{}
\label{fig:RCountFM1}
\end{figure}
\begin{figure}[H]
\centering
\begin{tikzpicture}[node distance=0.5cm and 0.5cm,>=latex,auto, every place/.style={draw}]

      \node [place,thick] (s) {$s$};
      \node [place,thick] (a) [right= 2cm of s]{$a$};
      \node [place,thick] (b) [right= 2cm of a]{$b$};
	
    \node[right=0.4cm of a] at (b) {$v'(b)=z-\gamma$};
    \node[left=0.4cm of s] at (s) {$v'(s)=x+\delta$};
    \node[below=0.4cm of s] at (a) {$v'(a)=y$};
    \path[->] (s) edge node {$w'(r_1)=1$} (a); 
    \path[->] (a) edge node {$w'(r_2)=1$} (b); 
\end{tikzpicture}
\caption{}
\label{fig:RCountFM2}
\end{figure}

But $v_{R(P)}(s)=z<v_{R(P')}(s)=z-\gamma$, proving that FM is not fulfilled.}
\end{enumerate}
\end{proof}

\begin{prop}
For any $\epsilon\in(0,1)$ and considering the domain to be $\epsilon$-coherent, then the family $\{B_\alpha\}_{\alpha\in(0,1)}$ satisfies:
\begin{enumerate}[(i)]
    \item Weak Unanimity; and
    \item Endorsed Unanimity.
\end{enumerate}
\end{prop}

\begin{proof}
   \begin{enumerate}[(i)]
   \item Weak Unanimity follows from propositions \ref{prop:D}, \ref{prop:ICoh} and \ref{prop:FGWeak}.
   \item Endorsed Unanimity follows from propositions  \ref{prop:I}, \ref{prop:DCoh} and \ref{prop:FGEnd}.
    \end{enumerate}

\end{proof}

\begin{prop}\label{prop:BalphaCohCountCC}
For any $\delta\in(0,1)$, $B_\alpha$ over a $\delta$-coherent domain still does not satisfy $\epsilon$-Collective coherence for any $\epsilon\in(0,1)$.
\end{prop}

\begin{proof}

First, we show that the $\epsilon$-coherence condition for $B_\alpha$ depends on the functions employed in its definition, namely on $D$ and $I$:

\begin{align*}
|v_{B_\alpha(P)}(s) -e_{B_\alpha(P)}(s)|&= \Big|v_{B_\alpha(P)}(s)- \frac{\sum_{r\in R^+(s)} \big(\alpha v_{D(P)}(_rs)+(1-\alpha) v_{I(P)}(s_r)\big)w_{D(P)}(r) }{\sum_{r\in(R^+(s)} w_{D(P)}(r) }\Big|\\\\
&= \Big|\big(\alpha (v_{D(P)}(s)+(1-\alpha)(v_{I(P)}(s)\big)-\big(\alpha e_{D(P)}(s))+(1-\alpha) e_{I(P)}(s))\big)\Big|\\
&= \Big|\alpha \big(v_{D(P)}(s)-e_{D(P)}(s)\big)+(1-\alpha) \big(v_{I(P)}(s)-e_{I(P)}(s)\big)\Big|
\end{align*}

Thus, since $D$ and $I$ do not satisfy $\epsilon$-collective coherence for any $\delta$-coherent profile (by propositions \ref{prop:DCohCountCC} and \ref{prop:ICohCountCC} respectively),  neither will $B_\alpha$ satisfy the property for any $\alpha\in(0,1)$. 
Indeed, consider for instance the DRF and $\delta$-coherent opinion profile $P$, with any $\delta\in(0,1)$, in figure \ref{fig:ICohCountCC} as employed in proposition \ref{prop:ICohCountCC}. If we compute $\epsilon$-collective coherence for $B_\alpha$ at sentence $s$ we obtain that:


$$|v_{B_\alpha(P)}(s) -e_{B_\alpha(P)}(s)|=\Big|\alpha (-1-0)+(1-\alpha) (-\frac 1 2 - \frac 1 2)\Big|=|-1|=1>\epsilon$$
 
 for any $\epsilon\in(0,1)$. So, $B_\alpha$ does not fulfill $\epsilon$-coherence for any $\alpha\in(0,1)$.  

\end{proof}

\begin{prop}\label{prop:BalphaCoh}
For any $\epsilon\in(0,1)$ and considering the domain to be $\epsilon$-coherent, $B_\alpha$ does not fulfil the following properties for any $\alpha\in(0,1)$:
\begin{enumerate}[(i)]
    \item Sided Unanimity and Narrow Unanimity;
    \item Monotonicity; and
    \item Independence.
\end{enumerate}
\end{prop}
\begin{proof}
\begin{enumerate}[(i)]

    \item Sided Unanimity. It suffices to build a DRF and an opinion profile for which Sided unanimity does not hold for any values of $\alpha$ and $\epsilon$, where $\alpha, \epsilon\in(0,1)$. The next counterexample serves too see that Narrow unanimity fails too.
    
    
    Consider the set $A=\{(x,y)\in(0,1)\quad | \quad 0<y<\epsilon\ \mbox{ and }\ 0<x<y-\alpha y\}$. We check first, that this set is actually not empty. For $\alpha\in (0,1)$, $y-\alpha y>0$, thus $y>\alpha y> 0$. So for $y\in(0,\epsilon)$, there are $x\in(0,1)$ satisfying $x<y-\alpha y$.

    
    \begin{figure}[H]
\centering
\begin{tikzpicture}[node distance=0.5cm and 0.5cm,>=latex,auto, every place/.style={draw}]

     \node [place,thick] (s) {$s$};
      \node [place,thick] (a) [right= 2cm of s]{$a$};
	
    \node[right=0.4cm of a] at (a) {$v(a)=x-y$};
    \node[left=0.4cm of s] at (s) {$v(s)=x$};
    \path[->] (s) edge node {$w(r)=1$} (a); 
\end{tikzpicture}
\caption{}
\label{fig:BalphaCohCountSU}
\end{figure}

    Now, we consider the DRF and opinion profile depicted in figure \ref{fig:BalphaCohCountSU} where $x$ and $y$ are values from $A$, namely $(x,y)\in A$.
    Since $|v(s)-e(s)|=|x-(x-y)|=|y|<\epsilon$, the opinion profile in the figure is $\epsilon$-coherent, and satisfies the assumptions for Sided Unanimity at $s$ because $v(s)=x>0$. However,
\begin{align*}
    v_{B_\alpha(P)}(s)&=\alpha v_{D(P)}(s)+(1-\alpha)v_{I(P)}(s)\\
    &= \alpha x + (1-\alpha)(x-y)\\
    &= x-y+\alpha y< 0
\end{align*}
    since $(x,y)\in A$. So, clearly this example shows that Sided Unanimity does not hold for the family $B_\alpha$ in an $\epsilon$-coherent profile. 
    
    \setcounter{enumi}{3}
    \item Independence.
    For any $\alpha\in(0,1)$, $B_\alpha$ does not fulfil Independence due to its dependence on $I$.
    
    \setcounter{enumi}{2}
    \item Monotonicity. Straightforward from the fact that $B_\alpha$ does not fulfil Independence  for any $\alpha\in(0,1)$ and from proposition \ref{prop:MtoI}.
\end{enumerate}
\end{proof}

\begin{prop}\label{prop:BalphaCohCountCC}
For any $\delta\in(0,1)$, $R_\alpha$ over a $\delta$-coherent domain satisfies $\epsilon$-Collective coherence for $\alpha\leq \frac \epsilon 2$.
\end{prop}

\begin{proof}
As seen before in proposition \ref{prop:R}, the collective coherence of $R_\alpha$ entirely depends  on the collective coherence of $D$, i.e.:

$$|v_{R_\alpha(P)}(s)-e_{R_\alpha(P)}(s)|=\alpha|v_{D(P)}(s)-e_{D(P)}(s)|$$

Thus, finding the worst case scenario for $D$ will give us the condition on $\alpha$ that ensures that $R_\alpha$ satisfies $\epsilon$-collective coherence for any $\epsilon$. Next, we consider an example showing that $|v_{D(P)}(s)-e_{D(P)}(s)|$ can be as close to $2$ as wanted depending on the number of agents.

\begin{figure}[H]
    \centering
    \includegraphics[scale=1.2]{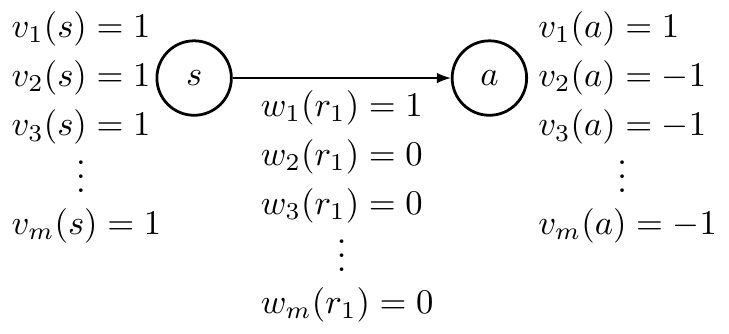}
    \caption{}
    \label{fig:RalphaCohCountCC}
\end{figure}

Let $P$ be the $\delta$-coherent opinion profile over a DRF depicted in figure \ref{fig:RalphaCohCountCC}, for any $\delta\in(0,1)$. For any $i>1$: $v_i(s)=1$, $v_i(a)=-1$ and $w_i(r_1)=0$; whereas $v_1(s)=1$, $v_1(a)=1$ and $w_1(r_1)=1$. We check the condition for collective coherence at $s$ to find that:

$$|v_{R_\alpha(P)}(s)-e_{R_\alpha(P)}(s)|=\alpha(1+\frac{m-2} m)<\epsilon$$

 if $\alpha<\frac {\epsilon}{1+\frac{m-2} m}$. Thus, by taking $\alpha<\frac \epsilon 2<\frac {\epsilon}{1+\frac{m-2} m}$ we ensure that $R_\alpha$  satisfies $\epsilon$-coherence  for the worst case. Therefore,  for any $\delta$-coherent opinion profile, $\delta\in(0,1)$, choosing $\alpha<\frac\epsilon 2$ will ensure that $R_\alpha$  satisfies $\epsilon$-collective coherence for any $\epsilon\in(0,1)$.

\end{proof}

\begin{prop}
Let $\epsilon\in(0,1)$ such that the domain of $R_\alpha$ is an $\epsilon$-coherent domain, then the family $\{R_\alpha\}_{\alpha\in(0,1)}$ satisfies:
\begin{enumerate}[(i)]
    \item Weak Unanimity for $\alpha> \frac 1 2$; and
    \item Endorsed unanimity for $\alpha> \frac 1{2-\epsilon}$.
\end{enumerate}
\end{prop}

\begin{proof}
\begin{enumerate}[(i)]
    \item Weak Unanimity. Consider a DRF with sentences $\sen$, $P$ an opinion profile over the DRF and $s\in\sen$ a sentence such that $v_i(s)=1$ for any agent $i$. We know that
    $$v_{D(P)}(s)=\frac 1 n \sum_{i\in Ag} v_i(s)=1.$$
    Now we turn our attention to $R$, the Recursive function. We consider the worst scenario for $R_\alpha$, which happens when $v(s)=1$ and $v_{R(P)}(s)=-1$. The DRF and profile depicted in  figure \ref{fig:RCountWU} above shows that, in fact, this scenario exists with $v_{D(P)}(s)=1$ and $v_{R(P)}(s)=-1$, and hence $v_{R_\alpha(P)}(s)=\alpha+(1-\alpha)(-1)=2\alpha - 1$. To fulfil Weak unanimity, we need that $v_{R_\alpha(P)}(s)>0$ holds, but we also know that $v_{R_\alpha(P)}(s)\geq 2\alpha -1$. Therefore, we can guarantee  Weak Unanimity by choosing $\alpha>\frac 1 2$. The proof for the negative case of Weak Unanimity goes analogously.

\item Endorsed Unanimity. To prove this property we will build a customised DRF and opinion profile to demonstrate the worst case that we can find when fulfilling the assumptions of Endorsed unanimity.


Consider a DRF 
and let $P=(O_1=(v_1,w_1),\dotsc, O_n=(v_n,w_n))$ be an $\epsilon$-coherent profile with full positive support on statement $s\in \sen$. 

First, we consider the worst case where $v_{R(P)}(s)=-1$ can be achieved when $s$ has full positive support. Figure \ref{fig:RalphaCohEU} depicts a $DRF$ and an opinion profile illustrating this situation.

\begin{figure}[H]
\centering
\begin{tikzpicture}[node distance=0.5cm and 0.5cm,>=latex,auto, every place/.style={draw}]

     \node [place,thick] (s) {$s$};
      \node [place,thick] (a1) [right= 2cm of s]{$a_1$};
      \node [place,thick] (a2) [right= 2cm of a1]{$a_2$};
      \node [thick] (dots) [right= 1cm of a2]{$\cdots$};
      \node [place,thick] (am1) [right= 1cm of dots]{$a_{m-1}$};
      \node [place,thick] (am) [right= 2cm of am1]{$a_m$};
	
    \node[below=0.4cm of a1] at (a1) {$v(a_1)=1$};
    \node[below=0.4cm of a2] at (a2) {$v(a_2)=1-x$};
    \node[above=0.4cm of am1] at (am1) {$v(a_{m-1})=1-(m-1)x$};
     \node[below=0.4cm of am] at (am) {$v(a_m)=-1$};
    \node[above=0.4cm of s] at (s) {$v(s)=1$};
    \path[->] (s) edge node {} (a1); 
    \path[->] (a1) edge node {} (a2); 
    \path[->] (a2) edge node {} (dots); 
    \path[->] (dots) edge node {} (am1); 
    \path[->] (am1) edge node {} (am); 

\end{tikzpicture}
\caption{}
\label{fig:RalphaCohEU}
\end{figure}

By choosing $0<x<\epsilon$ and $m\in\N$ such that $mx>2\geq(m-1)x$, this example shows an $\epsilon$-coherent profile where $v(a_1)=1$ (full positive support) and $v_R(s)=-1$. 
Next, we move to the general setting considered by the proof, an opinion profile with $n$ agents, knowing that the worst case for this property is possible.
Since $v_i(s_r)=1$ for any descendant $s_r\in D(s)$ and any agent $i$, the estimation function on $s$ will be $e_i(s)=1$ for any agent. Therefore, from the coherence condition at $s$ we conclude that
$$1-\epsilon <v_i(s)<1+\epsilon.$$

Consider that for every $i$, $v_i(s)=1-\delta_i$ such that $0\leq\delta_i<\epsilon$. This clearly satisfies the previous inequality. Now we take $\delta=\max_i \{\delta_1,\dotsc,\delta_n\}$ to create a new opinion profile $P'=(O'_1=(v'_1,w_1),\dotsc, O'_n=(v'_n,w_n))$ such that $v_i(a)=v'_i(a)$ for $a\in\sen\setminus \{s\}$ and $v'_i(s)=1-\delta$ for any $i$. 
Then, since $D$ fulfils Monotonicity and Narrow unanimity, we know that $v_{D(P')}(s)\leq v_{D(P)}(s)$  and $v_{D(P')}(s)=1-\delta$ respectively. And, from the example in figure \ref{fig:RalphaCohEU} we know that for any $\epsilon$-coherent opinion profile $v_{R(P)}(s)\geq-1$. Therefore,
\begin{align*}
    v_{R_\alpha(P)}(s)&= \alpha v_{D(P)}(s)+(1-\alpha)v_{R(P)}(s)\\
    &\geq \alpha v_{D(P')}(s)+(1-\alpha)(-1)\\
    &= (1-\delta)\alpha - (1-\alpha)= (2-\delta)\alpha-1
\end{align*}

So, if we set $\alpha\in(0,1)$ so that $(2-\delta)\alpha-1>0$, $R_\alpha$ will satisfy Endorsed Unanimity. Since $\delta<\epsilon$, as close as possible, imposing $\alpha\geq\frac 1{2-\epsilon}>\frac 1{2-\delta}$ the property is satisfied.

\end{enumerate}
\end{proof}

\begin{prop}
For any $\epsilon\in(0,1)$, and considering the domain to be $\epsilon$-coherent, $R_\alpha$ does not fulfil the following properties for any $\alpha\in(0,1)$:
\begin{enumerate}[(i)]
    \item Sided unanimity, and therefore Narrow unanimity;
    \item Familiar Monotonicity, and therefore Monotonicity; and
    \item Independence.
\end{enumerate}
\end{prop}
\begin{proof}
\begin{enumerate}[(i)]
    \item Sided Unanimity and Narrow Unanimity. It suffices to build a DRF and an opinion profile for which Sided Unanimity does not hold for any $\alpha,\epsilon\in (0,1)$. 

    Consider the DRF and opinion profile $P$ depicted in figure \ref{fig:RalphaCohCountSU}, where: $x\in(0,1)$ is such that $0<x<\frac {1-\alpha} {\alpha}$, $0<\delta < \epsilon$;  $m\in \N$ satisfies $(m-1)\delta\leq 1+< m\delta$; and for any $r\in\rel$, $w(r)=1$.
    \begin{figure}[H]
\centering
\begin{tikzpicture}[node distance=0.5cm and 0.5cm,>=latex,auto, every place/.style={draw}]

     \node [place,thick] (s) {$s$};
      \node [place,thick] (a1) [right= 2cm of s]{$a_1$};
      \node [thick] (dots) [right= 1cm of a1]{$\cdots$};
      \node [place,thick] (am1) [right= 1cm of dots]{$a_{m-1}$};
      \node [place,thick] (am) [right= 2cm of am1]{$a_m$};
	
    \node[below=0.4cm of a1] at (a1) {$v(a_1)=x-\delta$};
    \node[above=0.4cm of am1] at (am1) {$v(a_{m-1})=x-(m-1)\delta$};
     \node[below=0.4cm of am] at (am) {$v(a_m)=-1$};
    \node[above=0.4cm of s] at (s) {$v(s)=x$};
    \path[->] (s) edge node {} (a1); 
    \path[->] (a1) edge node {} (dots); 
    \path[->] (dots) edge node {} (am1); 
    \path[->] (am1) edge node {} (am); 

\end{tikzpicture}
\caption{}
\label{fig:RalphaCohCountSU}
\end{figure}
  
Clearly, $P$ is an $\epsilon$-coherent because $|v(a_m)-e(a_m)| = 0$, and for any $i<m$, $|v(a_i)-e(a_i)|=v(a_i)-v(a_{i+1})=\delta<\epsilon$, and $|v(s)-e(s)|=x-x+\delta<\epsilon$.
Furthermore, $P$ satisfies the assumptions of Sided unanimity at $s$ since $v(s)=x>0$.   
It is straightforward to see that $v_{D(P)}(s)=x$ and $v_{R(P)}(s)=v_{R(P)}(a_m)=-1$. Hence,
$v_{R\alpha(P)}(s)=x\alpha + (1-\alpha)(-1)=x\alpha +\alpha -1$.
But since $x<\frac {1-\alpha} \alpha$, we conclude that $v_{R\alpha(P)}(s)< \frac {1-\alpha} \alpha \alpha +\alpha -1=0$. This proves that Sided unanimity is not fulfilled, and as it is a single-opinion profile Narrow unanimity fails too. We can proceed analogously for the negative case of Sided Unanimity.

\item Familiar monotonicity. The counterexample employed in proposition \ref{prop:RCoh} to show that Familiar monotonicity does not hold for $R$ serves here as well to prove that $R_\alpha$ does not satisfy Familiar monotonicity for any $\alpha\in(0,1)$. From opinion profiles $P$ and $P'$ depicted in figures \ref{fig:RCohCountFM1} and \ref{fig:RCohCountFM2} respectively, we extract that $v_{D(P)}(s)=v_{D(P')}(s)=1$ and $1=v_{R(P)}(s)>v_{R(P')}(s)=1-x$. Therefore, it follows that $v_{R_\alpha(P)}(s)>v_{R_\alpha(P')}(s)$, hence proving that Familiar monotonicity does not hold.

    \item Independence.
    For any $\alpha\in(0,1)$, function $R_\alpha$ does not fulfil Independence due to its dependence on $R$.
    
\end{enumerate}
\end{proof}

\subsection{Constrained opinion profiles: assuming Consensus on acceptance degrees and coherent profiles}\label{Proofs:UniformCoherent}

Next, we show the results regarding our fourth, and last, scenario.   We now assume that opinion profiles are both $\epsilon$-coherent, for some $\epsilon\in(0,1)$, and agree on their acceptance degrees over relationships. The results that follow are summarised in table \ref{tableUniformAcceptance} in section \ref{subsec:sameAcceptanceResults}. 


Likewise in previous sections, next we only prove per aggregation function those properties that either were partially satisfied or not satisfied at all in previous scenarios, but do hold in this new scenario. We do not prove those properties for which the proofs in the previous sections serves as well for this scenario.

\begin{prop}\label{prop:Uniform}
Let be a $DRF$ and an opinion profile $P=(O_1,\dotsc,O_n)$. For any $s\in \sen$, assume that for each $r\in R^+(s)$ $w_i(r)=\lambda_r\in(0,1]$ for any $i$, then:

\begin{enumerate}[(i)]
    \item For any $\epsilon\in(0,1)$, if $0<\delta\leq\epsilon$ and the domain $\mathcal{D}$ is $\delta$-coherent then $D(P)$ is $\epsilon$-coherent, so satisfies $\epsilon$-Collective coherence.
    \item  For any $\epsilon\in(0,1)$, if $0<\delta\leq\epsilon$ and the domain $\mathcal{D}$ is $\delta$-coherent then $I(P)$ is $\epsilon$-coherent, so satisfies $\epsilon$-Collective coherence.
    \item  For any $\epsilon\in(0,1)$, if $0<\delta\leq\epsilon$ and the domain $\mathcal{D}$ is $\delta$-coherent then $B_\alpha(P)$ is $\epsilon$-coherent for any $\alpha\in(0,1)$, so satisfies $\epsilon$-Collective coherence.
    \item  For any $\epsilon\in(0,1)$, if $0<\delta\leq\epsilon$ and the domain $\mathcal{D}$ is $\delta$-coherent then $R_\alpha(P)$ is $\epsilon$-coherent for any $\alpha\in(0,1)$, so satisfies $\epsilon$-Collective coherence.
\end{enumerate}
\end{prop}

\begin{proof}
\begin{enumerate}[(i)]
    \item Collective Coherence of $D$. Let $s\in\sen$. We assume that for any $i$, $|v_i(s)-e_i(s)|<\delta\leq\epsilon$. Next we calculate the coherence condition for $D$ at  sentence $s$:
    \begin{align*}
    |v_{D(P)}(s)-e_{D(P)}(s)|&= \Bigg|\frac 1 n \sum_i v_i(s) - \frac 1 {\sum_{r\in R^+(s)} w_{D(P)}(r)} \sum_{r\in R^+(s)} w_{D(P)}(r)v_{D(P)}(s_r)\Bigg|\\
    &=\Bigg|\frac 1 n \sum_i v_i(s) - \frac 1 {\sum_{r\in R^+(s)} \lambda_r} \sum_{r\in R^+(s)} \lambda_r \big(\frac 1 n \sum_i v_i(s_r)\big)\Bigg|\\
    &= \Bigg|\frac 1 n \sum_i \Big( v_i(s)- \frac 1 {\sum_{r\in R^+(s)} w_i(r)}\sum_{r\in R^+(s)} w_i(r) v_i(s_r)\Big)\Bigg|\\
    &= \Bigg|\frac 1 n \sum_i \Big(v_i(s)-e_i(s)\Big)\Bigg| \leq \frac 1 n \sum_i \Big|v_i(s)-e_i(s)\Big|\\
    \end{align*}
    
    Thus, by $\delta$-coherence of the domain we obtain that:

    $$|v_{D(P)}(s)-e_{D(P)}(s)|\leq \frac 1 n \sum_i \Big|v_i(s)-e_i(s)\Big|<\frac 1 n \sum_i \delta \leq \epsilon$$
    
    This proves that the collective opinion by $D$ is $\epsilon$-coherent.

    \item Collective Coherence of $I$. We prove collective coherence for $I$ similarly to the proof above for $D$. Let $s\in\sen$. We assume that for any $i$, $|v_i(s)-e_i(s)|<\delta\leq\epsilon$. We compute the condition for the collective coherence of $I$ at sentence $s$ as follows:
    
     \begin{align*}
    |v_{I(P)}(s)-e_{I(P)}(s)|&= \Bigg|\frac 1 n \sum_i e_i(s) - \frac 1 {\sum_{r\in R^+(s)} w_{I(P)}(r)} \sum_{r\in R^+(s)} w_{I(P)}(r)v_{I(P)}(s_r)\Bigg|\\
    &=\Bigg|\frac 1 n \sum_i e_i(s) - \frac 1 {\sum_{r\in R^+(s)} \lambda_r} \sum_{r\in R^+(s)} \lambda_r \big(\frac 1 n \sum_i e_i(s_r)\big)\Bigg|\\
    &= \Bigg|\frac 1 n \sum_i \Big( e_i(s)- \frac 1 {\sum_{r\in R^+(s)} w_i(r)}\sum_{r\in R^+(s)} w_i(r) e_i(s_r)\Big)\Bigg|\\
    &=\Bigg|\frac 1 n \sum_i \frac {\sum_{r\in R^+(s)} w_i(r) v_i(s_r) - \sum_{r\in R^+(s)} w_i(r) e_i(s_r)}{\sum_{r\in R^+(s)} w_i(r)}\Bigg|\\
    &=\Bigg|\frac 1 n \sum_i \sum_{r\in R^+(s)} \frac { w_i(r)\Big( v_i(s_r)-e_i(s_r)\Big)} {\sum_{r\in R^+(s)} w_i(r)}\Bigg|
    \end{align*}
    
   So, by $\delta$-coherence of the domain, we obtain that:
   
   \begin{align*}|v_{D(P)}(s)-e_{D(P)}(s)|&\leq \frac 1 n \sum_i \sum_{r\in R^+(s)} \frac { w_i(r)\Big| v_i(s_r)-e_i(s_r)\Big|} {\sum_{r\in R^+(s)} w_i(r)}\\
   & <  \frac 1 n \sum_i  \frac {\sum_{r\in R^+(s)}  w_i(r)\ \delta } {\sum_{r\in R^+(s)} w_i(r)}\\
   &=\frac 1 n \sum_i \delta=\delta \leq \epsilon
   \end{align*}
   
     This proves that the collective opinion by $I$ is $\epsilon$-coherent.

    \item Collective Coherence of $B_\alpha$. We have just proven that $D$ and $I$ satisfy $\epsilon$-collective coherence assuming consensus on acceptance degrees and  a $\delta$-coherent domain with $\delta<\epsilon$. It directly follows that for any $\alpha\in(0,1)$, then $B_\alpha$ on a $\delta$-coherent domain also satisfies $\epsilon$-collective coherence.

    \item Collective Coherence of $R_\alpha$. We have just proven that $D$ satisfies $\epsilon$-collective coherence assuming consensus on  acceptance degrees and a $\delta$-coherent domain. $\epsilon$-collective coherence also holds for $R$ under the same assumptions following proposition \ref{prop:R} (see collective coherence for $R$). Hence, it follows that for any $\alpha\in(0,1)$, $R_\alpha$ on a $\delta$-coherent domain also satisfies $\epsilon$-collective coherence.

\end{enumerate}
\end{proof}

\medskip

\newpage 
\bibliographystyle{elsarticle-num}
\bibliography{references}

\end{document}